\documentclass{article}
\usepackage{geometry}

\usepackage[utf8]{inputenc} 
\usepackage[T1]{fontenc}    
\usepackage[pagebackref=true]{hyperref}       
\usepackage{url}            
\usepackage{booktabs}       
\usepackage{amsfonts}       
\usepackage{nicefrac}       
\usepackage{microtype}      
\usepackage{xcolor}         
\usepackage{amsmath}
\usepackage{algorithm}
\usepackage{algorithmic}
\usepackage{bbm}
\usepackage{subfig}
\usepackage{graphicx}
\usepackage{mathtools}
\usepackage{pdflscape}
\usepackage{listings}
\usepackage{caption}
\usepackage{amsthm}
\usepackage{enumitem}
\usepackage{authblk}
\usepackage{natbib}

\renewcommand*\backref[1]{\ifx#1\relax \else (Cited on #1) \fi}

\newtheorem{fact}{Fact}

\title{Provably and Practically Efficient Neural Contextual Bandits}

\author[$\ddagger$]{Sudeep Salgia}
\author[*]{Sattar Vakili}
\author[$\ddagger$]{Qing Zhao}
\affil[$\ddagger$]{School of Electrical \& Computer Engineering, Cornell University, Ithaca, NY, \emph{\{ss3827,qz16\}@cornell.edu} }
\affil[*]{MediaTek Research, UK, \emph{sattar.vakili@mtkresearch.com}}

\usepackage{amssymb}

\usepackage{hyperref}  
\hypersetup{
    colorlinks=true,
    linkcolor=blue,
    citecolor =blue,
    filecolor=magenta,      
    urlcolor=magenta,
}

\def\nn{\nonumber}

\def\NN{\textsc{NN}}
\def\NTK{\textsc{NTK}}

\def\Dc{\mathcal{D}}
\def\Oc{\mathcal{O}}

\def\Oct{\tilde{\mathcal{O}}}

\newcommand{\1}{\mathbbm{1}}

\def\E{\mathbb{E}}
\def\Ss{\mathbb{S}}
\def\Rr{\mathbb{R}}
\def\Ii{\mathbb{I}}
\newcommand{\R}{\mathbb{R}}
\newcommand{\N}{\mathbb{N}}

\def\sigmahat{\hat{\sigma}}

\newcommand{\cA}{\mathcal{A}}
\newcommand{\cB}{\mathcal{B}}
\newcommand{\cC}{\mathcal{C}}
\newcommand{\cD}{\mathcal{D}}
\newcommand{\cE}{\mathcal{E}}
\newcommand{\cF}{\mathcal{F}}
\newcommand{\cH}{\mathcal{H}}

\newcommand{\cM}{\mathcal{M}}
\newcommand{\cN}{\mathcal{N}}

\newcommand{\cT}{\mathcal{T}}

\newcommand{\cX}{\mathcal{X}}

\newcommand{\cZ}{\mathcal{Z}}

\newcommand{\bfa}{\mathbf{a}}
\newcommand{\bfb}{\mathbf{b}}
\newcommand{\bff}{\mathbf{f}}
\newcommand{\bfg}{\mathbf{g}}
\newcommand{\bfh}{\mathbf{h}}
\newcommand{\bfr}{\mathbf{r}}
\newcommand{\bfu}{\mathbf{u}}
\newcommand{\bfv}{\mathbf{v}}
\newcommand{\bfw}{\mathbf{w}}
\newcommand{\bfx}{\mathbf{x}}
\newcommand{\bfy}{\mathbf{y}}
\newcommand{\bfz}{\mathbf{z}}

\newcommand{\bfA}{\mathbf{A}}
\newcommand{\bfB}{\mathbf{B}}
\newcommand{\bfD}{\mathbf{D}}
\newcommand{\bfF}{\mathbf{F}}
\newcommand{\bfG}{\mathbf{G}}
\newcommand{\bfH}{\mathbf{H}}
\newcommand{\bfI}{\mathbf{I}}

\newcommand{\bfM}{\mathbf{M}}
\newcommand{\bfV}{\mathbf{V}}
\newcommand{\bfW}{\mathbf{W}}
\newcommand{\bfX}{\mathbf{X}}
\newcommand{\bfY}{\mathbf{Y}}
\newcommand{\bfZ}{\mathbf{Z}}

\newcommand{\bfxi}{\boldsymbol{\xi}}

\newcommand{\bfPi}{\boldsymbol{\Pi}}
\newcommand{\bfDelta}{\boldsymbol{\Delta}}
\newcommand{\bfSigma}{\boldsymbol{\Sigma}}

\newcommand{\ip}[2]{\left \langle #1, #2 \right \rangle}
\newcommand{\ec}[2]{{#1}^{(#2)}}
\newcommand{\diag}{\text{diag}}
\newcommand{\tr}{\mathrm{tr}}
\newcommand{\cov}{\text{Cov}}
\newcommand{\var}{\text{Var}}

\DeclareMathOperator*{\argmax}{arg\,max}

\newtheorem{lemma}{Lemma}
\newtheorem{theorem}{Theorem}

\newtheorem{definition}{Definition}

\newtheorem{assumption}{Assumption}

\begin{document}

\maketitle

\begin{abstract}
  We consider the neural contextual bandit problem. In contrast to the existing work which primarily focuses on ReLU neural nets, we consider a general set of smooth activation functions. Under this more general setting,
  (i) we derive non-asymptotic error bounds on the difference between an overparameterized neural net and its corresponding neural tangent kernel, (ii) we propose an algorithm with a provably sublinear regret bound that is also efficient in the finite regime as demonstrated by empirical studies. The non-asymptotic error bounds may be of broader interest as a tool to establish the relation between the smoothness of the activation functions in neural contextual bandits and the smoothness of the kernels in kernel bandits.
\end{abstract}

\section{Introduction}

The stochastic contextual bandit has been extensively studied in the literature~\citep[see,][and references therein]{ langford2007epoch,lattimore2020bandit}. In this problem, at each discrete sequential step, a \emph{context} is revealed by the environment. Then, a bandit agent chooses one of the $K$ available actions. Choosing each action yields a random reward, the distribution of which is determined by the context. The problem was primarily studied under a linear setting where the expected reward of each arm is a linear function of the context~\citep{chu2011contextual, li2019nearly, chen2021efficient}. It was later extended to kernel-based models~\citep{valko2013finite}, where the expected reward function associated with the context-action pairs belongs to the reproducing kernel Hilbert space (RKHS) determined by a known kernel. Recently, a variation of the problem has been proposed where the expected reward function associated with the context-action pairs is modeled using a neural net~\citep{zhou2020neuralUCB, zhang2020neuralTS, gu2021batched, kassraie2021neural}. The three contextual bandit settings mentioned above are increasingly more complex in the following order:
\emph{Linear $\rightarrow$ Kernel-based $\rightarrow$ Neural}. \\

Contextual bandits find application in recommender systems, information retrieval, healthcare and finance~\citep[see a survey on applications in][]{bouneffouf2019survey}. The problem can also be seen as a middle step from classic stochastic bandits~\citep{auer2002finite} and (non-contextual) linear bandits~\citep{abbasi2011improved} towards a reinforcement learning (RL) problem on a Markov decision process (MDP), where the contexts resemble the states of the MDP. One of the first formulations of linear contextual bandits was referred to as \emph{associative RL} with linear value function~\citep{auer2002using}. Nonetheless, contextual bandit is different from RL in that the contexts are determined arbitrarily (adversarially) rather than following a stochastic Markovian process.  \\

\subsection{Existing Results on Neural Contextual Bandits}

With neural nets demonstrating a great representation power (much more general than linear models), there has been an increasing interest in modeling bandit and RL problems based on neural nets. This setting can be implemented using the typical neural net toolboxes. The neural contextual bandit has been considered in several works: \cite{zhou2020neuralUCB} and \cite{zhang2020neuralTS}, respectively, adopted the upper confidence bound (UCB) and Thompson sampling (TS) algorithms to the neural bandit setting. The algorithms are, respectively, referred to as NeuralUCB and NeuralTS. \cite{gu2021batched} considered a batch observation setting, where the reward function can be evaluated only on batches of data. \cite{kassraie2021neural} provided analysis for a variation of NeuralUCB algorithm with diminishing instantaneous regret. \\

The analysis of neural bandits has been enabled through the theory of neural tangent (NT) kernel~\citep{jacot2018neural} which approximates an overparameterized neural net with the kernel corresponding to its infinite width, and the bounds on the approximation error (e.g., see, ~\citet{arora2019exact}). \cite{zhou2020neuralUCB,zhang2020neuralTS, gu2021batched} proved $\Oct(\Gamma_k(T)\sqrt{T})$\footnote{The notations $\Oc$ and $\Oct$ are used for mathematical order and that up to logarithmic factors, respectively. } bounds on cumulative regret, where $T$ is the number of steps and $\Gamma_k(T)$ is a complexity term determined by the neural tangent kernel $k$ associated with the particular neural network. Specifically, $\Gamma_k(T)$ is the information gain corresponding to $k$ for datasets of size $T$. For the ReLU neural nets on $d$-dimensional domains considered in \cite{zhou2020neuralUCB,zhang2020neuralTS, gu2021batched}, $\Gamma_k(T)=\Oct(T^{\frac{d-1}{d}})$ is the best upper bound known for the information gain~\citep{kassraie2021neural, vakili2021uniform}. Inserting this bound on $\Gamma_k(T)$, the aforementioned regret bounds become trivial (superlinear) generally when $d>1$, unless further restrictive assumptions are made on the contexts. \cite{kassraie2021neural} addressed this issue by considering the \emph{Sup} variant of NeuralUCB (referred to as SupNeuralUCB). The Sup variant is an adoption of SupLinUCB, which was initially introduced in~\citep{chu2011contextual} for linear contextual bandits, to the neural setting. \cite{kassraie2021neural} proved an $\Oct(\sqrt{\Gamma_k(T)T})$ regret bound for SupNeuralUCB in the case of ReLU neural nets, which solved the issue of superlinear regret bound (non-diminishing instantaneous regret).

\subsection{Contribution}

All the existing works mentioned above, are limited to neural nets with ReLU activation functions. This limitation is rooted in the fact that the existing error bound between an overparameterized neural net and the corresponding NT kernel~\citep[][ Theorem~$3.2$]{arora2019exact} holds only for the ReLU neural nets. In Theorem~\ref{theorem:appx_error}, which may be of broader interest, we extend this result by providing error bounds for neural nets with arbitrarily smooth activation functions. Using Theorem~\ref{theorem:appx_error}, together with the recently established bounds on $\Gamma_k(T)$ corresponding to arbitrarily smooth neural nets~\citep{vakili2021uniform}, we provide regret bounds for neural contextual bandits under a more general setting than only ReLU activation function. In particular, in Theorem~\ref{theorem:regret_guarantees}, we prove an $\Oct(T^{\frac{2d+2s-3}{2d+4s-4}})$ upper bound on regret, where $s$ is the smoothness parameter of the activation functions ($s=1$ in the case of ReLU, and $s$ grows larger as the activations become smoother). Our regret bounds recover that of \cite{kassraie2021neural} for ReLU activations, and can become as small as $\Oct(\sqrt{T})$ when the activations become infinitely smooth.  \\

Broadening the scope of neural bandits to general smooth activations is of interest in several aspects. The smooth activation functions are more suitable for some applications such as implicit representations (see ~\citet{li2019better, sitzmann2020implicit} and references therein). In addition, although sublinear, the regret bounds for ReLU neural nets grow at a fast $\Oct(T^{\frac{2d-1}{2d}})$ rate, that quickly approaches the linear rate as $d$ grows. We show that this rate can significantly reduce for smooth activations. Furthermore, our results help establish connections between varying smoothness of activations in neural bandits and varying smoothness of kernels in kernel-based bandits . Kernel-based bandits~\citep{srinivas2010gaussian} and Kernel-based contextual bandits~\citep{valko2013finite} are well studied problems with regret bounds reported for popular kernels such as Mat{\'e}rn family, which is perhaps the most commonly used family of kernels in practice~\citep{snoek2012practical, shahriari2015taking}. Our regret bound for a neural contextual bandit problem with activations with smoothness parameter $s$ is equivalent to the regret bound in the case of kernel-based bandits with a Mat{\'e}rn kernel~\citep{valko2013finite, li2021gaussian} with smoothness parameter $s-\frac{1}{2}$. This connection may be of general interest in terms of contrasting neural nets and kernel-based models through NT kernel theory. \\

In Theorem~\ref{theorem:conf_interval}, we prove confidence intervals for overparameterized neural nets with smooth activation functions, which are later used in the analysis of our algorithm. The confidence intervals are a consequence of our Theorem~\ref{theorem:appx_error} and those for kernel regression. In contrast to prior work, our confidence intervals are tighter and hold for a general set of smooth activation functions (see Sec.~\ref{sec:appx_error}). \\

In addition to the above analytical results for neural bandits with general smooth activation functions, we propose a practically efficient algorithm.
The Sup variants of UCB algorithms are known to perform poorly in experiments~\citep[][]{calandriello2019gaussian, li2021gaussian}, including SupNeuralUCB presented in~\cite{kassraie2021neural},  due to over-exploration. We propose NeuralGCB, a neural contextual bandit algorithm \underline{g}uided by both upper and lower \underline{c}onfidence \underline{b}ounds to provide a finer control of exploration-exploitation trade-off to avoid over-exploration. 
In the experiments, we show that NeuralGCB outperforms all other existing algorithms, namely NeuralUCB \citep[][]{zhou2020neuralUCB}, NeuralTS  \citep[][]{zhang2020neuralTS} and SupNeuralUCB \citep{kassraie2021neural}.    \\

\subsection{Other Related Work}

Linear bandits have been considered also in a non-contextual setting, where the expected rewards are linear in the actions in a fixed domain~\citep[e.g., see the seminal work of][]{abbasi2011improved}. The kernel-based bandits (also referred to as Gaussian process bandits) have been extensively studied under a non-contextual setting~\citep{srinivas2010gaussian, Chowdhury2017bandit}. The GP-UCB and GP-TS algorithms proposed in this setting also show suboptimal regret bounds (see~\citet{vakili2021open} for details). Recently there have been several works offering optimal order regret bounds for kernel-based (non-contextual) bandits~\citep{salgia2021domain, camilleri2021high, li2021gaussian}. These results however do not address the additional difficulties in the contextual setting.

\section{Preliminaries and Problem Formulation}
\label{sec:preliminaries}

In the stochastic contextual bandit problem, at each discrete sequential step $t=1,2,\dots, T$, for each action $a\in[K] := \{1,2,\dots,K\}$, a context vector $\bfx_{t}\in \cX$ is revealed, where $\cX$ is a compact subset of $\Rr^d$. Then, a bandit agent chooses an action $a_t\in[K]$ and receives the corresponding reward $r_{t} = h(\bfx_{t, a_t})+\xi_t$, where $h:\Rr^d\rightarrow \Rr$ is the underlying reward function and $\xi_t\in\Rr$ is a zero-mean martingale noise process. The goal is to choose actions which maximize the cumulative reward over $T$ steps. This is equivalent to minimizing the cumulative \emph{regret}, that is defined as the total difference between the maximum possible context-dependent reward and the actua  received reward
\begin{eqnarray}\label{eq:regret_def}
R(T) =\sum_{t=1}^T (h(\bfx_{t, a^*_t}) - h(\bfx_{t, a_t})),
\end{eqnarray}
where $a^*_t\in \arg\max_{a\in [K]}h(\bfx_{t, a})$ is the context-dependent maximizer of the reward function at step~$t$. Following is a formal statement of the assumptions on $\cX$, $f$ and $\xi_t$. These are mild assumptions which are shared with other works on neural bandits and NT kernel. 
\begin{assumption}\label{ass1}
\emph{(i)} The input space is the $d$ dimensional hypersphere:  $\cX=\Ss^{d-1}$. \emph{(ii)} The reward function $h$ belongs to the RKHS induced by the NT kernel, and $h(\bfx)\in [0, 1], \ \forall \ \bfx \in \cX$. \emph{(iii)} $\xi_t$ are assumed to be conditionally $\nu$-sub-Gaussian, i.e., for any $\zeta\in\Rr$, $\ln(\E[e^{\zeta \xi_t}|\xi_1, \dots, \xi_{t-1}]) \leq \zeta^2 \nu^2/2$. 
\end{assumption}

\subsection{The Neural Net Model and Corresponding NT Kernel}\label{subs:nn}

In this section, we briefly outline the neural net model and the associated NT kernel. 
Let $f(\bfx; \bfW)$ be a fully connected feedforward neural net with $L$ hidden layers of equal width $m$. We can express $f$ using the following set of recursive equations:
\begin{align}
    f^{(1)}(\bfx) & = \ec{\bfW}{1}\bfx, \ \  f^{(l)}(\bfx) = \sqrt{\frac{c_{s}}{m}}\ec{\bfW}{l}\sigma_s(\ec{f}{l-1}(\bfx)), ~~~1<l\le L  \nonumber \\
    f(\bfx; \bfW) & =\sqrt{\frac{c_{s}}{m}} \ec{\bfW}{L+1}\sigma_s(\ec{f}{L}(\bfx)).  \label{eq:nn}
\end{align}

Let $\bfW=(\bfW^{l})_{1\le l\le L+1}$ denote the collection of all weights. The dimensions of weight matrices satisfy: $\bfW^{(1)}\in \Rr^{m\times d}$; for $1<l\le L$, $\bfW^{(l)}\in \Rr^{m\times m}$; $\bfW^{(L+1)}\in \Rr^{1\times m}$.  All the weights $\bfW^{(l)}_{i,j}$, $\forall \  l,i,j$, are initialized to $\mathcal{N}(0,1)$ and $c_s := 2/(2s-1)!!$. With an abuse of notation, $\sigma_s: \Rr^m \rightarrow \Rr^m$ is used for coordinate-wise application the activations of the form $\sigma_s(u)=(\max(0,u))^s$ to the outputs of each layer. Note that $\sigma_s$ is $s-1$ times differentiable which gives our results a general interpretability across smoothness of activations and the resulting function classes and allows us to draw parallels between our results and well established results for kernel-based bandits.  \\

It has been shown that gradient based training of the neural net described above reaches a global minimum where the weights remain within a close vicinity of random initialization. As a result the model can be approximated with its linear projection on the tangent space at random initialization~$\bfW_0$:
\begin{eqnarray}\nonumber
    f(\bfx; \bfW)\approx f(\bfx;\bfW_0) +( \bfW-\bfW_0)^{\top}\bfg(\bfx;\bfW_0),
\end{eqnarray}
where the notation $\bfg(\bfx; \bfW)=\nabla_{\bfW} f(\bfx; \bfW)$ is used for the gradient of $f$ with respect to the parameters.
The approximation error can be bounded by the second order term $\Oc(\|\bfW-\bfW_0\|^2)$, and shown to be diminishing as the width $m$ grows, where the implied constant depends on the spectral norm of the Hessian matrix~\citep{liu2020linearity}. 
The NT kernel corresponding to this neural net when $m$ grows to infinity is defined as the kernel in the feature space determined by the gradient at initialization:
\begin{eqnarray}\nn
k(\bfx,\bfx')=\lim_{m\rightarrow \infty}\bfg^{\top}(\bfx; \bfW_0)\bfg(\bfx;\bfW_0).
\end{eqnarray}
The particular form on the NT kernel depends on the activation function and the number of layers. For example, for a two layer neural net with ReLU activations, we have
\begin{eqnarray}\nn
k(\bfx,\bfx') = \frac{1}{\pi}\left(\bfx^{\top}\bfx'(\pi-\arccos(\bfx^{\top}\bfx'))+\sqrt{1-(\bfx^{\top}\bfx')^2}\right) + \frac{\bfx^{\top}\bfx'}{\pi}(\pi-\arccos(\bfx^{\top}\bfx')),
\end{eqnarray}
For the closed form derivation of NT kernel for other values of $s$ and $L$, see~\cite{vakili2021uniform}.

\subsection{Assumptions on Neural Net and NT Kernel}

The following technical assumptions are mild assumptions which are used in the handling of the overparameterized neural nets using NT kernel. These assumptions are common in the literature and often are fulfilled without loss of generality~\citep{arora2019exact, zhou2020neuralUCB, zhang2020neuralTS, gu2021batched, kassraie2021neural}. 

\begin{assumption}\label{ass2}
\emph{(i)} Consider $\bfH\in\Rr^{KT\times KT}$ such that $[\bfH]_{i,j} = k(\bfz_i,\bfz_j)$ for all pairs in $\cZ \times \cZ$, where $\cZ = \{\{\bfx_{t, a}\}_{a=1}^K\}_{t=1}^T$. We assume $\bfH \succcurlyeq \lambda_0 \Ii$. 
\emph{(ii)} $f(\bfx; \bfW_0)$ is assumed to be $0$. \emph{(iii)} The number of epochs during training, $J$, and the learning rate, $\eta$, satisfy $J = \tilde{\Omega}(T)$ and $\eta = \Oc(1/T)$.
\end{assumption}

\subsection{Information Gain}
The regret bounds for neural bandits are typically given in terms of maximal information gain (or the effective dimension) of the NT kernel. The maximal information gain is defined as the maximum log determinant of the kernel matrix~\citep{srinivas2010gaussian}:
\begin{eqnarray}
\Gamma_k(t)=\sup_{\bfX_t\subseteq \cX} \log\left(\det\left(\Ii+\frac{1}{\lambda}k_{\bfX_t, \bfX_t}\right)\right).
\end{eqnarray}

It is closely related to the effective dimension of the kernel, which denotes the number of features with a significant impact on the regression model and can be finite for a finite dataset even when the feature space of the kernel is infinite dimensional~\citep{zhang2005learning, valko2013finite}. It is defined as 
\begin{eqnarray}
\tilde{d}_k(t)  = \tr\left(k_{\bfX_t,\bfX_t}(k_{\bfX_t,\bfX_t}+\lambda\Ii)^{-1}\right).
\end{eqnarray}

It is known that the information gain and the effective dimension are the same up to logarithmic factors~\citep{calandriello2019gaussian}. We give our results in terms of information gain; nonetheless, that can be replaced by effective dimension.

\section{Approximation Error and Confidence Bounds}\label{sec:appx_error}

As discussed earlier, the error between an overparameterized neural net and the associated NT kernel can be quantified and shown to be small for large $m$. To formalize this, we recall the technique of kernel ridge regression. Given a dataset $\Dc_t=\{(\bfx_i,y_i)\}_{i=1}^t$, let $f_{\NTK}$ denote the regressor obtained by kernel ridge regression using NT kernel. In particular, we have
\begin{eqnarray}
f_{\NTK}(\bfx) = k^{\top}_{\bfX_t}(\bfx)(\lambda\Ii_t+k_{\bfX_t,\bfX_t})^{-1}\bfY_t,
\end{eqnarray}
where $k_{\bfX_t}(\bfx)=[k(\bfx_1,\bfx), \dots, k(\bfx_t,\bfx)]^{\top}$ is the NT kernel evaluated between $\bfx$ and $t$ data points, $k_{\bfX_t,\bfX_t}=[k(\bfx_i, \bfx_j)]_{i,j=1}^t$ is the NT kernel matrix evaluated on the data, $\bfY_t=[y_1, \dots, y_t]^{\top}$ is the vector of output values, and $\lambda \geq 0$ is a free parameter. Note that $f_{\NTK} $ can be computed in closed form (without training a neural net). In addition, let $f_{\NN}$ be  prediction of the neural net at the end of training using $\Dc_t$. Theorem $3.2$ of~\cite{arora2019exact} states that, when the activations are ReLU and $m$ is sufficiently large, we have, with probability at least $1-\delta$, that $|f_{\NN}(\bfx)-f_{\NTK}(\bfx)|\le \epsilon$. In this theorem, the value of $m$ should be sufficiently large in terms of $t, L, \delta, \epsilon$, and $\lambda_0$. We now present a bound on the error between the neural net and the associated NT kernel for smooth activations. 

\begin{theorem}\label{theorem:appx_error}
Consider the neural net $f(\bfx; \bfW)$ defined in~\eqref{eq:nn} consisting of $L$ hidden layers each of width $m \geq \mathrm{poly}(1/\varepsilon, s, 1/\lambda_0, |\cD|, \log(1/\delta))$ with activations $\sigma_s$. Let $f_{\NTK}$ and $f_{\NN}$ be the kernel ridge regressor and trained neural net using a dataset $\Dc$. Then for any $\bfx \in \cX$, with probability at least $1 - \delta$ over the random initialization of the neural net, we have,
\begin{align*}
    |f_{\NTK}(\bfx) - f_{\NN}(\bfx)| \leq \varepsilon.
\end{align*}
\end{theorem}

Theorem~\ref{theorem:appx_error} is an generalization of~Theorem $3.2$ in~\cite{arora2019exact} to the class of smooth activation functions $\{\sigma_s, s \in \N\}$. 
To prove Theorem~\ref{theorem:appx_error}, we show the following bound on the error between NT kernel and the kernel induced by the finite width neural net at initialization. This result is an generalization of Theorem $3.1$ in \cite{arora2019exact} to smooth activation functions. 

\begin{lemma} \label{lemma:kernel_init_to_NTK}
Consider the neural net $f(\bfx; \bfW)$ defined in~\eqref{eq:nn} consisting of $L$ hidden layers each of width $m \geq \Oc(\frac{s^L L^2c_s^2}{\varepsilon^2}\log^2(sL/\delta \varepsilon))$ with activations $\sigma_s$. Let $\bfg(\bfx;\bfW)=\nabla_{\bfW}f(\bfx;\bfW)$ and $k$ be the associated NT kernel, as defined in Section~\ref{subs:nn}.
Fix $\varepsilon > 0$ and $\delta \in (0, 1)$. Then for any $\bfx, \bfx' \in \cX$, with probability at least $1 - \delta$ over the initialization of the network, we have,
\begin{align*}
    \left| \bfg^{\top}(\bfx;\bfW_0)\bfg(\bfx';\bfW_0) - k(\bfx, \bfx') \right| \leq (L + 1)\varepsilon.
\end{align*} 
\end{lemma}

The proof entails a non-trivial generalization of various lemmas used in the proof of~\citet[Theorem 3.2]{arora2019exact} to the class of smooth activation functions. We defer the detailed proof of the lemma and theorem to Appendix~\ref{sec:thm_1_proof}.

\paragraph{Confidence Intervals:} The analysis of bandit problems classically builds on confidence intervals applicable to the values
of the reward function. The NT kernel allows us to use the confidence intervals for kernel ridge regression, in building confidence intervals for overparameterized neural nets. In particular, given a dataset $\cD_t$, let 
\begin{eqnarray}
\bfV_t=\lambda\Ii+\sum_{i=1}^t\bfg(\bfx_i;\bfW_0)\bfg^{\top}(\bfx_i;\bfW_0)
\end{eqnarray} 
be the Gram matrix in the feature space determined by the gradient. Taking into account the linearity of the model in the feature space, we can define a (surrogate) posterior variance as follows,
\begin{eqnarray}
\sigmahat^2_t(\bfx) = \bfg^{\top}(\bfx)\bfV_t^{-1}\bfg(\bfx), \label{eqn:def_sigmahat}
\end{eqnarray}
that can be used as a measure of uncertainty in the prediction provided by the trained neural net. 
Using Theorem~\ref{theorem:appx_error} in conjunction with the confidence intervals for kernel ridge regression, we prove the following confidence intervals for a sufficiently wide neural net. 

\begin{theorem}\label{theorem:conf_interval}
Let $\cD_{t} = \{(\bfx_i, r_i)\}_{i = 1}^t$ denote a dataset obtained under the observation model described in Section~\ref{sec:preliminaries} such that the points $\{\bfx_i\}_{i = 1}^t$ are independent of the noise sequence $\{\xi_i\}_{i = 1}^t$. Suppose Assumptions~\ref{ass1} and~\ref{ass2} hold. Suppose the neural net defined in~\eqref{eq:nn} consisting of $L$ hidden layers each of width $m \geq \mathrm{poly}(T, s, L, K, \lambda^{-1}, \lambda_0^{-1}, \log(1/\delta))$ is trained using this dataset. Then, with probability at least $1 - \delta$, the following relation holds for any $\bfx \in \cX$:
\begin{align}
    \left|h(\bfx) - f(\bfx; \bfW_t) \right| \leq \frac{C_{s, L}t}{\lambda m} +  \beta_t  \sigmahat_t(\bfx),
\end{align}
where $\bfW_t$ denotes the parameters of the trained model, $\beta_t  = S +  2 \nu \sqrt{ \log (1/\delta) } +  C_{s,L}'t(1 - \eta \lambda)^{J/2}\sqrt{t/\lambda} + C''_{s, L}\sqrt{t^3/\lambda m}$, $S$ is the RKHS (corresponding to the NT kernel) norm of $h$ and $C_{s, L}, C_{s, L}', C_{s,L}''$ are constants depending only on $s$ and $L$. 
\end{theorem}

Both results presented in this section may be of broader interest in neural net literature. We would like to emphasize that these bounds hold for all activation functions $\{\sigma_{s}: s \in \N\}$ improving upon the existing results which hold only for ReLU activation function. Furthermore, this bound is tighter than the one derived in~\cite{kassraie2021neural}, even for the case of ReLU activation function.The confidence intervals are important building blocks in the analysis of NeuralGCB in Section~\ref{sec:alg}. Please refer to Appendix~\ref{sec:thm_2_proof} for a detailed proof of the theorem.

\section{Algorithm}\label{sec:alg}

The UCB family of algorithms achieve optimal regret in classic stochastic finite action bandits~\citep{auer2002finite}.
The optimal regret, however, hinges on the statistical independence of the actions. In more complex settings such as kernel-based and neural bandits, the inherent statistical dependence across adaptively chosen actions leads to skewed posterior distributions, resulting in sub-optimal and potentially trivial (i.e., superlinear) regret guarantees of UCB based algorithms~\citep{vakili2021open}. To address this issue, one approach adopted in the literature is to use samples with limited adaptivity. These algorithms are typically referred to as Sup variant of UCB algorithms, and have been developed for linear~\citep{chu2011contextual}, kernel-based~\citep{valko2013finite} and neural~\citep{kassraie2021neural} contextual bandits. These Sup variants, however, are known to perform poorly in practice due to their tendency to overly explore suboptimal arms. \\

We propose NeuralGCB, a neural bandit algorithm guided by both upper and lower confidence bounds
to avoid over-exploring, leading to superior performance in practice  while preserving the nearly optimal regret guarantee. The key idea of NeuralGCB is a finer control of the exploration-exploitation tradeoff based on the predictive standard deviation of past actions. This finer control encourages more exploitation and reduces unnecessary exploration.  Specifically, NeuralGCB partitions past action-reward pairs into $R = \log T$ subsets (some subsets may be empty at the beginning of the learning horizon). Let $\{\Psi^{(r)}\}_{r=1}^R$ denotes these subsets of action-reward pairs where the index $r$ represents the level of uncertainty about the data points in that set (with $r=1$ representing the highest uncertainty). Upon seeing a new context at time $t$, NeuralGCB traverses down the subsets starting from $r = 1$. At each level $r$, it evaluates the largest predictive standard deviation among current set of candidate actions, $A_r$, and compares it with $2^{-r}$. \\

If this value is smaller than $2^{-r}$, NeuralGCB checks if it can exploit based on predictions at the $r^{\text{th}}$ level. Specifically, if the predictive standard deviation of the action that maximizes the UCB score is $\Oc(1/\sqrt{t})$ (indicating a high reward with reasonably high confidence), then NeuralGCB plays that action. Otherwise, it updates $A_r$ by eliminating actions with small UCB score and moves to the next level. This encourages exploitation of less certain actions in a controlled manner as opposed to SupNeuralUCB which exploits only actions with much higher certainty.  \\

On the other hand, if the value is greater than $2^{-r}$, NeuralGCB directly exploits the maximizer of the mean computed using data points in $\Psi^{(r-1)}$ until the allocated budget for exploitation at level $r$ is exhausted. It then resorts to more exploratory actions by choosing actions with large values of $\sigmahat_{t-1}^{(r)}$. The exploitation budget is set to match the length of exploration sequence at the corresponding level to ensure an optimal balance of exploitation and exploration and preserve the optimal regret order while boosting empirical performance.  \\

In addition to improved empirical performance, NeuralGCB takes into consideration the practical requirements for training the neural nets. Specifically, as pointed out in~\cite{gu2021batched}, it is practically more efficient to train the neural net over batches of observations, rather than sequentially at each step. Consequently, before evaluating the mean and variance at any level $r$, NeuralGCB retrains the neural net corresponding to that level only if $q_r$ samples have been added to that level since it was last trained. This index-dependent choice of batch size, $q_r$ lends a natural adaptivity to the retraining frequency by reducing it as time progresses.  \\

A pseudo code of the algorithm is outlined in Algorithm~\ref{alg:NeuralGCBmainpaper}. In the pseudo code, $\text{UCB}_{t}^{(r)}$ and $\text{LCB}_t^{(r)}$ refer to the upper and lower confidence scores respectively at time $t$ corresponding to index $r$ and are defined as $\text{UCB}_{t}^{(r)}(\cdot) = f(\cdot; \bfW_{t}^{(r)}) + \beta \sigmahat_{t}^{(r)}(\cdot)$ and $\text{LCB}_{t}^{(r)}(\cdot) = f(\cdot; \bfW_{t}^{(r)}) - \beta \sigmahat_{t}^{(r)}(\cdot)$. GetPredictions is a local routine that calculates the predictive mean and variance after appropriately retraining the neural net (see supplementary for a pseudo code). Lastly, the arrays \texttt{ctr}, \texttt{max\_mu} and \texttt{fb} are used to store the exploitation count, maximizer of the neural net output and feedback time instants. We now formally state the regret guarantees for NeuralGCB in the following theorem.

\begin{theorem} \label{theorem:regret_guarantees}
Suppose Assumptions~\ref{ass1} and \ref{ass2} hold. Consider NeuralGCB given in Algorithm~\ref{alg:NeuralGCB}, with $R$ neural nets, as defined in~\eqref{eq:nn} with $L$ hidden layers each of width $m \geq \text{poly}(T, L, K, \lambda^{-1}, \lambda_0^{-1}, S^{-1}, \log(1/\delta))$. Suppose NeuralGCB is run with a fixed batch size for each group, then the regret~defined in~\eqref{eq:regret_def} satisfies
\begin{align*}
    R(T) =\Oct\left( \sqrt{T \Gamma_k(T)} +  \sqrt{T \Gamma_k(T) \log(1/\delta)} +\max_rq_r\Gamma_k(T)\right) 
\end{align*} 
\end{theorem}
As suggested by the above theorem, NeuralGCB preserves the regret guarantees of SupNeuralUCB which are much tighter than those of NeuralUCB. Moreover, these guarantees hold for even for smooth activation functions as opposed to just for ReLU activation. Furthermore, the above regret guarantees show that the retraining neural nets only at the end of batches of observations increases the regret at most with an additive term in the batch size. Our results also hold with an adaptive batch size (see supplementary material).


\begin{algorithm}
	\caption{NeuralGCB}
	\label{alg:NeuralGCBmainpaper}
	\begin{algorithmic}[1]
		\STATE \textbf{Require}: Time horizon $T$, maximum initial variance $\sigma_0$, error probability $\delta$
		\STATE \textbf{Initialize}: $R \gets \lceil \log_2 T \rceil$, Ensemble of $R$ Neural Nets with $\bfW_0^{(r)} = \bfW_0$, $\Psi_{0}^{(r)} \leftarrow \emptyset, \ \  \forall \ r \in [R]$, $\cH \leftarrow \emptyset$, arrays \texttt{ctr}, \texttt{max\_mu}, \texttt{fb} of size $R$ with all elements set to $0$, batch sizes $\{q_r\}_{r = 1}^R$
		\FOR{$t = 1,2,3, \dots T$}
		\STATE $r \leftarrow 1, \hat{A}_r(t) = [K]$
		\WHILE{True}
		\STATE Receive the context-action pairs $\{\bfx_{t, a}\}_{a = 1}^K$
		\STATE $\{f(\bfx_{t,a}; \bfW_t^{(r)}), \sigmahat_{t-1}^{(r)}(\bfx_{t,a})\}_{a =1}^{K}, \bfW_{t}^{(r)}, \texttt{fb}[r] \leftarrow$ GetPredictions$\left(\cH, \Psi_{t-1}^{(r)}, \{\bfx_{t, a}\}_{a = 1}^K, \texttt{fb}[r], \bfW_{t-1}^{(r)}, q_r\right)$ 
		\STATE $\tilde{\sigma}_{t-1}^{(r)} \leftarrow \max_{a \in \hat{A}_r(t)} \sigmahat_{t-1}^{(r)}(x_{t, a})$, \quad $\texttt{max\_mu}[r] \leftarrow \argmax_{a \in \hat{A}_r(t)} f(x_{t, a}; \bfW_{t-1}^{(r)})$
		\IF{$\tilde{\sigma}_{t-1}^{(r)} \leq \sigma_0 2^{-r}$}
		\STATE $a_{\text{UCB}} \leftarrow \argmax_{a \in \hat{A}_r(t)} \text{UCB}_{t-1}^{(r)}(\bfx_{t,a})  $
		\IF{${\sigma}_{t-1}^{(r)}(x_{t, a_{\text{UCB}}}) \leq \eta_0/\sqrt{t}$}
		\STATE Choose $a_t \gets a_{\text{UCB}}$ and set $\upsilon_t \gets 1$
		\STATE Receive $y_t = h(\bfx_{t, a_t}) + \xi_t$ and update $\cH \leftarrow \cH \cup \{(x_{t, a_t}, y_t)\}$
		\STATE Set $\Psi_{t}^{(r+1)} \leftarrow \Psi_{t-1}^{(r+1)} \cup \{(t, \upsilon_t)\}$ and $\Psi_{t}^{(r')} \leftarrow \Psi_{t-1}^{(r')} $ for all $r' \in [R]\setminus\{r+1\}$
		\STATE \textbf{break}
		\ELSE
		\STATE $\hat{A}_{r+1}(t) \leftarrow \{a \in \hat{A}_{r}(t) : \text{UCB}_{t-1}^{(r)}(\bfx_{t,a})\geq \max_{a' \in \hat{A}_r(t)} \text{LCB}_{t-1}^{(r)}(\bfx_{t,a'})\}$, $r \leftarrow r + 1$
		\ENDIF
		\ELSE
		\IF{$r = 1$ \algorithmicor $\ \texttt{ctr}[r] > \alpha_0 4^r$}
		\STATE Choose any $a_t \in \hat{A}_r(t)$ such that $\sigmahat_{t-1}^{(r)}(x_{t, a}) > \sigma_0 2^{-r}$ and set $\upsilon_t \gets 2$
		\ELSE
		\STATE Choose $a_t \gets \texttt{max\_mu}[r-1]$ and set $\upsilon_t \gets 3$
		\ENDIF
		\STATE Receive $y_t = h(\bfx_{t, a_t}) + \xi_t$ and update $\cH \leftarrow \cH \cup \{(x_{t, a_t}, y_t)\}$, $\texttt{ctr}[r] \leftarrow \texttt{ctr}[r] + 1$
		\STATE Set $\Psi_{t}^{(r)} \leftarrow \Psi_{t-1}^{(r)} \cup \{(t, \upsilon_t)\}$ and $\Psi_{t}^{(r')} \leftarrow \Psi_{t-1}^{(r')} $ for all $r' \in [R]\setminus\{r\}$
		\STATE \textbf{break}
		\ENDIF
		\ENDWHILE
		\ENDFOR
	\end{algorithmic}
\end{algorithm}

\section{Empirical Studies}\label{sec:exp}

In this section, we provide numerical experiments on comparing NeuralGCB with several representative baselines, namely, LinUCB~\citep{chu2011contextual}, NeuralUCB~\citep{zhou2020neuralUCB}, NeuralTS~\citep{zhang2020neuralTS}, SupNeuralUCB~\citep{kassraie2021neural} and Batched NeuralUCB~\citep{gu2021batched}.  \\

We perform the empirical studies on three synthetic and two real-world datasets. We first compare NeuralGCB with the fully sequential algorithms (LinUCB, NeuralUCB, NeuralTS ans SupNeuralUCB). We then compare the regret incurred and the time taken by NeuralGCB, NeuralUCB and Batch NeuralUCB. We perform both set of experiments with two different activation functions. The construction of the synthetic datasets and the experimental settings are described below.

\begin{figure*}[p]
\centering
\subfloat[$h_1(x)$ with $\sigma_1(x)$]{\label{fig:ip_s1}\centering \includegraphics[scale = 0.23]{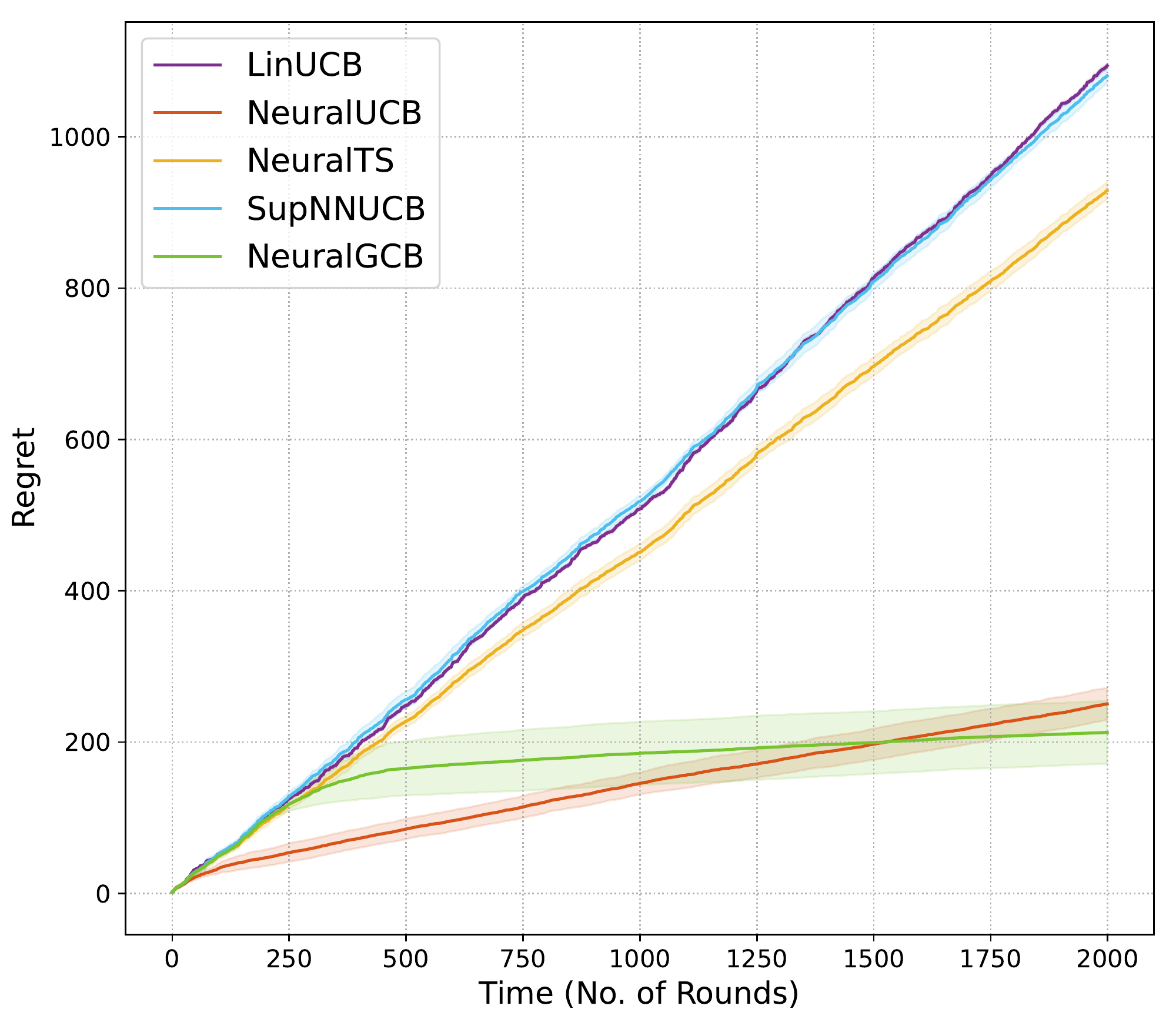}}
~
\subfloat[$h_1(x)$ with $\sigma_2(x)$]{\label{fig:ip_s2}\centering \includegraphics[scale = 0.23]{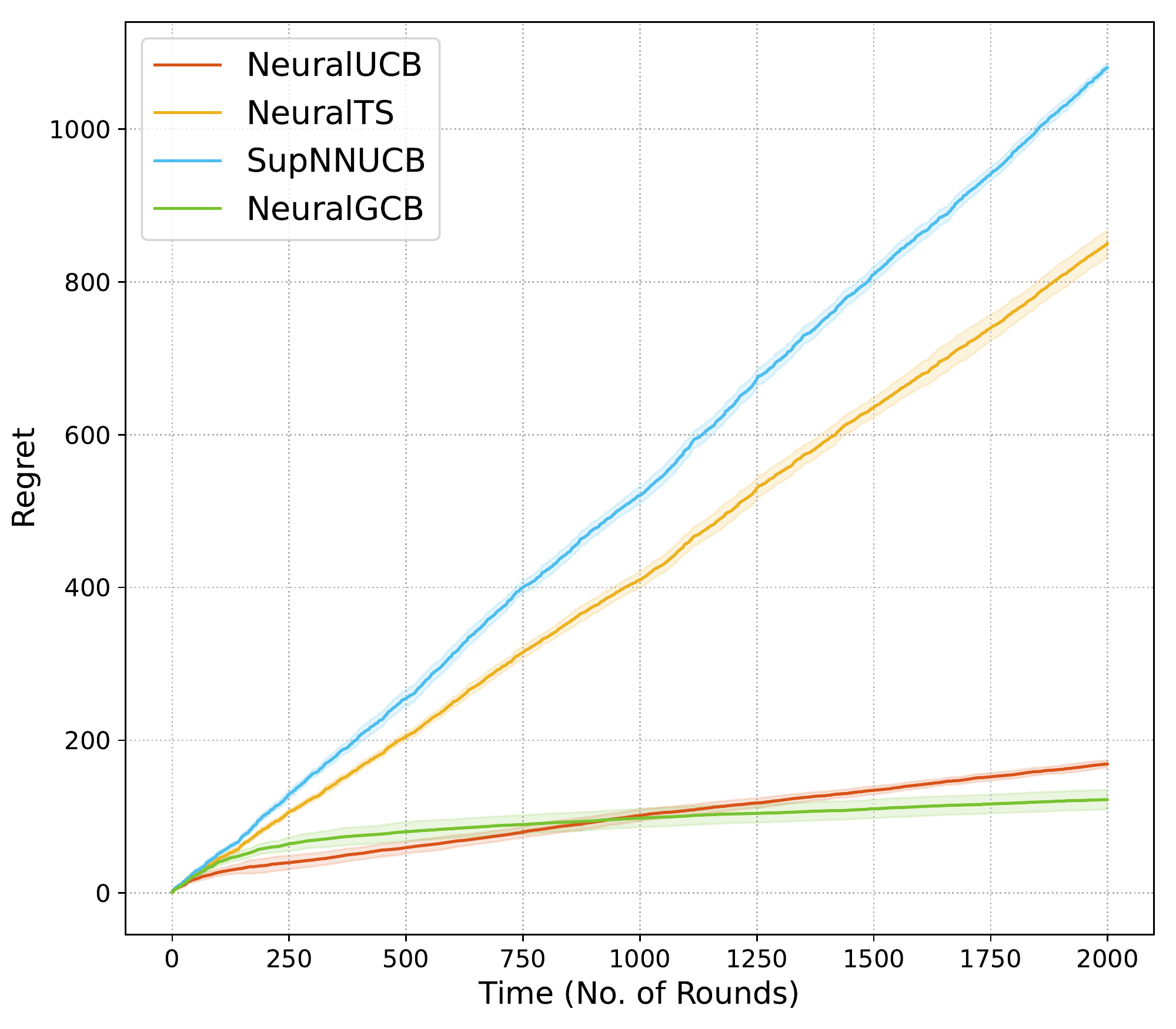}}
~
\subfloat[Batched algorithms on $h_1(x)$]{\label{fig:cosine_s1} \centering \includegraphics[scale = 0.23]{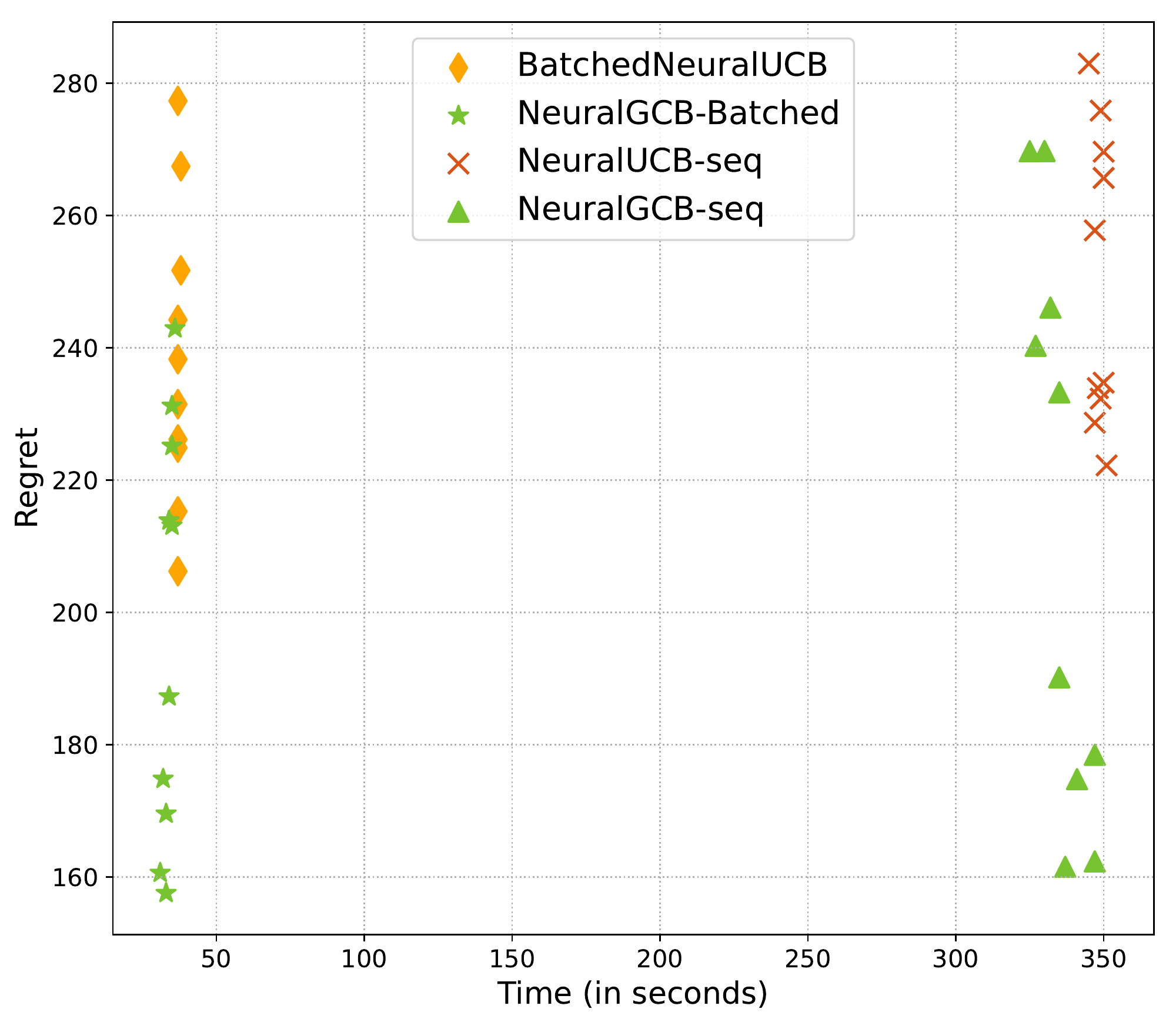}}

\subfloat[$h_2(x)$ with $\sigma_1(x)$]{\label{fig:cosine_s1}\centering \includegraphics[scale = 0.23]{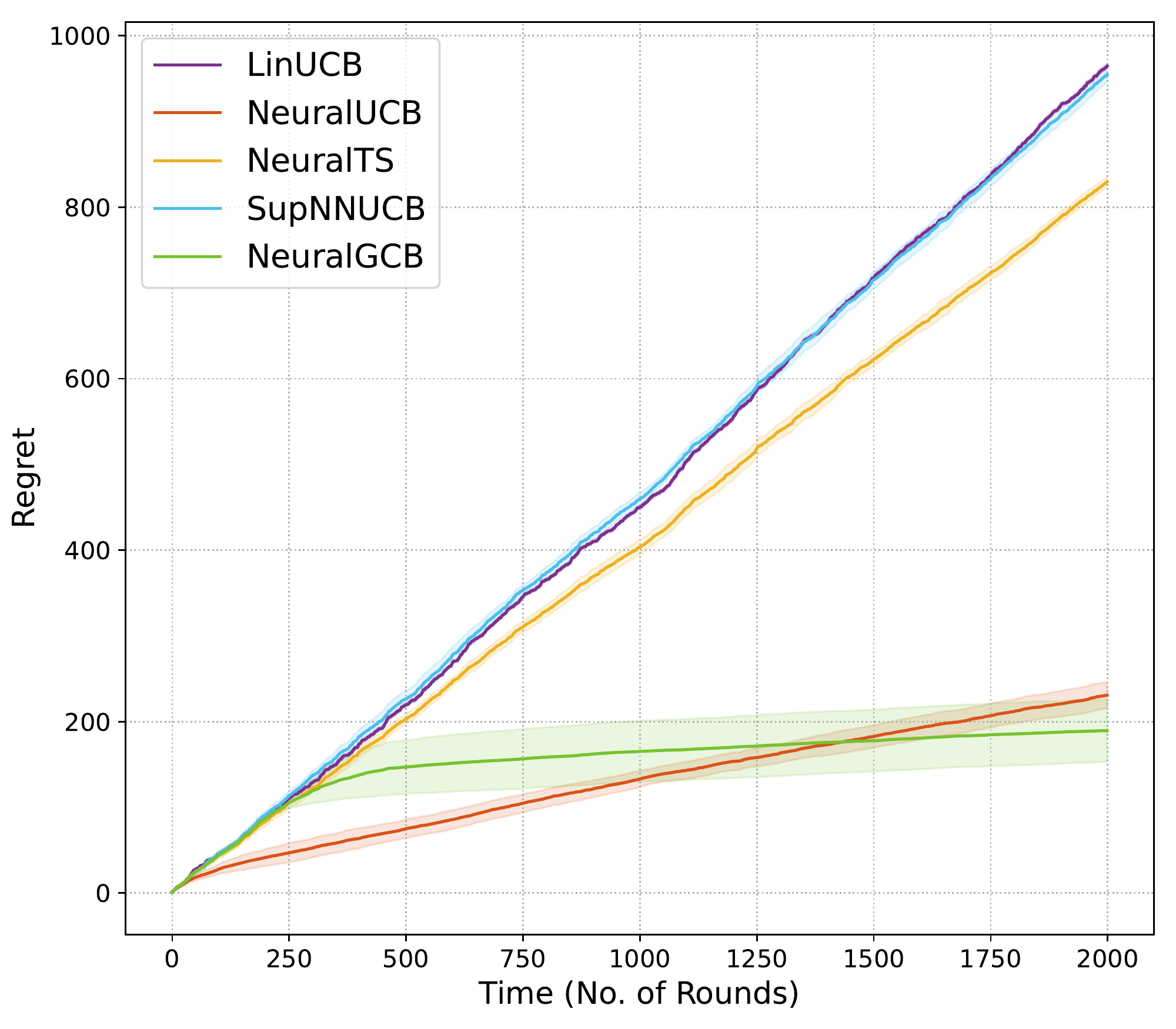}}
~
\subfloat[$h_2(x)$ with $\sigma_2(x)$]{\label{fig:cosine_s2}\centering \includegraphics[scale = 0.23]{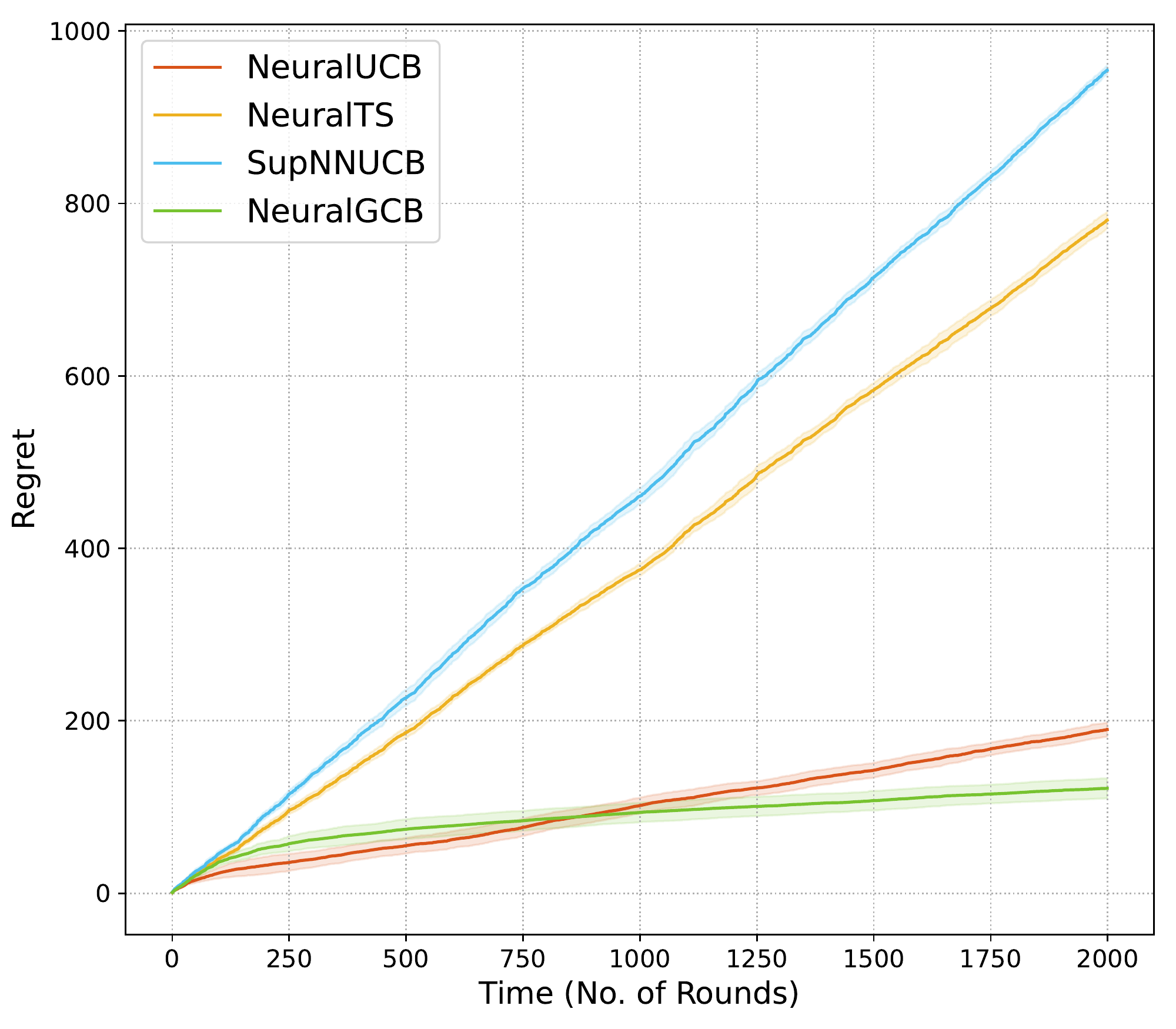}}
~
\subfloat[Batched algorithms on $h_2(x)$]{\label{fig:cosine_batched} \centering \includegraphics[scale = 0.23]{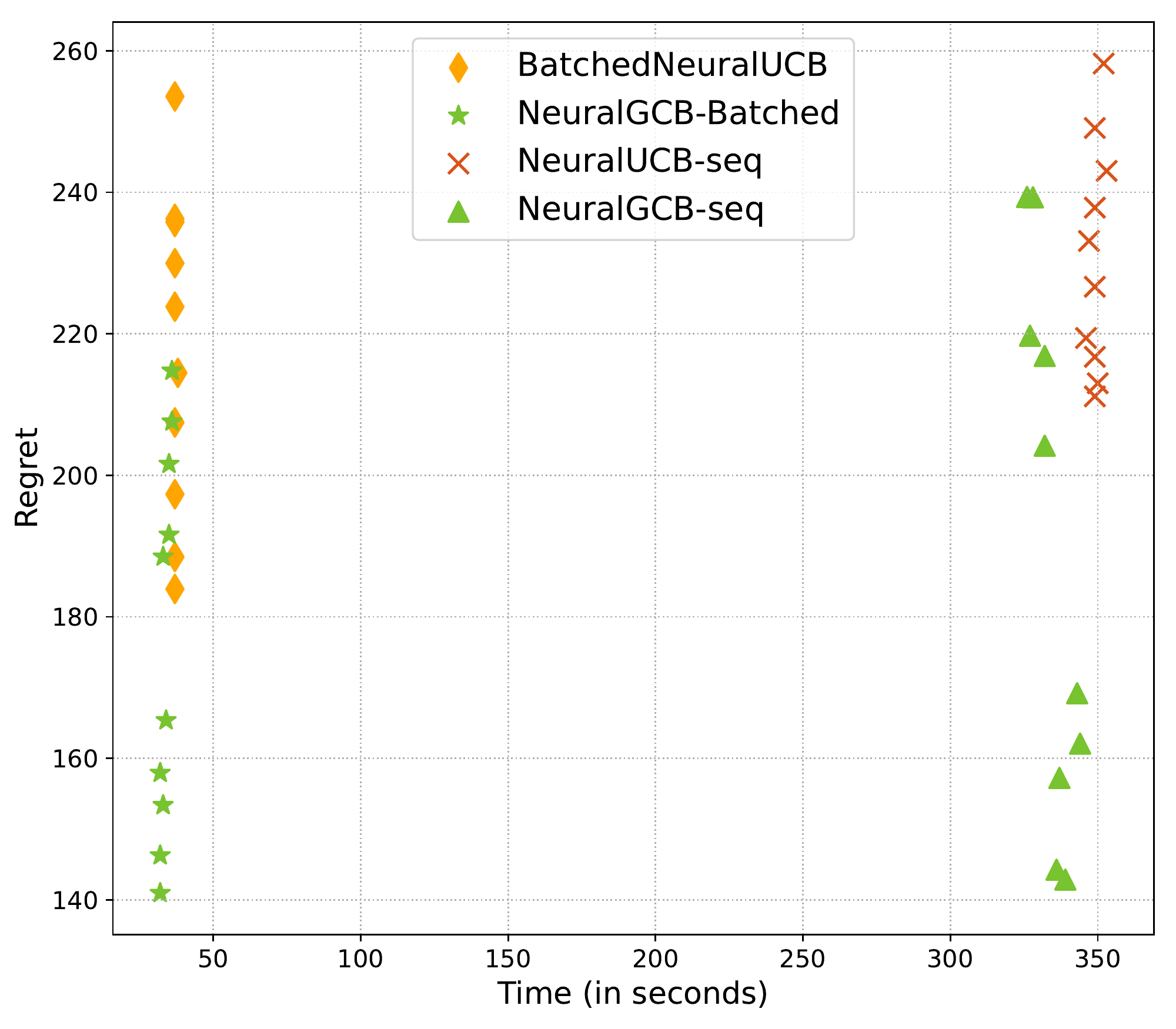}}

\subfloat[Mushroom with $\sigma_1(x)$]{\label{fig:mushroom_s1}\centering \includegraphics[scale = 0.23]{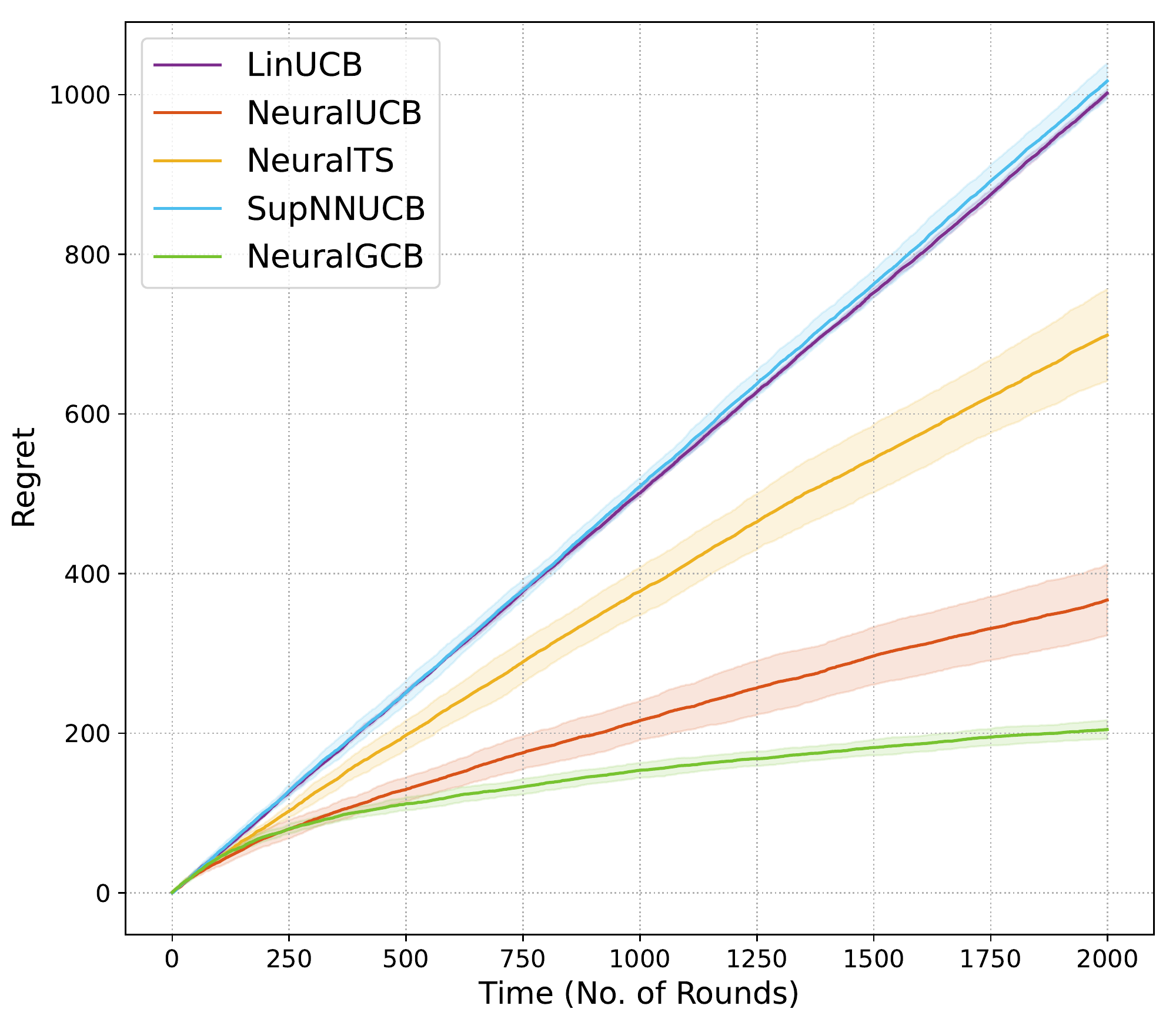}}
~
\subfloat[Mushroom with $\sigma_2(x)$]{\label{fig:mushroom_s2}\centering \includegraphics[scale = 0.23]{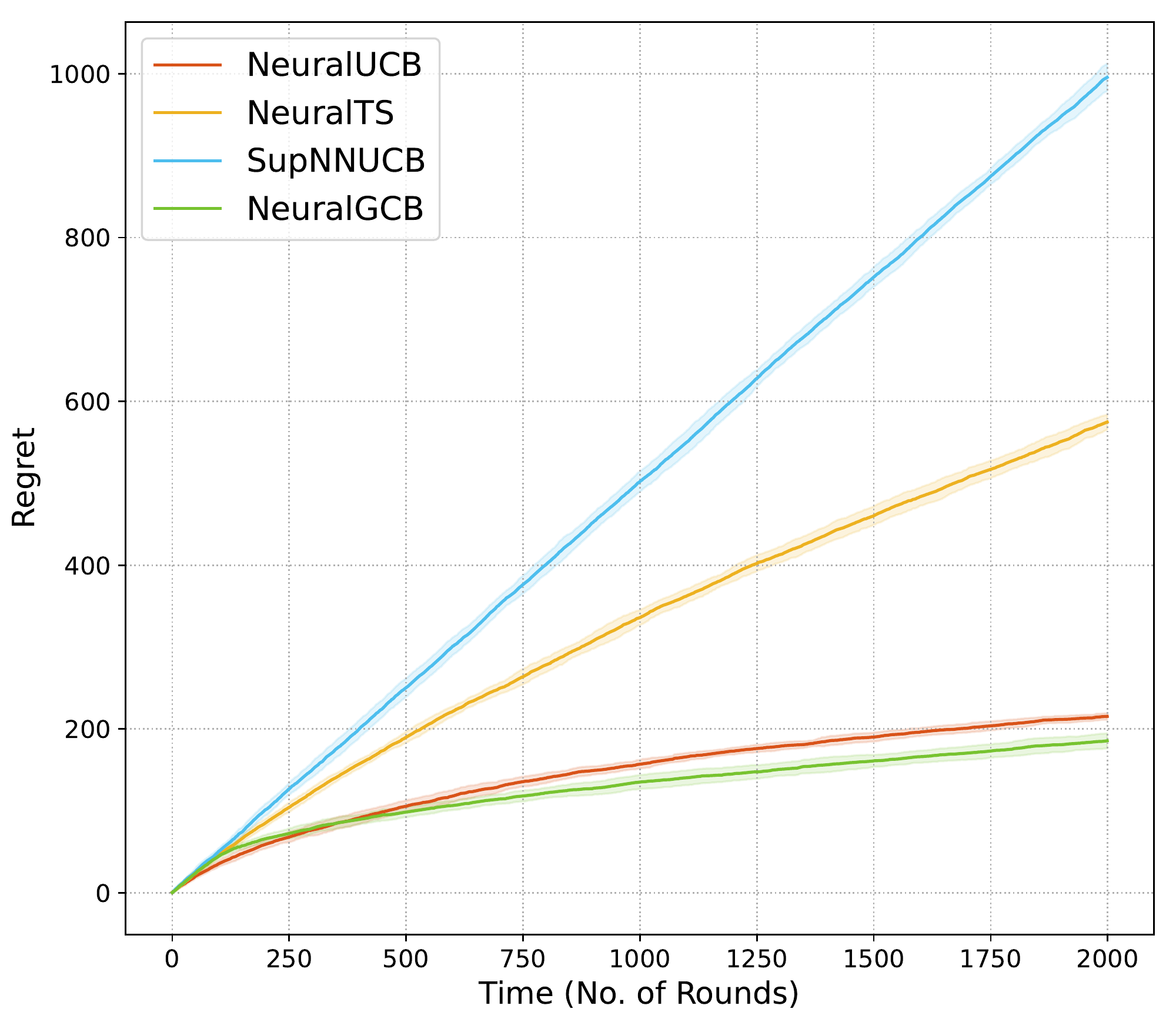}}
~
\subfloat[Batched algorithms on Mushroom]{\label{fig:mushroom_batched} \centering \includegraphics[scale = 0.23]{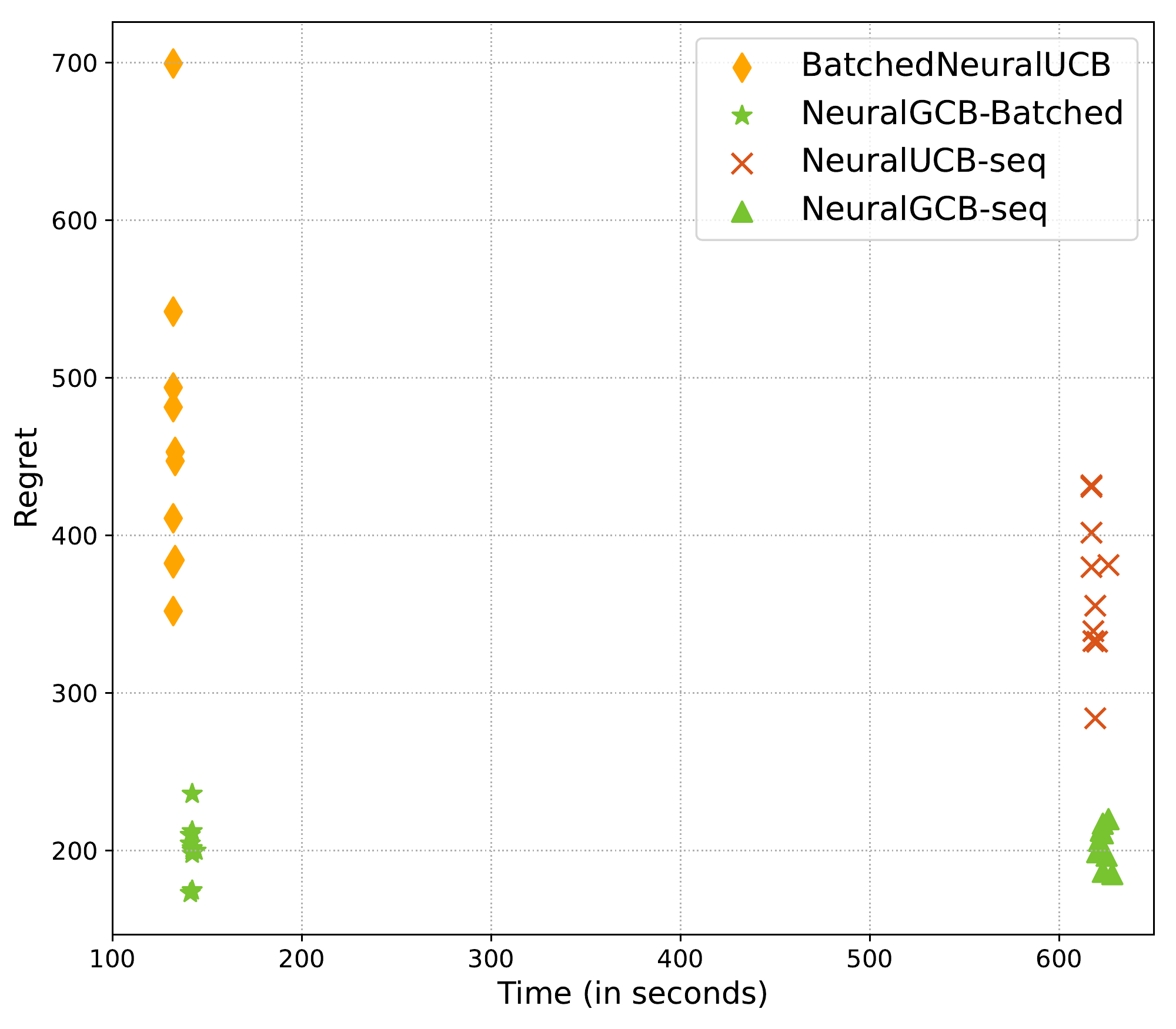}}
\caption{First, second and third rows correspond to the reward functions $h_1(x)$, $h_2(x)$ and the Mushroom dataset, respectively. The two leftmost columns show the cumulative regret incurred by the algorithms against number of steps, with $\sigma_1$ activation functions for the first column and $\sigma_2$ for the second. The rightmost column compares the regret incurred and the time taken (in seconds) for batched and sequential versions of NeuralUCB and NeuralGCB.}
\label{fig:plots_all}
\vspace{-1em}
\end{figure*}

\subsection{Datasets}

For each of the synthetic datasets, we construct a contextual bandit problem with a feature dimension of $d = 10$ and $K = 4$ actions per context running over a time horizon of $T = 2000$ rounds. The set of context vectors $\{\{\bfx_{t, a}\}_{a = 1}^K\}_{t = 1}^T$ are drawn uniformly from the unit sphere. Similar to~\citet{zhou2020neuralUCB}, we consider the following three reward functions:
\begin{align}
    h_1(\bfx) = 4 |\bfa^{\top}\bfx|^2; \quad \ h_2(\bfx) = 4 \sin^2(\bfa^{\top} \bfx); \quad \ h_3(\bfx) = \|\bfA \bfx\|_2.
\end{align}
For the above functions, the vector $\bfa$ is drawn uniformly from the unit sphere and each entry of matrix $\bfA$ is randomly generated from $\cN(0, 0.25)$.

We also consider two real datasets for classification namely Mushroom and Statlog (Shuttle), both of which are available on the UCI repository~\citep{UCI}.
The classification problem is then converted into a contextual bandit problem using techniques outlined in~\cite{Li2010contextualclassification}. Each datapoint in the dataset $(\bfx, y) \in \R^d \times \R$ is transformed into $K$ vectors of the form $\ec{\bfx}{1} = (\bfx, \mathbf{0}, \dots, \mathbf{0}) \dots, \ec{\bfx}{K} = (\mathbf{0}, \dots, \mathbf{0}, \bfx) \in \R^{Kd} $ corresponding to the $K$ actions associated with context $\bfx$. Here $K$ denotes the number of classes in the original classification problem. The reward function is set to $h(\ec{\bfx}{k}) = \1\{y = k\}$, that is, the agent receives a reward of $1$ if they classify the context correctly and $0$ otherwise. We present the results for $h_1(x), h_2(x)$ and the Mushroom dataset in the main paper, and the additional results in the supplementary material.

\subsection{Experimental Setting}

For all the experiments, the rewards are generated by adding zero mean Gaussian noise with a standard deviation of $0.1$ to the reward function. All the experiments are run for a time horizon of $T = 2000$. We report the regret averaged over $10$ Monte Carlo runs with different random seeds. For the real datasets, we shuffle the context vectors for each Monte Carlo run. For all the algorithms, we set the parameter $\nu$ to $0.1$, and $S$, the RKHS norm of the reward, to $4$ for synthetic functions, and $1$ for real datasets. The exploration parameter $\beta_t$ is set to the value prescribed by each algorithm. \\  

We consider a 2 layered neural net for all the experiments as described in Equation~\eqref{eq:nn}. We carry two sets of experiments each with different activation functions, namely $\sigma_1$ (or equivalently, ReLU) and $\sigma_2$. For the experiments with $\sigma_1$ as the activation, we set the number of hidden neurons to $m = 20$ and $m=50$ for synthetic and real datasets respectively. Similarly, for $\sigma_2$, $m$ is set to $30$ and $80$ for synthetic and real datasets, respectively. For all the experiments, we perform a grid search for $\lambda$ and $\eta$ over $\{0.05, 0.1, 0.5\}$ and $\{0.001, 0.01, 0.1\}$, respectively, and choose the best ones for each algorithm. The number of epochs is set to $200$ for synthetic datasets and Mushroom, and to $400$ for Statlog. For the experiments with sequential algorithms, we retrain the neural nets at every step, including NeuralGCB. For Batched NeuralUCB, we use a fixed batch size of $10$ for synthetic datasets, and $20$ for Mushroom. For NeuralGCB we set batch size to $q_r = 5\cdot2^{r-1}$ for synthetic datasets, and $q_r = 5\cdot2^{r+1}$ for Mushroom. More details and additional experiments are given in Appendix~\ref{sec:expts_appendix}.

\subsection{Results}

The two leftmost columns in Fig.~\ref{fig:plots_all}show that NeuralGCB outperforms other algorithms in the case of both synthetic and real world datasets, corroborating the theoretical claims. That holds true for both set of experiments with different activation functions, further bolstering the practical efficiency of the proposed algorithm. In addition, the regret incurred by the algorithms in experiments with $\sigma_2$ as the activation is less than that for the experiments with $\sigma_1$ as the activation, demonstrating the effect of smooth kernels on the performance of the algorithms in practice. \\

In the third column, we compare the regret and running time between the batched and the sequential versions of NeuralUCB and NeuralGCB. We plot the regret incurred against time taken for different training schedules for $10$ different runs. For all functions, the regret incurred by the batched version is comparable to that of the sequential version while having a significantly less running time. Furthermore, NeuralGCB has a smaller regret compared to Batched NeuralUCB for comparable running times.

\section{Conclusion}

In this work, we studied the problem of neural contextual bandits with general set of smooth activation functions. We established non-asymptotic error bounds on the difference between an overparametrized neural net and its corresponding NT kernel along with confidence bounds for prediction using neural nets under this general setting. Furthermore, we proposed a new algorithm that incurs sublinear regret under this general setting is also efficient in practice, as demonstrated by extensive empirical studies.

\bibliography{references}
\bibliographystyle{abbrvnat}


\newpage
 \section*{Appendix}

\appendix

We present the proofs of the theorems and lemmas stated in the main paper, as well as additional empirical results, in the appendix. The appendix is structured as follows: In Appendix ~\ref{sec:thm_1_proof}, we provide the proof of  Theorem~\ref{theorem:appx_error}, while we defer the proof of auxiliary lemmas to Appendix~\ref{sec:helper_lemma_proofs}. Proof of Theorem~\ref{theorem:conf_interval} is given in Appendix ~\ref{sec:thm_2_proof}. Further details on NeuralGCB and proof of Theorem~\ref{theorem:regret_guarantees} are provided in Appendix ~\ref{sec:thm_3_proof}. Lastly, the details on empirical studies and further results are reported in Appendix~\ref{sec:expts_appendix}.

\section{Proof of Theorem~\ref{theorem:appx_error}}
\label{sec:thm_1_proof}

Before we provide the proof of Theorem~\ref{theorem:appx_error}, we first set up some preliminaries and notation that would be useful throughout the proof. 

\subsection{Proof Preliminaries}

\subsubsection{Notation} 
For any $n \in \N$, $[n]$ denotes the set $\{1,2, \dots, n\}$. For any vector $\bfv \in \R^{n}$, $\diag(\bfv)$ is the $\R^{n \times n}$ diagonal matrix with the elements on $\bfv$ on its diagonal entries. $\|\bfv\|$ denotes the $L_2$ norm of the vector $\bfv$. $\|\bfM\|_2$ and $\|\bfM\|_{F}$ denotes the spectral and Frobenius norms respectively of a matrix $\bfM$. For events $A$ and $B$, we define $A \Rightarrow B := \neg A \lor B$. For matrix $\bfA$, we denote the projection matrix for the column space of $\bfA$ by $\bfPi_{\bfA} := \bfA \bfA^{\dagger}$, where $\bfA^{\dagger}$ denotes the pseudo-inverse of $\bfA$. Similarly, we denote the orthogonal projection matrix as $\bfPi_{\bfA}^{\perp} := \bfI - \bfA \bfA^{\dagger}$. For two random variables, $X$ and $Y$, $X \overset{d}{=}_{A} Y$ means $X$ is equal to $Y$ in distribution conditioned on the $\sigma$-algebra generated by the event $A$. For any $\rho \in [-1,1]$, we use $\Sigma_{\rho}$ to denote the matrix $\begin{pmatrix} 1 & \rho \\ \rho & 1 \end{pmatrix}$. Lastly, for $n \in \N$, we define $(2n -1)!! = \prod_{i = 1}^n (2i - 1)$. For example $3!! = 3$ and $5!! = 15$.

\subsubsection{Fully connected Neural Network}

In this proof, we consider a general fully connected neural net consisting of $L$ hidden layers defined recursively as follows:
\begin{align}
    \ec{\bff}{l}(\bfx) = \ec{\bfW}{l}\ec{\bfh}{h-1}(\bfx) \in \R^{d_l}, \quad \ec{\bfh}{l}(\bfx) = \sqrt{\frac{c_{\sigma}}{d_l}} \sigma \left( \ec{\bff}{l}(\bfx)\right) \in \R^{d_l}, \ \ \ \ l = 1,2,\dots, L,
\end{align}
where $\ec{\bfh}{0}(\bfx) = \bfx$ and $\bfx \in \cX$ is the input to the network. In the above expression, $\ec{\bfW}{l} \in \R^{d_l \times d_{l-1}}$ is the weight matrix in the $l^{\text{th}}$ layer for $l \in [L]$ and $d_0 = d$, $\sigma( \cdot): \R \to \R$ is a coordinate-wise activation function and the constant $c_{\sigma} := \left(\E_{z \sim \mathcal{N}(0,1)}[\sigma(z)^2] \right)^{-1}$. The output of the neural network is given by its last layer defined as follows:
\begin{align}
    f(\bfx; \bfW) & = \ec{\bff}{L+1}(\bfx) = \ec{\bfW}{L+1}\ec{\bfh}{L}(\bfx),
\end{align}
where $\ec{\bfW}{L+1} \in \R^{1 \times d_{L}}$ is the weight matrix in the final layer, and $\bfW = (\ec{\bfW}{1},\ec{\bfW}{2}, \dots, \ec{\bfW}{L+1})$ represents all the parameters of the network. Recall that the domain is assumed to be the hypersphere $\mathbb{S}^{d-1}$. Consequently, $\|\bfx\| = 1$ for all $\bfx \in \cX$. This is just the generalization of the of the neural net defined in eqn.~\eqref{eq:nn} with possibly different width in each layer. \\


The partial derivative of the output of the neural network with respect to a particular weight matrix is given as
\begin{align}
    \frac{\partial f(\bfx; \bfW)}{\partial \ec{\bfW}{l}} = \ec{\bfb}{l}(\bfx) \cdot \left( \ec{\bfh}{l-1}(\bfx) \right)^{\top}, \quad l = 1, 2, \dots, L+1,
    \label{eqn:gradient_def}
\end{align}
where $\ec{\bfb}{l}(\bfx)$ is defined recursively as
\begin{align}
    \ec{\bfb}{l}(\bfx) = \begin{cases} 1 & l = L + 1, \\ \sqrt{\dfrac{c_{\sigma}}{d_l}} \ec{\bfD}{l}(\bfx) \left(\ec{\bfW}{l + 1}\right)^{\top} \ec{\bfb}{l + 1}(\bfx) & l = 1,2, \dots, L. \end{cases}
\end{align}
In the above definition,  
\begin{align}
    \ec{\bfD}{l}(\bfx):= \diag\left( \sigma'\left(\ec{\bff}{l}(\bfx)\right) \right) \in \R^{d_l \times d_l},
\end{align}
is a diagonal matrix, where $\sigma'(\cdot)$ is the derivative of the activation function $\sigma$ and is also applied coordinate wise. Note that $\ec{\bfb}{l}(\bfx) \in \R^{d_l}$ for $l =1, 2, \dots, L$ and $\ec{\bfb}{L + 1}(\bfx) \in \R$. In other words, $\ec{\bfb}{l}(\bfx)$ is the gradient of output of the neural network $f(\bfx; \bfW)$ with respect to $\ec{\bff}{l}$, the pre-activation of layer $l$.

\subsection{Neural Tangent Kernel}

In the infinite width limit, the pre-activation functions $\ec{\bff}{l}$ at every hidden layer $l \in [L]$ has all its coordinates tending to i.i.d. centered Gaussian Processes with covariance matrix $\ec{\Sigma}{l-1}: \R^d \times \R^d \to \R$ defined recursively for $l \in [L]$ as:
\begin{align}
    \ec{\Sigma}{0}(\bfx, \bfx') & = \bfx^{\top} \bfx', \nonumber \\
    \ec{\Lambda}{l}(\bfx, \bfx') & = \begin{bmatrix}\ec{\Sigma}{l-1}(\bfx, \bfx) & \ec{\Sigma}{l-1}(\bfx, \bfx')   \\ \ec{\Sigma}{l-1}(\bfx', \bfx)  & \ec{\Sigma}{l-1}(\bfx', \bfx') \\ \end{bmatrix}, \nonumber \\
    \ec{\Sigma}{l}(\bfx, \bfx') & = c_{\sigma} \E_{(u,v) \sim \cN(0, \ec{\Lambda}{l}(\bfx, \bfx')) } [\sigma(u)\sigma(v)].
    \label{eqn:NTK_kernels_def}
\end{align}
Similar to $\ec{\Sigma}{l}(\bfx, \bfx')$, we also define $\ec{\dot{\Sigma}}{l}(\bfx, \bfx')$ as follows:
\begin{align}
    \ec{\dot{\Sigma}}{l}(\bfx, \bfx') & = c_{\sigma} \E_{(u,v) \sim \cN(0, \ec{\Lambda}{l}(\bfx, \bfx')) } [\sigma'(u)\sigma'(v)],
\end{align}
for $l \in [L]$ and $\ec{\dot{\Sigma}}{L +1}(\bfx, \bfx') = 1$ for $\bfx, \bfx' \in \cX$. The final NTK expression for the fully-connected network is given as
\begin{align}
    \ec{\Theta}{L}(\bfx, \bfx') = \sum_{l = 1}^{L + 1} \left( \ec{\Sigma}{l - 1}(\bfx, \bfx') \prod_{j = l}^{L+1} \ec{\dot{\Sigma}}{j}(\bfx, \bfx') \right).
\end{align}
Note that this is same as $k(\bfx, \bfx')$ referred to in the main text.

\subsection{The Activation Function}

In this work, we assume that the activation function $\sigma \in \cA_{\sigma}$, where $\cA_{\sigma} = \{\sigma_s(x): s \in \N \}$ and ${\sigma}_s(x) = (\max(0, x))^s$. Note that ${\sigma}_1(x)$ corresponds to the popular ReLU activation function and existing results hold only for ${\sigma}_1(x)$. 

We state some definitions which we will be later used to establish certain properties of activation functions in $\cA$.

\begin{fact}{(\citet{vakili2021uniform})}
The normalizing constant $c_{\sigma_s} = \dfrac{2}{(2s - 1)!!}$ for all $\sigma_s \in \cA$.
\end{fact}
For simplicity of notation we use $c_s$ instead of $c_{\sigma_s}$ for the rest of the proof.

\begin{definition}
A function $f:\R \to \R$ is said to be $\alpha$-homogeneous, if $f(\lambda x) = \lambda^{\alpha} f(x)$ for all $x \in \R$ and $\lambda > 0$.
\end{definition}
It is straightforward to note that $\sigma_s(x)$ is $s$-homogeneous. \\

Let $\cM_{+}(d)$ denote the set of positive semi-definite matrices of dimension $d$, that is, $\cM_{+} = \{\bfM \in \R^{d \times d} : \bfx^{\top}\bfM \bfx \geq 0 \ \ \forall \ \bfx \in \R^d\}$. Similarly, we use $\cM_{++}(d)$ to denote the class of positive definite matrices of dimension $d$. For $\gamma \in [0,1]$, we denote
\begin{align}
    \cM_{+}^{\gamma} = \left\{ \begin{pmatrix} \Sigma_{11} & \Sigma_{12} \\ \Sigma_{12} & \Sigma_{22} \\ \end{pmatrix} \in \cM_{+}(2) \bigg| 1 - \gamma \leq \Sigma_{11}, \Sigma_{22} \leq 1 + \gamma \right\}.
\end{align}


\begin{definition}
The dual of an activation function $\sigma(\cdot)$ is the function $\bar{\sigma}: [-1, 1] \to \R$ defined as 
\begin{align}
    \bar{\sigma}(\rho) = c_{\sigma} \E_{(X, Y) \sim \cN(0, \Sigma_{\rho})}[\sigma(X) \sigma(Y)].
\end{align}
\end{definition}
Note that from definition of $c_{\sigma}$, $\bar{\sigma}(\rho) \in [-1,1]$ and $\bar{\sigma}(1) = 1$.  \\

\begin{fact}[\citet{Daniely2016}, Lemma 11]
The dual $\bar{\sigma}$ is continuous in $[-1,1]$, smooth in $(-1,1)$, convex in $[0,1]$ and is non-decreasing. 
\end{fact}

The dual of an activation function can be extended to $\breve{\sigma}:\cM_{+}(2) \to \R$ as $\breve{\sigma}(\Sigma) = c_{\sigma} \E_{(X, Y) \sim \cN(0, \Sigma)}[\sigma(X) \sigma(Y)]$. Note that if $\Sigma = \begin{pmatrix} \Sigma_{11} & \Sigma_{12} \\ \Sigma_{12} & \Sigma_{22} \\ \end{pmatrix}$ and $\sigma$ is $k$-homogeneous, we have,
\begin{align*}
    \breve{\sigma}(\Sigma) = (\Sigma_{11} \Sigma_{22})^{k/2} \bar{\sigma} \left( \frac{\Sigma_{12}}{\sqrt{\Sigma_{11} \Sigma_{22}}} \right).
\end{align*}

Let $\bar{\sigma}_s$ be the dual function of $\sigma_s \in \cA_{\sigma}$ for all $s \in \N$ and $\bar{\cA}_{\sigma} = \{\bar{\sigma}_{s}: s \in \N\}$ denote the set of dual functions. It is not difficult to note that $\bar{\sigma}_s(-1) = 0$ and $\bar{\sigma}_s(1) = 1$ for all $s \in \N$. Since $\bar{\sigma}_s$ is non-decreasing, $\bar{\sigma}_s(\rho) \in [0,1]$ for all $\rho \in [-1, 1]$.

\begin{fact}[\citet{vakili2021uniform}, Lemma 1]
The functions in $\bar{\cA}_{\sigma}$ satisfy,
\begin{align}
    \bar{\sigma}_s'(\rho) = \frac{s^2}{2s - 1} \bar{\sigma}_{s-1}(\rho) \label{eqn:dual_derivative}
\end{align}
for $s > 1$. Here $\bar{\sigma}_s'$ denotes the derivative of $\bar{\sigma}_s$. 
\end{fact}

Consequently, $|\bar{\sigma}_s'(\rho)| \leq \dfrac{s^2}{2s - 1} |\bar{\sigma}_{s-1}(\rho)| \leq \dfrac{s^2}{2s - 1}$. Thus, $\bar{\sigma}_s$ is $s^2/(2s - 1)$ Lipschitz.

\begin{definition}
For any function $\sigma_s \in \cA_{\sigma}$, we define $\mu_{s, \rho} := \E_{(X, Y) \sim \cN(0, \Sigma_{\rho})}[\sigma_s(X) \sigma_s(Y)] = \dfrac{\bar{\sigma}_s}{c_{s}}$.
\end{definition}


\subsection{Proof of Lemma~\ref{lemma:kernel_init_to_NTK}}

The central piece in the proof of Theorem~\ref{theorem:appx_error} is Lemma~\ref{lemma:kernel_init_to_NTK}. We focus our attention on first proving Lemma~\ref{lemma:kernel_init_to_NTK}. Since the proof is involved, we first provide an outline of the proof to give the reader an overview of the approach before delving into the technical details.

Informally, Lemma~\ref{lemma:kernel_init_to_NTK} states that for sufficiently large network widths, the following relation holds with probability of at least $1 - \delta$ over the random initialization of the network weights.
\begin{align*}
    \left| \ip{\frac{\partial f( \bfx; \bfW)}{\partial \bfW}}{\frac{\partial f(\bfx'; \bfW)}{\partial \bfW}} - \ec{\Theta}{L}(\bfx, \bfx') \right| \leq (L + 1)\varepsilon.
\end{align*}
Firstly, note that
\begin{align*}
    \ip{\frac{\partial f(\bfx; \bfW)}{\partial \bfW}}{\frac{\partial f(\bfx'; \bfW)}{\partial \bfW}} = \sum_{l = 1}^{L + 1} \ip{\frac{\partial f(\bfx; \bfW)}{\partial \ec{\bfW}{l}}}{\frac{\partial f(\bfx'; \bfW)}{\partial \ec{\bfW}{l}}}
\end{align*}
and recall that 
\begin{align*}
    \ec{\Theta}{L}(\bfx, \bfx') = \sum_{l = 1}^{L + 1} \left( \ec{\Sigma}{l - 1}(\bfx, \bfx') \prod_{j = l}^{L+1} \ec{\dot{\Sigma}}{j}(\bfx, \bfx') \right).
\end{align*}
Using these relations, note that it is sufficient to show that, 
\begin{align}
    \left| \ip{\frac{\partial f(\bfx; \bfW)}{\partial \ec{\bfW}{l}}}{\frac{\partial f(\bfx'; \bfW)}{\partial \ec{\bfW}{l}}} - \ec{\Sigma}{l - 1}(\bfx, \bfx') \prod_{j = l}^{L+1} \ec{\dot{\Sigma}}{j}(\bfx, \bfx') \right| \leq \varepsilon
    \label{eqn:proof_sketch_1}
\end{align}
holds for all $l \in [L]$ with probability $1 - \delta$. Furthermore, we have, 
\begin{align*}
    \ip{\frac{\partial f(\bfx; \bfW)}{\partial \ec{\bfW}{l}}}{\frac{\partial f(\bfx'; \bfW)}{\partial \ec{\bfW}{l}}} & = \ip{\ec{\bfb}{l}(\bfx) \cdot \left( \ec{\bfh}{l-1}(\bfx) \right)^{\top}}{\ec{\bfb}{l}(\bfx') \cdot \left( \ec{\bfh}{l-1}(\bfx') \right)^{\top}} \\
    & = \ip{\ec{\bfh}{l-1}(\bfx)}{\ec{\bfh}{l-1}(\bfx)'} \ip{\ec{\bfb}{l}(\bfx)}{\ec{\bfb}{l}(\bfx')}.
\end{align*}

The proof revolves around establishing $\ip{\ec{\bfh}{l-1}(\bfx)}{\ec{\bfh}{l-1}(\bfx')}$ is close to $\ec{\Sigma}{l - 1}(\bfx, \bfx')$ while $\ip{\ec{\bfb}{l}(\bfx)}{\ec{\bfb}{l}(\bfx')}$ is close to $\prod_{j = l}^{L+1} \ec{\dot{\Sigma}}{j}(\bfx, \bfx')$. On combining these two relations, we obtain the result in~\eqref{eqn:proof_sketch_1} and consequently prove the theorem.

Throughout the proof, we fix some $s \in \N$ and hence $\sigma = \sigma_s$. Recall from equation~\eqref{eqn:NTK_kernels_def}, we have,
\begin{align*}
    \ec{\Sigma}{0}(\bfx, \bfx') & = \bfx^{\top} \bfx', \nonumber \\
    \ec{\Lambda}{l}(\bfx, \bfx') & = \begin{bmatrix}\ec{\Sigma}{l-1}(\bfx, \bfx) & \ec{\Sigma}{l-1}(\bfx, \bfx')   \\ \ec{\Sigma}{l-1}(\bfx', \bfx)  & \ec{\Sigma}{l-1}(\bfx', \bfx') \\ \end{bmatrix}, \nonumber \\
    \ec{\Sigma}{l}(\bfx, \bfx') & = c_{s} \E_{(u,v) \sim \cN(0, \ec{\Lambda}{l}(\bfx, \bfx')) } [\sigma_s(u)\sigma_s(v)].
\end{align*}
Since $\|\bfx\| = 1$ for all $x \in \cX$, $\ec{\Sigma}{0}(\bfx, \bfx) = 1$ for all $x \in \cX$. Using induction, we can establish that $\ec{\Sigma}{l}(\bfx, \bfx) = 1$ for all $x \in \cX$, for all $l \in \{0,1,2, \dots, L\}$. The base case follows immediately as $\ec{\Sigma}{0}(\bfx, \bfx) = 1$ for all $x \in \cX$. Assume true for $l - 1$. Consequently, we have, $\ec{\Lambda}{l}(\bfx, \bfx) = \begin{bmatrix} 1 & 1 \\ 1& 1 \\ \end{bmatrix}$. On plugging this value in the definition of $\ec{\Sigma}{l}(\bfx, \bfx)$, we obtain $\ec{\Sigma}{l}(\bfx, \bfx) = 1$, completing the inductive step. As a result, $|\ec{\Sigma}{l}(\bfx, \bfx')| \leq 1$ for all $\bfx, \bfx' \in \cX$ and hence we can write $\ec{\Sigma}{l}(\bfx, \bfx') = \bar{\sigma}_s(\ec{\Sigma}{l - 1}(\bfx, \bfx'))$. \\

As the final step before the proof, we define a sequence of events which will be used throughout the proof. 
Let $\ec{\bfDelta}{l}(\bfx, \bfx') := \ec{\bfD}{l}(\bfx)\ec{\bfD}{l}(\bfx')$. We define the following events:
\begin{itemize}
    \item $\displaystyle \cA^{l}(\bfx, \bfx', \varepsilon_1) = \left\{ \left| \left(\ec{\bfh}{l}(\bfx) \right)^{\top}\ec{\bfh}{l}(\bfx') - \ec{\Sigma}{l}(\bfx, \bfx')\right| \leq \varepsilon_1\right\}$,
    \item $\displaystyle \bar{\cA}^l(\bfx, \bfx', \varepsilon_1) = \cA^{l}(\bfx, \bfx', \varepsilon_1) \cap \cA^{l}(\bfx, \bfx, \varepsilon_1) \cap \cA^{l}(\bfx', \bfx', \varepsilon_1)$
    \item $\displaystyle \bar{\cA}(\bfx, \bfx', \varepsilon_1) = \bigcap_{l = 0}^{L} \bar{\cA}^l(\bfx, \bfx', \varepsilon_1)$
    \item $\displaystyle \cB^{l}(\bfx, \bfx', \varepsilon_2) = \left\{ \left| \ip{\ec{\bfb}{l}(\bfx)}{\ec{\bfb}{l}(\bfx')} - \prod_{j = l}^{L+1} \ec{\dot{\Sigma}}{j}(\bfx, \bfx')\right| \leq \varepsilon_2 \right\} $
    \item $\displaystyle \bar{\cB}^l(\bfx, \bfx', \varepsilon_2) = \cB^{l}(\bfx, \bfx', \varepsilon_2) \cap \cB^{l}(\bfx, \bfx, \varepsilon_2) \cap \cB^{l}(\bfx', \bfx', \varepsilon_2)$
    \item $\displaystyle \bar{\cB}(\bfx, \bfx', \varepsilon_2) = \bigcap_{l = 1}^{L + 1} \bar{\cB}^l(\bfx, \bfx', \varepsilon_2)$
    \item $\displaystyle \bar{\cC}(\bfx, \bfx', \varepsilon_3) = \{|f(\bfx; \bfW)| \leq \varepsilon_3, |f( \bfx'; \bfW)| \leq \varepsilon_3\} $
    \item $\displaystyle \cD^{l}(\bfx, \bfx', \varepsilon_4) = \left\{ \left|c_s \dfrac{\tr(\ec{\bfDelta}{l}(\bfx, \bfx'))}{d_l} -  \ec{\dot{\Sigma}}{l}(\bfx, \bfx')\right| < \varepsilon_4 \right\}$
    \item $\displaystyle \bar{\cD}^l(\bfx, \bfx', \varepsilon_4) = \cD^{l}(\bfx, \bfx', \varepsilon_4) \cap \cD^{l}(\bfx, \bfx, \varepsilon_4) \cap \cD^{l}(\bfx', \bfx', \varepsilon_4)$
    \item $\displaystyle \bar{\cD}(\bfx, \bfx', \varepsilon_4) = \bigcap_{l = 1}^{L + 1} \bar{\cD}^l(\bfx, \bfx', \varepsilon_4)$
    \item $\displaystyle \cE^{l}(\bfx, \bfx', \varepsilon_5) = \left\{ \|\ec{\bfDelta}{l}(\bfx, \bfx')\|_2 < \varepsilon_5 \right\}$
    \item $\displaystyle \bar{\cE}^l(\bfx, \bfx', \varepsilon_5) = \cE^{l}(\bfx, \bfx', \varepsilon_4) \cap \cE^{l}(\bfx, \bfx, \varepsilon_4) \cap \cE^{l}(\bfx', \bfx', \varepsilon_5)$
    \item $\displaystyle \bar{\cE}(\bfx, \bfx', \varepsilon_5) = \bigcap_{l = 1}^{L + 1} \bar{\cE}^l(\bfx, \bfx', \varepsilon_5)$
\end{itemize}

We also state the following two Lemmas taken from~\citet{arora2019exact} which are used at several points in the proof before beginning with the first part.

\begin{lemma}
For any two events, $A$ and $B$, $\Pr(A\Rightarrow B) \geq \Pr(B|A)$. \label{lemma:pr_a_implies_b}
\end{lemma} 

\begin{lemma}
Let $\bfw \sim \cN(0, \bfI_d)$, $\bfG \in \R^{d \times k}$ be a fixed matrix and $\bfF = \bfw^{\top}\bfG$ be a random matrix. Then, conditioned on the value of $\bfF$, $\bfw$ remains Gaussian in the null space of the row space of $\bfG$. Mathematically,
\begin{align*}
    \bfPi_{\bfG}^{\perp}\bfw \overset{d}{=}_{\bfF = \bfw^{\top}\bfG} \bfPi_{\bfG}^{\perp}\tilde{\bfw},
\end{align*}
where $\tilde{\bfw}$ is a i.i.d. copy of $\bfw$. \label{lemma:nullspace_equivalence}
\end{lemma}

\subsubsection{Pre-activations are close to the NTK matrix}

In the first step, we show that $\ip{\ec{\bfh}{l-1}(\bfx)}{\ec{\bfh}{l-1}(\bfx')}$ is close to $\ec{\Sigma}{l - 1}(\bfx, \bfx')$. We formalize this idea in the following lemma.

\begin{lemma}
For $\sigma(z) = \sigma_s(z)$ and $[\ec{\bfW}{l}]_{ij} \overset{\text{i.i.d.}}{\sim} \cN(0,1)$ for all $i \in [d_{l+1}], j\ \in [d_l]$ and $l \in \{0,1, 2,\dots, L\}$. There exist constants $c_1, c_2 > 0$ such that if $\min_{l} d_l \geq c_1 \frac{s^L L^2c_s^2}{\varepsilon^2}\log^2(sL/\delta \varepsilon) $ and $\varepsilon \leq c_2/s$, then for fixed ${\bfx}, {\bfx}' \in \cX$, 
\begin{align*}
    \left| \ip{\ec{\bfh}{l}(\bfy)}{\ec{\bfh}{l}(\bfy')} - \ec{\Sigma}{l}(\bfy, \bfy') \right| \leq \varepsilon,
\end{align*}
holds with probability at least $1 - \delta$ for all $l \in \{0,1,\dots,L\}$ and all $(\bfy, \bfy') \in \left\{ ({\bfx}, {\bfx}), ({\bfx}, {\bfx}'), ({\bfx}', {\bfx}') \right\}$. In other words, $\Pr(\bar{\cA}(\varepsilon)) \geq 1 - \delta$, for fixed ${\bfx}, {\bfx}'$. 
\label{theorem:gradient_approx_theorem}
\end{lemma}

\begin{proof}
The proof of this Lemma relies on following lemmas.

\begin{lemma}
Let $\rho \in [-1,1]$ and $(X, Y) \sim \cN(0, \Sigma_{\rho})$. If $(X_1, Y_1), (X_2, Y_2), \dots, (X_n, Y_n)$ are independent samples from $\cN(0, \Sigma_{\rho})$ and $S_n = \frac{1}{n} \sum_{i =1}^n \sigma_s(X_i) \sigma_s(Y_i)$ for some $\sigma_s \in \cA$, then for any $t \geq 0$,
\begin{align*}
    \Pr(S_n -\mu_s \geq t) & \leq \begin{cases} (n + 1) \exp\left(- \dfrac{nt^2}{2 \sqrt{3 \mu_{4s,\rho}}} \right) & \text{ if } t \leq t^*(n), \\ (n + 1) \exp\left(- \dfrac{(nt)^{1/s}}{8(1 + \rho)} \right) & \text{ if } t > t^*(n). 
    \end{cases} \\
    \Pr(S_n - \mu_s \leq -t) &  \leq  \exp\left(  - \frac{n t^2}{2\mu_{2s, \rho}} \right),
\end{align*}
where $t^*(n) = \left( \dfrac{\sqrt{3 \mu_{4s,\rho}}}{4(1 + \rho)} \right)^{2 - 1/s} n^{-\left(\frac{s - 1}{2s - 1} \right)}$. Consequently, if $n \geq n_{*}(s, \delta, \rho)$, then
\begin{align*}
    \Pr \left( |S_n - \mu_{s, \rho}| \geq \sqrt{\frac{2(\sqrt{3 \mu_{4s,\rho}} + \mu_{2s,\rho})}{n} \log \left(\frac{n + 1}{\delta} \right)} \right) \leq \delta,
\end{align*}
where $\displaystyle n_{*}(s, \delta, \rho) := \min\left\{m \in \N: m \geq \frac{(8(1+ \rho))^{2s}}{2 \sqrt{3 \mu_{4s, \rho}}} \left[ \log \left( \frac{m + 1}{\delta} \right)\right]^{2s - 1}\right\} + 1$ \label{lemma:concentration_of_activation_fn}
\end{lemma}

For simplicity of notation, we define $\varphi_s(n, \delta, \rho) := \sqrt{\frac{2(\sqrt{3 \mu_{4s,\rho}} + \mu_{2s,\rho})}{n} \log \left(\frac{n + 1}{\delta} \right)}$. Thus, the result of Lemma~\ref{lemma:concentration_of_activation_fn} can be restated as $\Pr(|S_n - \mu_{s, \rho}| \geq \varphi_s(n, \delta, \rho)) \leq \delta$.

\begin{lemma}
For a given $s \in \N$, the dual activation function $\breve{\sigma}_s$ is $\beta_s$-Lipschitz in $\cM_{+}^{\gamma_s}$ w.r.t. the $\infty$-norm for $\gamma_s \leq 1/s$ and $\beta_s \leq 6s$.
\label{lemma:decency_parameters}
\end{lemma}

\begin{lemma}[\cite{Daniely2016}]
For $\sigma(z) = \sigma_s(z)$ and $[\ec{\bfW}{l}]_{ij} \overset{\text{i.i.d.}}{\sim} \cN(0,1)$ for all $i \in [d_{l+1}], j\ \in [d_l]$ and $l \in \{0,1, 2,\dots, L\}$. If $\min_{l \in [L]} d_l \geq n_s^*\left( \frac{\varepsilon}{B_{L,s}}, \frac{\delta}{8 \bar{d}} \right)$, then for fixed ${\bfx}, {\bfx}' \in \cX$, 
\begin{align*}
    \left| \ip{\ec{\bfh}{l}(\bfy)}{\ec{\bfh}{l}(\bfy')} - \ec{\Sigma}{l}(\bfy, \bfy') \right| \leq \varepsilon,
\end{align*}
holds with probability at least $1 - \delta$ for all $l \in \{0,1,\dots,L\}$ and all $(\bfy, \bfy') \in \left\{ ({\bfx}, {\bfx}), ({\bfx}, {\bfx}'), ({\bfx}', {\bfx}') \right\}$.  In the above expression, $B_{L, s} = \sum_{j = 0}^{L - 1}\beta_s^i$, $\bar{d} = \sum_{l = 1}^{L+1} d_l$ and $n^*_s(\varepsilon, \delta) = \min \{ n : \varphi_s(n,\delta) \leq \varepsilon\}$.\label{lemma:daniely_error_lemma}
\end{lemma}

Lemma~\ref{theorem:gradient_approx_theorem} follows immediately from Lemma~\ref{lemma:daniely_error_lemma} which in turn follows from the previous lemmas. The proofs of Lemma~\ref{lemma:concentration_of_activation_fn} and~\ref{lemma:decency_parameters} are provided in Appendix~\ref{sec:helper_lemma_proofs}.

\end{proof}

\subsubsection{Pre-activation gradients are close to NTK derivative matrices}

In the second step, we show that $\ip{\ec{\bfb}{l}(\bfx)}{\ec{\bfb}{l}(\bfx')}$ is close to $\prod_{j = l}^{L+1} \ec{\dot{\Sigma}}{j}(\bfx, \bfx')$. We formalize this idea in the following lemma.

\begin{lemma}
For $\sigma = \sigma_s$ and $[\ec{\bfW}{l}]_{ij} \overset{\text{i.i.d.}}{\sim} \cN(0,1)$ for all $i \in [d_{l+1}], j\ \in [d_l]$ and $l \in \{0,1, 2,\dots, L\}$. There exist constants $c_1, c_2 > 0$ such that if $\min_{l} d_l \geq c_1 \frac{s^L L^2c_s^2}{\varepsilon^2}\log^2(sL/\delta \varepsilon) $ and $\varepsilon \leq c_2 s^{-L+2}/L$, then for fixed ${\bfx}, {\bfx}' \in \cX$, 
\begin{align*}
    \left| \ip{\ec{\bfb}{l}(\bfy)}{\ec{\bfb}{l}(\bfy')} - \prod_{j = l}^{L+1} \ec{\dot{\Sigma}}{j}(\bfy, \bfy') \right| \leq L(\beta_s + 1)s^{L+1}\varepsilon,
\end{align*}
holds with probability at least $1 - \delta$ for all $l \in \{0,1,\dots,L\}$ and all $(\bfy, \bfy') \in \left\{ ({\bfx}, {\bfx}), ({\bfx}, {\bfx}'), ({\bfx}', {\bfx}') \right\}$. In other words, if $\min_{l} d_l \geq c_1 \frac{s^L L^2c_s^2}{\varepsilon_1^2}\log^2(sL/\delta \varepsilon_1) $ and $\varepsilon_1 \leq c_2 s^{-L+2}/L$, then for fixed ${\bfx}, {\bfx}'$, $\Pr\left(\bar{\cA}(\varepsilon_1/2) \land \bar{\cB}(L(\beta_s + 1)s^{L+1}\varepsilon_1) \right) \geq 1 - \delta$. \label{theorem:ntk_derivative_approx_theorem}
\end{lemma}

\begin{proof}

We first state some helper lemmas that will be used in the proof followed by the proof of the above lemma.

\begin{lemma}[\citet{arora2019exact}, Lemma E.4]
\begin{align*}
    \Pr\left[ \bar{\cA}^{L}(\varepsilon/2) \Rightarrow \bar{\cC}\left(2 \sqrt{\log \left( \frac{4}{\delta} \right)} \right) \right] \geq 1 - \delta \ \ \quad \forall \ \varepsilon \in [0,1], \ \ \forall \ \delta \in (0, 1). 
\end{align*}\label{lemma:a_l_implies_c}
\end{lemma}

\begin{lemma}
If $\bar{\cA}^{L}(\varepsilon_1/2) \land \bar{\cB}^{l+1}(\varepsilon_2) \land \bar{\cC}(\varepsilon_3) \land \bar{\cD}^l(\varepsilon_4) \land \bar{\cE}^l(\varepsilon_5)$ for any pair of points $\bfx, \bfx' \in \cX$, then for any $(\bfy, \bfy') \in \{ (\bfx, \bfx), (\bfx, \bfx'), (\bfx', \bfx') \}$, we have
\begin{align*}
    \left| c_{s} \frac{\tr(\ec{\bfDelta}{l}(\bfy, \bfy'))}{d_l}\ip{\ec{\bfb}{l+1}(\bfy)}{\ec{\bfb}{l+1}(\bfy')}  - \prod_{j = l}^L \ec{\dot{\Sigma}}{l}(\bfy, \bfy') \right| \leq s^{L - l}\varepsilon_4 + s\varepsilon_2.
\end{align*} \label{lemma:trace_D_times_b}
\end{lemma}

\begin{lemma}
If $\bar{\cA}^{L}(\varepsilon_1/2) \land \bar{\cB}^{l+1}(\varepsilon_2) \land \bar{\cC}(\varepsilon_3) \land \bar{\cD}^l(\varepsilon_4) \land \bar{\cE}^l(\varepsilon_5)$ for any pair of points $\bfx, \bfx' \in \cX$, then for any $(\bfy, \bfy') \in \{ (\bfx, \bfx), (\bfx, \bfx'), (\bfx', \bfx') \}$
\begin{align*}
    & \left| \frac{c_s}{d_l}\left(\ec{\bfb}{l+1}(\bfy)\right)^{\top} \ec{\bfW}{l+1}\bfPi_{\bfG}^{\perp} \ec{\bfDelta}{l}(\bfy, \bfy') \bfPi_{\bfG}^{\perp}  \left(  \ec{\bfW}{l+1}\right)^{\top}\ec{\bfb}{l+1}(\bfy) - \frac{c_s}{d_l} \tr(\ec{\bfDelta}{l}(\bfy, \bfy')) \ip{\ec{\bfb}{l+1}(\bfy)}{\ec{\bfb}{l+1}(\bfy')}   \right|  \\
    &\leq c_s \left(\ip{\ec{\bfb}{l+1}(\bfy)}{\ec{\bfb}{l+1}(\bfy)} + \ip{\ec{\bfb}{l+1}(\bfy')}{\ec{\bfb}{l+1}(\bfy')}\right) \left( \sqrt{\frac{2}{d_l} \log\left( \frac{3}{\delta} \right) } + \frac{1}{d_l}\log\left( \frac{3}{\delta} \right) \right) \varepsilon_5 \\
    & \ \ \ \ \ \ \ \ \ \ \ \ +  \frac{2c_s}{d_l}\ip{\ec{\bfb}{l+1}(\bfy)}{\ec{\bfb}{l+1}(\bfy')} \varepsilon_5.
\end{align*}
holds with probability at least $1 - \delta$. Consequently, for any $\bfy \in \{\bfx, \bfx'\}$,
\begin{align*}
    \sqrt{\frac{c_s}{d_l}} \left\| \left(\ec{\bfb}{l+1}(\bfy)\right)^{\top} \ec{\bfW}{l+1}\bfPi_{\bfG}^{\perp} \ec{\bfD}{l}(\bfy) \right\| & \leq \sqrt{2c_s \ip{\ec{\bfb}{l+1}(\bfy)}{\ec{\bfb}{l+1}(\bfy)} \varepsilon_5}.
\end{align*} \label{lemma:b_h_proj_G_close_to_trace_D_times_b}
\end{lemma}

\begin{lemma}
If $\bar{\cA}^{L}(\varepsilon_1/2) \land \bar{\cB}^{l+1}(\varepsilon_2) \land \bar{\cC}(\varepsilon_3) \land \bar{\cD}^l(\varepsilon_4) \land \bar{\cE}^l(\varepsilon_5)$ for any pair of points $\bfx, \bfx' \in \cX$, then for any $(\bfy, \bfy') \in \{ (\bfx, \bfx), (\bfx, \bfx'), (\bfx', \bfx') \}$
\begin{align*}
    & \frac{c_s}{d_l} \bigg|\left(\ec{\bfb}{l+1}(\bfy) \right)^{\top} \ec{\bfW}{l+1} \ec{\bfDelta}{l}(\bfy, \bfy') \left(\ec{\bfW}{l+1}\right)^{\top} \ec{\bfb}{l+1}(\bfy') \\
    & \ \ \ \ \ \ \ \ \ \ \ \ \ \ \ \ - \left(\ec{\bfb}{l+1}(\bfy) \right)^{\top} \ec{\bfW}{l+1} \bfPi_{\bfG}^{\perp}\ec{\bfDelta}{l}(\bfy, \bfy')\bfPi_{\bfG}^{\perp} \left(\ec{\bfW}{l+1}\right)^{\top} \ec{\bfb}{l+1}(\bfy')\bigg| \leq \\ 
    & \frac{c_s \varepsilon_5}{\sqrt{d_l}}\left[(s^{(L - l)/2} + \varepsilon_2)\sqrt{2 \log\left(\frac{8}{\delta}\right)} + s^{L - l} \varepsilon_3 \sqrt{2}\right] \times \\
    & \ \ \ \ \ \ \ \ \ \ \ \ \ \ \ \ \ \left\{\frac{1}{\sqrt{d_l}}\left[(s^{(L - l)/2} + 1)\sqrt{2 \log\left(\frac{8}{\delta}\right)} + s^{L - l} \varepsilon_3 \sqrt{2}\right] + \sqrt{8}(s^{(L - l)/2} + 1) \right\}
\end{align*}
holds with probability at least $1 - \delta$.\label{lemma:b_W_Delta_diff_with_its_proj}
\end{lemma}

\begin{lemma} For any $\varepsilon_1 \in [0,1/s]$, $\varepsilon_2, \varepsilon_4 \in [0,1]$ and $ \varepsilon_3, \varepsilon_5 \geq 0$,
\begin{align*}
    \Pr \left[\bar{\cA}^{L}(\varepsilon_1/2) \land \bar{\cB}^{l+1}(\varepsilon_2) \land \bar{\cC}(\varepsilon_3) \land \bar{\cD}^l(\varepsilon_4) \land \bar{\cE}^l(\varepsilon_5) \Rightarrow \bar{\cB}^{l}(\varepsilon') \right] \geq 1 - \delta,
\end{align*}
where 
\begin{align*}
    &\varepsilon' =  2c_s (s^{L - l} + 1)\left( \sqrt{\frac{2}{d_l} \log\left( \frac{6}{\delta} \right) } + \frac{1}{d_l}\log\left( \frac{6}{\delta} \right) \right) \varepsilon_5  +  \frac{2c_s}{d_l}(s^{L - l} + 1) \varepsilon_5 + s^{L - l}\varepsilon_4 + s\varepsilon_2 \\
    & + \frac{c_s \varepsilon_5}{\sqrt{d_l}}\left[(s^{(L - l)/2} + 1)\sqrt{2 \log\left(\frac{16}{\delta}\right)} + s^{L - l} \varepsilon_3 \sqrt{2}\right] \times \\
    & \ \ \ \ \ \ \ \ \ \ \ \left\{\frac{1}{\sqrt{d_l}}\left[(s^{(L - l)/2} + 1)\sqrt{2 \log\left(\frac{16}{\delta}\right)} + s^{L - l} \varepsilon_3 \sqrt{2}\right]  + \sqrt{8}(s^{(L - l)/2} + 1) \right\}.
\end{align*}\label{lemma:b_h_to_b_h_plus_1}
\end{lemma}

We now prove Lemma~\ref{theorem:ntk_derivative_approx_theorem} by using induction on Lemma~\ref{lemma:b_h_to_b_h_plus_1}. Before the inductive step, we note some relations that will be useful to us during the analysis. Firstly, using Lemma~\ref{theorem:gradient_approx_theorem}, we can conclude that if $d_l \geq n_s^*(\varepsilon/B_{L,s}, \delta/32\bar{d}(L+1))$ and $\varepsilon \leq c_2/s$, then
\begin{align}
    \Pr\left[ \bar{\cA}^{L}(\varepsilon/2) \right] \geq 1 - \delta/4. \label{eqn:thm_2_proof_prob_bound_A_L}
\end{align}
From Lemma~\ref{lemma:a_l_implies_c}, we can conclude that
\begin{align}
    \Pr\left[ \bar{\cA}^{L}(\varepsilon/2) \Rightarrow \bar{\cC}\left(2 \sqrt{\log \left( \frac{16}{\delta} \right)} \right) \right] \geq 1 - \delta/4. \label{eqn:thm_2_proof_prob_bound_C}
\end{align}
If $d_l \geq n_s^*(\varepsilon/3, \delta/24(L+1))$, then we can conclude that $3s^2 \varphi_{s-1}((\min_l d_l), \delta/24(L+1)) + \frac{s^2\beta_s \varepsilon}{2s - 1} \leq (\beta + 1)s^2\varepsilon$. Consequently for $\varepsilon \leq 2/s$, 
\begin{align}
    \Pr\left[ \bar{\cA}(\varepsilon/2) \Rightarrow \bar{\cD}((\beta_s + 1)s^2\varepsilon) \land \bar{\cE}\left(\varepsilon''\right) \right] \geq 1 - \delta/4,\label{eqn:thm_2_proof_a_l_implies_d_e}
\end{align} 
where $\displaystyle \varepsilon'' = 3s^2 \left[2\log\left(\frac{6(L+1) (\max_{l} d_l) }{\delta}\right) \right]^{s-1}$. Applying the union bound on~\eqref{eqn:thm_2_proof_prob_bound_A_L},~\eqref{eqn:thm_2_proof_prob_bound_C} and~\eqref{eqn:thm_2_proof_a_l_implies_d_e}, we obtain that 
\begin{align}
    \Pr\left[ \bar{\cA}(\varepsilon/2) \land \bar{\cC}\left(2 \sqrt{\log \left( \frac{16}{\delta} \right)} \right) \land \bar{\cD}((\beta_s + 1)s^2\varepsilon) \land \bar{\cE}\left(\varepsilon''\right) \right] \geq 1 - 3\delta/4.\label{eqn:thm_2_prob_a_c_d_e}
\end{align}
Let $d_l$ be chosen large enough to ensure that
\begin{align*}
     s^{L - l+2}\varepsilon & \geq 2c_s (s^{L - l} + 1)\left( \sqrt{\frac{2}{d_l} \log\left( \frac{24(L + 1)}{\delta} \right) } + \frac{1}{d_l}\left(  + \log\left( \frac{24(L + 1)}{\delta} \right) \right) \right) \varepsilon'' \\
    & + c_s \varepsilon'' \sqrt{\frac{8}{d_l}} (s^{(L - l)/2} + 1) \left[(s^{(L - l)/2} + 1)\sqrt{2 \log\left(\frac{64(L + 1)}{\delta}\right)} + 4s^{L - l}\sqrt{\log \left( \frac{16(L + 1)}{\delta} \right)} \right] \\
    & + \frac{c_s \varepsilon''}{d_l} \left[(s^{(L - l)/2} + 1)\sqrt{2 \log\left(\frac{64(L + 1)}{\delta}\right)} + 4s^{L - l} \sqrt{\log \left( \frac{16(L + 1)}{\delta} \right)}\right]^2.
\end{align*}
Define $\upsilon_l := (L + 1 - l)(\beta_s + 1)s^{L - l  +2}\varepsilon$ for $l = 0,1, \dots, L+1$. Invoking Lemma~\ref{lemma:b_h_to_b_h_plus_1} with $\varepsilon_1 = \varepsilon$, $\varepsilon_2 = \upsilon_{l+1}$, $\varepsilon_3 = 2 \sqrt{\log \left( \frac{16}{\delta} \right)}$, $\varepsilon_4 = (\beta_s + 1)s^2 \varepsilon$ and $\varepsilon_5 = \varepsilon''$ gives us that
\begin{align*}
    \Pr \left[\bar{\cA}^{L}(\varepsilon/2) \land \bar{\cB}^{l+1}(\upsilon_{l+1}) \land \bar{\cC}\left(2 \sqrt{\log \left( \frac{16}{\delta} \right)}\right) \land \bar{\cD}^l((\beta_s + 1)s^2 \varepsilon) \land \bar{\cE}^l(\varepsilon'') \Rightarrow \bar{\cB}^{l}(\upsilon_l) \right] \geq 1 - \frac{\delta}{4(L + 1)}.
\end{align*}
Thus,
\begin{align*}
    & \Pr \left[\bar{\cA}(\varepsilon/2) \land  \bar{\cB}(\upsilon_1) \right] \\
    & \geq \Pr \left[\underbrace{\bar{\cA}(\varepsilon/2) \land \bar{\cC}\left(2 \sqrt{\log \left( \frac{16}{\delta} \right)}\right) \land \bar{\cD}((\beta_s + 1)s^2 \varepsilon) \land \bar{\cE}(\varepsilon'')}_{:=\cF(\varepsilon)} \land \left( \bigcap_{l = 1}^{L+1} \bar{\cB}^{l}(\upsilon_l) \right) \right] \\
    & \geq \Pr \left[\cF(\varepsilon) \land \bar{\cB}^{L+1}(\upsilon_{L+1}) \land \bigcap_{l = 0}^{L} \left\{ \cF(\varepsilon) \land \left(\bigcap_{j = L + 1 -l}^{L+1} \bar{\cB}^{j+1}(\upsilon_{j+1}) \right) \Rightarrow \bar{\cB}^{L-l}(\upsilon_{L-l}) \right\} \right] \\
    & \geq \Pr \Bigg[\cF(\varepsilon) \land \bar{\cB}^{L+1}(\upsilon_{L+1}) \land \\
    & \ \ \ \ \ \ \ \ \ \ \ \ \ \ \ \ \ \ \ \bigcap_{l = 0}^{L} \left\{ \bar{\cA}^{L}(\varepsilon/2) \land \bar{\cC}\left(2 \sqrt{\log \left( \frac{16}{\delta} \right)}\right) \land \bar{\cD}^{l}((\beta_s + 1)s^2 \varepsilon) \land \bar{\cE}^{l}(\varepsilon'') \land  \bar{\cB}^{l+1}(\upsilon_{l+1})  \Rightarrow \bar{\cB}^{l}(\upsilon_{l}) \right\} \Bigg] \\
    & \geq 1 - \Pr\left[ \neg \left\{\bar{\cA}(\varepsilon/2) \land \bar{\cC}\left(2 \sqrt{\log \left( \frac{16}{\delta} \right)}\right) \land \bar{\cD}((\beta_s + 1)s^2 \varepsilon) \land \bar{\cE}(\varepsilon'') \right\} \right] - \Pr\left[\neg \bar{\cB}^{L+1}(\upsilon_{L+1}) \right] \\
    & \ \ \ \ \ \ - \sum_{l = 0}^L\Pr \left[ \neg\left\{\bar{\cA}^{L}(\varepsilon/2) \land \bar{\cC}\left(2 \sqrt{\log \left( \frac{16}{\delta} \right)}\right) \land \bar{\cD}^{h}((\beta_s + 1)s^2 \varepsilon) \land \bar{\cE}^{l}(\varepsilon'') \land  \bar{\cB}^{l+1}(\upsilon_{l+1})  \Rightarrow \bar{\cB}^{l}(\upsilon_{l}) \right\} \right] \\
    & \geq 1 - \frac{3\delta}{4} - \sum_{l = 0}^L\frac{\delta}{4(L + 1)}\\
    & \geq 1 - \delta.
\end{align*}

Thus, $\Pr \left[\bar{\cA}(\varepsilon/2) \land  \bar{\cB}(L(\beta_s + 1)s^{L+1}\varepsilon) \right] \geq 1 - \delta$, as required. The proof of the helper lemmas are provided in Appendix~\ref{sec:helper_lemma_proofs}.

\end{proof}

Now, it is straightforward to note that Lemma~\ref{lemma:kernel_init_to_NTK} immediately follows from the Lemma~\ref{theorem:gradient_approx_theorem} and Lemma~\ref{theorem:ntk_derivative_approx_theorem}. With Lemma~\ref{lemma:kernel_init_to_NTK} in place, the rest of the proof of Theorem~\ref{theorem:appx_error} follows similar to the proof of Theorem 3.2 in~\citet{arora2019exact}. The only step in the proof of Theorem 3.2 that is dependent on kernel being ReLU, is Lemma F.6, where they bound the perturbation in gradient based on the amount of perturbation in the weight matrices. We would like to emphasize that this is only step apart from the ones which use Theorem 3.1. Such steps follow immediately by invoking Lemma~\ref{lemma:kernel_init_to_NTK} proven above, instead of Theorem 3.1. Using the result from~\citet[Theorem 3.2]{liu2020linearity}, we know that the spectral norm of the Hessian of the neural net is $\Oc(m^{-1/2})$. Here the implied constant only depends on $s$ and $L$. Thus for the a perturbation in the weight matrices of the order of $\Oc(\sqrt{m} \omega)$, the corresponding perturbation in the gradient is $\Oc(\omega)$, with the implied constant depending only on $s$ and $L$. On substituting this relation in place of Lemma F.6, the claim of Theorem~\ref{theorem:appx_error} follows using the same arguments as the ones used in the proof of Theorem 3.2 of~\citet{arora2019exact}.

\section{Proof of Helper Lemmas}
\label{sec:helper_lemma_proofs}

We first state some additional intermediate lemmas which are used in the proofs several helper lemmas then provide a proof of all the lemmas.

\begin{definition}
\begin{align*}
    \ec{\bfG}{l}(\bfx, \bfx') = [\ec{\bfh}{l}(\bfx) \ \  \ec{\bfh}{l}(\bfx')] \in \R^{d_l \times 2}
\end{align*}
\begin{align*}
    \ec{\tilde{\bfG}}{l}(\bfx, \bfx') = [\ec{\bfG}{l}(\bfx, \bfx')]^{\top}\ec{\bfG}{l}(\bfx, \bfx') = \begin{bmatrix} \ip{\ec{\bfh}{l}(\bfx)}{\ec{\bfh}{l}(\bfx)} & \ip{\ec{\bfh}{l}(\bfx)}{\ec{\bfh}{l}(\bfx')} \\ \ip{\ec{\bfh}{l}(\bfx')}{\ec{\bfh}{l}(\bfx)} & \ip{\ec{\bfh}{l}(\bfx')}{\ec{\bfh}{l}(\bfx')} \\  \end{bmatrix} \equiv \begin{bmatrix} G_{11} & G_{12} \\ G_{12} & G_{22} \\  \end{bmatrix} 
\end{align*}
\end{definition}

\begin{lemma}
\begin{align*}
    \Pr\left[ \bar{\cA}^{l}(\varepsilon) \Rightarrow \bar{\cD}^{l + 1}\left(3s^2 \varphi_{s-1}(d_l, \delta/6, 1) + \frac{2s^2\beta_s \varepsilon}{2s - 1} \right) \land \bar{\cE}^{l+1} \left(3s^2 \left[2\log\left(\frac{6d_l}{\delta}\right) \right]^{s-1} \right) \right] \geq 1 - \delta \\
    \quad \forall \ \varepsilon \in [0,1/s],\  \delta \in (0, 1).
\end{align*} \label{lemma:a_h_implies_d_h_plus_1}
\end{lemma}

\begin{lemma}
Let
\begin{align*}
     \varepsilon' = 3s^2 \varphi_{s-1}((\min_l d_l), \delta/6(L+1), 1) + \frac{2s^2\beta_s \varepsilon}{2s - 1}, \quad
     \varepsilon'' = 3s^2 \left[2\log\left(\frac{6(L+1) (\max_{l} d_l) }{\delta}\right) \right]^{s-1}.
\end{align*}
Then,
\begin{align*}
    \Pr\left[ \bar{\cA}(\varepsilon) \Rightarrow \bar{\cD}(\varepsilon') \land \bar{\cE}(\varepsilon'') \right] \geq 1 - \delta  \quad \forall \ \varepsilon \in [0,1/s],\  \delta \in (0, 1).
\end{align*} \label{lemma:a_all_implies_d_all}
\end{lemma}

\begin{lemma}
For any $1 \leq l \leq L$, any fixed $\{\ec{\bfW}{i}\}_{i = 1}^{l-1}$ and $\bfx, \bfx' \in \cX$, with probability $1 - \delta$ over the randomness of $\ec{\bfW}{l}$, we have,
\begin{align*}
    \left| c_{s} \frac{\tr(\ec{\bfDelta}{l}(\bfx, \bfx'))}{d_l} - \frac{s^2}{2s - 1}\breve{\sigma}_{s-1}\left(\ec{\tilde{\bfG}}{l}(\bfx, \bfx') \right) \right| \leq s^2(G_{11}G_{22})^{\frac{(s-1)}{2}} \varphi_{s-1}(d_l, \delta/2, \rho_G),
\end{align*} 
and
\begin{align*}
    \|\ec{\bfDelta}{l}(\bfx, \bfx')\|_2 & \leq s^2 \left[\left(\sqrt{G_{11}G_{22}} + G_{12}\right)\log\left(\frac{2d_l}{\delta}\right) \right]^{s-1},
\end{align*}
where $\rho_{G} = \frac{G_{12}}{\sqrt{G_{11}G_{22}}}$. \label{lemma:trace_of_diagonal}
\end{lemma}

\begin{lemma}[\cite{boucheron2013}]
Let $\bfxi \sim \cN(0, \bfI_n)$ be an $n$-dimensional unit Gaussian random vector and $\bfA \in \R^{n \times n}$ be a symmetric matrix, then for any $t >0$,
\begin{align*}
    \Pr \left( \left|\bfxi^{\top} \bfA \bfxi - \E\left[\bfxi^{\top} \bfA \bfxi \right] \right| > 2 \|\bfA\|_{F} \sqrt{t} + 2 \|\bfA\|_2 t \right) \leq \exp(-t),
\end{align*}
or equivalently,
\begin{align*}
    \Pr \left( \left|\bfxi^{\top} \bfA \bfxi - \E\left[\bfxi^{\top} \bfA \bfxi \right] \right| >  t \right) \leq \exp\left(-\frac{t^2}{4\|\bfA\|_{F}^2 +  4\|\bfA\|_2}\right).
\end{align*}\label{lemma:gaussian_chaos}
\end{lemma}

\begin{lemma}
If $\bar{\cA}^{L}(\varepsilon_1/2) \land \bar{\cB}^{l+1}(\varepsilon_2) \land \bar{\cC}(\varepsilon_3) \land \bar{\cD}^l(\varepsilon_4) \land \bar{\cE}^l(\varepsilon_5)$ for any pair of points $\bfx, \bfx' \in \cX$, then for any $\bfy \in \{ \bfx, \bfx' \}$
\begin{align*}
    \left\| \bfPi_{\bfG} \left( \ec{\bfW}{l+1} \right)^{\top} \ec{\bfb}{l+1}(\bfy) \right\| & \leq (s^{(L - l)/2} + 1)\sqrt{2 \log\left(\frac{4}{\delta}\right)} + s^{L - l} \varepsilon_3 \sqrt{2}.
\end{align*}
holds with probability at least $1 - \delta$.  \label{lemma:proj_G_times_W_b_bound}
\end{lemma}



\subsection{Proof of Lemma~\ref{lemma:concentration_of_activation_fn}}

We fix a particular $s \in \N$. Let $(X_1, Y_1), (X_2, Y_2), \dots, (X_n, Y_n)$ be $n$ independent samples from $\cN(0, \Sigma_{\rho})$ and let $Z_i = \sigma_s(X_i)\sigma_s(Y_i)$. Define $S_n := \frac{1}{n}\sum_{i = 1}^n Z_i$. Note that $\E[S_n] = \E[Z_1] = \mu_s$. To prove the bound, we first bound $\Pr(Z_i > t)$ for any $t \geq 0$ and then using techniques borrowed from~\citet{Gantert2014concentration}, we obtain an upper bound on $\Pr(S_n - \mu_s > t)$. The bound on $\Pr(S_n - \mu_s < -t)$ follows from Cramer's Theorem~\citep{Cramer1938}. \\

We begin with bounding $\Pr(Z_i > t)$ for any $t \geq 0$. Since $Z_i = (\sigma_1(X_i)\sigma_1(Y_i))^s$, we just focus on bounding $\Pr(\sigma_1(X)\sigma_1(Y) > t)$ for $(X,Y) \sim \cN(0, \Sigma_{\rho})$. Fix a $t \geq 0$. We have,
\begin{align*}
    \Pr(\sigma_1(X)\sigma_1(Y) > t) & = \E \left[ \1 \left\{\sigma_1(X)\sigma_1(Y) > t \right\}\right] \\
    & = \E \left[ \1 \left\{ XY > t, X \geq 0, Y \geq 0 \right\}\right].
\end{align*}

Using a change of variables define $U = (X+Y)/\sqrt{2}$ and $V = (X-Y)/\sqrt{2}$. Note that $U$ and $V$ are zero-mean Gaussian random variables with $\cov(U, V) = 0$ implying that they are independent. Also $\var(U) = 1 + \rho$ and $\var(V) = 1 - \rho$. Thus,
\begin{align*}
    \Pr(\sigma_1(X)\sigma_1(Y) > t) & = \E_{X, Y} \left[ \1 \left\{ XY > t, X \geq 0, Y \geq 0 \right\}\right] \\
    & = \E_{U, V} \left[ \1 \left\{ U^2 - V^2 > 2t, U \geq 0, V \in [-U, U] \right\}\right] \\
    & \leq \E_{U, V} \left[ \1 \left\{ U^2 > V^2 + 2t, U \geq 0 \right\}\right] \\
    & \leq \E_{V} \left[ \Pr\left( U > \sqrt{V^2 + 2t} \right) \right] \\
    & \leq \E_{V} \left[ \exp\left( -\frac{V^2 + 2t}{2(1 + \rho)} \right)\right] \\
    & \leq \frac{1}{\sqrt{2 \pi (1 - \rho)}}\int_{-\infty}^{\infty} \exp\left( -\frac{v^2 + 2t}{2(1 + \rho)}\right) \exp\left( -\frac{v^2}{2(1 - \rho)}\right)\ \mathrm{d}v  \\
    & \leq \frac{e^{-t/(1 + \rho)}}{\sqrt{2 \pi (1 - \rho)}}\int_{-\infty}^{\infty} \exp\left( -\frac{v^2}{(1 - \rho^2)}\right) \  \mathrm{d}v  \\
    & \leq \exp\left( - \frac{t}{1 + \rho} \right) \sqrt{\frac{1 + \rho}{2}}  \\
    & \leq \exp\left( - \frac{t}{1 + \rho} \right).
\end{align*}
In the sixth line we used the concentration bound for Gaussian random variable and in the eighth line we used the Gaussian integral. Consequently, 
\begin{align}
    \Pr(Z_i > t) = \Pr(\sigma_1(X)\sigma_1(Y) > t^{1/s}) \leq \exp\left( - \frac{t^{1/s}}{1 + \rho} \right). \label{eqn:tail_bound_zi}
\end{align}
Using this relation, we now move on to bound $\Pr(S_n \geq x)$ for $x \geq \mu_{s, \rho}$. For the rest of the proof we drop the argument $\rho$ for simplicity. Firstly, note that
\begin{align*}
    \Pr(S_n \geq x) \leq \underbrace{\Pr\left( \max_{j \in [n]} Z_j \geq n(x - \mu_s) \right)}_{A_1^n} + \underbrace{\Pr\left( S_n \geq x, \ \max_{j \in [n]} Z_j < n(x - \mu_s) \right)}_{A_2^n}.
\end{align*}
We begin with the first term.
\begin{align*}
    \Pr\left( \max_{j \in [n]} Z_j \geq n(x - \mu_s) \right) \leq n \Pr(Z_i \geq n(x - \mu_s)) \leq n \exp\left(- \frac{n^{1/s}(x - \mu_s)^{1/s}}{1 + \rho} \right).
\end{align*}
Let $\beta_{\zeta}(n) = \frac{\zeta n^{1/s}}{1 + \rho} > 0$ for some $\zeta > 0$ which will be specified later. We have,
\begin{align*}
    A_2^n & = \Pr\left( S_n \geq x, \ \max_{j \in [n]} Z_j < n(x - \mu_s) \right) \\
     & = \Pr\left( \exp \left( \beta_{\zeta}(n) S_n \right) \geq \exp \left( \beta_{\zeta}(n) x \right), \ \max_{j \in [n]} Z_j < n(x - \mu_s) \right) \\
    & \leq \exp \left( - \beta_{\zeta}(n) x \right) \E\left[ \exp \left( \beta_{\zeta}(n) S_n \right) \1 \left\{ \max_{j \in [n]} Z_j < n(x - \mu_s) \right\} \right] \\
    & \leq \exp \left( - \beta_{\zeta}(n) x \right) \prod_{j = 1}^n \E\left[ \exp \left( \frac{\beta_{\zeta}(n) Z_j}{n} \right) \1 \left\{ Z_j < n(x - \mu_s) \right\} \right] \\
    \implies \log \left( A_2^n \right) & \leq  - \beta_{\zeta}(n) x + \sum_{j = 1}^n \ec{\Gamma_{\zeta}}{j}(n),
\end{align*}
where $\ec{\Gamma_{\zeta}}{j}(n) = \log \left( \E\left[ \exp \left( \dfrac{\beta_{\zeta}(n) \ec{Z_j}{n}}{n} \right) \right] \right)$ and $\ec{Z_j}{n} := Z_j \1 \left\{ Z_j < n(x - \mu_s) \right\}$. Using the relations $\log(1 + x) \leq x$ and $e^x \leq 1 + x + \frac{x^2}{2}e^x$, we have, 
\begin{align*}
    \sum_{j = 1}^n \ec{\Gamma_{\zeta}}{j}(n) & = \sum_{j = 1}^n \log \left( \E\left[ \exp \left( \frac{\beta_{\zeta}(n) \ec{Z_j}{n}}{n} \right) \right] \right) \\
    & \leq \sum_{j = 1}^n \log \left( 1 + \E\left[ \frac{\beta_{\zeta}(n) \ec{Z_j}{n}}{n} \right] + \frac{1}{2} \E\left[ \left(\frac{\beta_{\zeta}(n) \ec{Z_j}{n}}{n} \right)^2 \exp\left(\frac{\beta_{\zeta}(n) \ec{Z_j}{n}}{n} \right) \right] \right) \\
    & \leq \sum_{j = 1}^n  \left\{ \E\left[ \frac{\beta_{\zeta}(n) \ec{Z_j}{n}}{n} \right] + \frac{1}{2} \E\left[ \left(\frac{\beta_{\zeta}(n) \ec{Z_j}{n}}{n} \right)^2 \exp\left(\frac{\beta_{\zeta}(n) \ec{Z_j}{n}}{n} \right) \right] \right\}.
\end{align*}
Note that,
\begin{align*}
    \sum_{j = 1}^n \E\left[ \frac{\beta_{\zeta}(n) \ec{Z_j}{n}}{n} \right] & = \sum_{j = 1}^n \frac{\beta_{\zeta}(n)}{n} \E \left[ \ec{Z_j}{n} \right] \\
    & \leq \beta_{\zeta}(n) \mu_s.
\end{align*}
For the second term we have, 
\begin{align*}
    \E\left[ \left(\frac{\beta_{\zeta}(n) \ec{Z_1}{n}}{n} \right)^2 \exp\left(\frac{\beta_{\zeta}(n) \ec{Z_1}{n}}{n} \right) \right] & \leq \left(\frac{\beta_{\zeta}(n)}{n} \right)^2 \E \left[ \left(\ec{Z_1}{n}\right)^{\frac{2(1 + \eta)}{\eta}} \right]^{\frac{\eta}{1 + \eta}} \E\left[  \exp\left(\frac{(1 + \eta)\beta_{\zeta}(n) \ec{Z_1}{n}}{n} \right) \right]^{\frac{1}{1 + \eta}}
\end{align*}
for some $\eta > 0$ by using H\"older's inequality. Since $\E \left[ \left(\ec{Z_1}{n}\right)^{2(1 + \eta)/\eta} \right]^{\eta/(1 + \eta)} = \left[\mu_{\frac{2s(1 + \eta)}{\eta}}\right]^{\frac{\eta}{1 + \eta}}$, we focus on the other term. We have,
\begin{align*}
    & \E\left[  \exp\left(\frac{(1 + \eta)\beta_{\zeta}(n) \ec{Z_1}{n}}{n} \right) \right] \\
    & = 1 + \int_{0}^{n(x - \mu_s)}\frac{(1+ \eta) \beta_{\zeta}(n)}{n}\exp\left(\frac{(1 + \eta)\beta_{\zeta}(n) z}{n} \right) \Pr(Z_1 > z) \ \mathrm{d}z \\
    & = 1 + n(x - \mu_s)\int_{0}^{1}\frac{(1+ \eta) \beta_{\zeta}(n)}{n}\exp\left(\frac{(1 + \eta)\beta_{\zeta}(n) n(x - \mu_s) y}{n} \right) \Pr(Z_1 > n(x - \mu_s)y) \ \mathrm{d}y \\
    & \leq 1 + (x - \mu_s)\int_{0}^{1}(1+ \eta) \beta_{\zeta}(n)\exp\left((1 + \eta)\beta_{\zeta}(n) (x - \mu_s) y - \frac{(n(x - \mu_s)y)^{1/s}}{1 + \rho}\right)  \ \mathrm{d}y \\
    & \leq 1 + (x - \mu_s)(1+ \eta) \beta_{\zeta}(n)\int_{0}^{1}\exp\left(\beta_{\zeta}(n) \left[(1 + \eta) (x - \mu_s) y - \frac{(x - \mu_s)^{1/s}y^{1/s}}{\zeta}\right]\right)   \ \mathrm{d}y.
\end{align*}
For $\zeta \leq \dfrac{(x - \mu_s)^{1/s - 1}}{2(1 + \eta)}$, we have, 
\begin{align*}
    & \E\left[  \exp\left(\frac{(1 + \eta)\beta_{\zeta}(n) \ec{Z_1}{n}}{n} \right) \right] \\
    & \leq 1 + (x - \mu_s)(1+ \eta) \beta_{\zeta}(n)\int_{0}^{1}\exp\left(\beta_{\zeta}(n) \left[(1 + \eta) (x - \mu_s) y - \frac{(x - \mu_s)^{1/s}y^{1/s}}{\zeta}\right]\right)   \ \mathrm{d}y \\
    & \leq 1 + (x - \mu_s)(1+ \eta) \beta_{\zeta}(n)\int_{0}^{1}\exp\left(-\frac{1}{2}\beta_{\zeta}(n) (1 + \eta) (x - \mu_s) y\right)   \ \mathrm{d}y \\
    & \leq 3.
\end{align*}
On combining all the equations, we obtain that for $\zeta \leq \frac{(x - \mu_s)^{1/s - 1}}{2(1 + \eta)}$ and $\eta > 0$
\begin{align*}
    \log \left( A_2^n \right) & \leq  - \beta_{\zeta}(n) (x - \mu_s) + \frac{(\beta_{\zeta}(n))^2}{2n}  \left( \left[\mu_{\frac{2s(1 + \eta)}{\eta}}\right]^{\frac{\eta}{1 + \eta}} 3^{\frac{1}{1 + \eta}}\right)
\end{align*}
We set $\eta = 1$. Thus for $\zeta \leq 0.25{(x - \mu_s)^{1/s - 1}}$, we have,
\begin{align*}
    \log \left( A_2^n \right) & \leq  - \frac{\zeta n^{1/s}(x - \mu_s)}{1 + \rho} + \frac{\zeta^2 n^{2/s - 1}}{2(1 + \rho)^2}  \sqrt{3\mu_{4s}}.
\end{align*}
Note that the RHS is minimized at $\zeta_{*} = \dfrac{n^{1 - 1/s}(x - \mu_s)(1 + \rho)}{\sqrt{3\mu_{4s}}}$. However, $\zeta_{*} \leq 0.25{(x - \mu_s)^{1/s - 1}}$ only for $\displaystyle (x - \mu_s) \leq \underbrace{\left( \frac{\sqrt{3 \mu_{4s}}}{4(1 + \rho)} \right)^{2 - 1/s} n^{-\left(\frac{s - 1}{2s - 1} \right)}}_{:= t^*(n)}$. Thus, for $(x - \mu_s) \leq t^*(n)$, we can plug in $\zeta = \zeta_{*}$ in the above expression to obtain
\begin{align*}
    \log \left( A_2^n \right) & \leq  - \frac{n(x - \mu_s)^2}{2 \sqrt{3 \mu_{4s}}}.
\end{align*}
For $(x - \mu_s) > t^*(n)$, we plug in $\zeta = 0.25{(x - \mu_s)^{1/s - 1}}$ to obtain
\begin{align*}
    \log \left( A_2^n \right) & \leq  - \frac{n^{1/s}(x - \mu_s)^{1/s}}{8(1 + \rho)}.
\end{align*}
On combining the two, we obtain that for any $t \geq 0$
\begin{align*}
    \Pr(S_n -\mu_s \geq t) = \begin{cases} (n + 1) \exp\left(- \dfrac{nt^2}{2 \sqrt{3 \mu_{4s}}} \right) & \text{ if } t \leq t^*(n), \\ (n + 1) \exp\left(- \dfrac{(nt)^{1/s}}{8(1 + \rho)} \right) & \text{ if } t > t^*(n).
    \end{cases}
\end{align*}
We now consider the other side. Consider any $x \leq \mu_s$ and $\lambda > 0$. We have, 
\begin{align*}
    \Pr(S_n \leq x) & = \Pr\left(\exp(-\lambda S_n) \geq e^{-\lambda x} \right) \\
    & \leq \E \left[ \exp(-\lambda S_n)\right] e^{\lambda x} \\
    & \leq \E \left[ 1 - \lambda S_n + \frac{\lambda^2 S_n^2}{2}\right] e^{\lambda x} \\
    & \leq  \left( 1 - \lambda \E[S_n] + \frac{\lambda^2 \E[S_n^2]}{2}\right) e^{\lambda x} \\
    & \leq  \left( 1 - \lambda \mu_s + \frac{\lambda^2 \mu_{2s}}{2n}\right) e^{\lambda x} \\
    & \leq  \exp\left(  - \lambda \mu_s + \frac{\lambda^2 \mu_{2s}}{2n} + \lambda x \right) \\
    & \leq  \exp\left(  \lambda (x - \mu_s) + \frac{\lambda^2 \mu_{2s}}{2n} + \lambda x \right).
\end{align*}
Minimizing this over $\lambda$ yields $\lambda_{*} = -\dfrac{n(x - \mu_s)}{\mu_{2s}} > 0$. On plugging this value, we obtain,
\begin{align*}
    \Pr(S_n \leq x) &  \leq  \exp\left(  - \frac{n (x - \mu_s)^2}{2\mu_{2s}} \right).
\end{align*}
Consequently for any $t\geq 0$, we have, 
\begin{align*}
    \Pr(S_n - \mu_s \leq -t) &  \leq  \exp\left(  - \frac{n t^2}{2\mu_{2s}} \right).
\end{align*}
On combining the relations, we can conclude that if $n \geq n_{*}(s, \delta, \rho)$, then
\begin{align*}
    \Pr \left( |S_n - \mu_s| \geq \sqrt{\frac{2(\sqrt{3 \mu_{4s}} + \mu_{2s})}{n} \log \left(\frac{n + 1}{\delta} \right)} \right) \leq \delta.
\end{align*}



\subsection{Proof of Lemma~\ref{lemma:decency_parameters}}

Fix a $s \in \N$. We drop the subscript $s$ for the remainder of the proof for ease of notation. We need to show that the dual activation function $\breve{\sigma}$ is $\beta$-Lipschitz in $\cM_{+}^{\gamma}$ w.r.t. the $\infty$-norm. Let $\Sigma = \begin{pmatrix} \Sigma_{11} & \Sigma_{12} \\ \Sigma_{12} & \Sigma_{22} \\ \end{pmatrix} \in \cM_{+}^{\gamma}$. In order to prove $\beta$-Lipschitzness w.r.t. the $\infty$-norm, it is sufficient to show that,
\begin{align*}
    \|\nabla \breve{\sigma}_s\|_1 = \left| \frac{\partial \breve{\sigma}_s }{\partial \Sigma_{12}}\right| + \left| \frac{\partial \breve{\sigma}_s }{\partial \Sigma_{11}}\right| + \left| \frac{\partial \breve{\sigma}_s }{\partial \Sigma_{22}}\right| \leq \beta.
\end{align*}
Recall that since $\sigma_{s}$ is $s$-homogeneous, 
\begin{align*}
    \breve{\sigma}_s(\Sigma) = (\Sigma_{11} \Sigma_{22})^{s/2} \bar{\sigma}_s \left( \frac{\Sigma_{12}}{\sqrt{\Sigma_{11} \Sigma_{22}}} \right).
\end{align*}
Using this relation, we have,
\begin{align*}
    \frac{\partial \breve{\sigma}_s }{\partial \Sigma_{12}} & = (\Sigma_{11} \Sigma_{22})^{\frac{s-1}{2}} \bar{\sigma}_s' \left( \frac{\Sigma_{12}}{\sqrt{\Sigma_{11} \Sigma_{22}}} \right) \\
    \frac{\partial \breve{\sigma}_s }{\partial \Sigma_{11}} & = \frac{s}{2} (\Sigma_{11} \Sigma_{22})^{\frac{s}{2} - 1} \Sigma_{22}\bar{\sigma}_s \left( \frac{\Sigma_{12}}{\sqrt{\Sigma_{11} \Sigma_{22}}} \right) - \frac{1}{2} (\Sigma_{11} \Sigma_{22})^{\frac{s- 1}{2}} \bar{\sigma}_s' \left( \frac{\Sigma_{12}}{\sqrt{\Sigma_{11} \Sigma_{22}}} \right)  \frac{\Sigma_{12}}{\Sigma_{11}}  \\
    & = \Sigma_{22} \left\{ \frac{s}{2} (\Sigma_{11} \Sigma_{22})^{\frac{s}{2} - 1} \bar{\sigma}_s \left( \frac{\Sigma_{12}}{\sqrt{\Sigma_{11} \Sigma_{22}}} \right) - \frac{1}{2} (\Sigma_{11} \Sigma_{22})^{\frac{s- 1}{2}} \bar{\sigma}_s' \left( \frac{\Sigma_{12}}{\sqrt{\Sigma_{11} \Sigma_{22}}} \right)  \frac{\Sigma_{12}}{\Sigma_{11}\Sigma_{22}} \right\}  \\
    \frac{\partial \breve{\sigma}_s }{\partial \Sigma_{22}} & = \frac{s}{2} (\Sigma_{11} \Sigma_{22})^{\frac{s}{2} - 1} \Sigma_{11}\bar{\sigma}_s \left( \frac{\Sigma_{12}}{\sqrt{\Sigma_{11} \Sigma_{22}}} \right) - \frac{1}{2} (\Sigma_{11} \Sigma_{22})^{\frac{s- 1}{2}} \bar{\sigma}_s' \left( \frac{\Sigma_{12}}{\sqrt{\Sigma_{11} \Sigma_{22}}} \right)  \frac{\Sigma_{12}}{\Sigma_{22}}  \\
    & = \Sigma_{11} \left\{ \frac{s}{2} (\Sigma_{11} \Sigma_{22})^{\frac{s}{2} - 1} \bar{\sigma}_s \left( \frac{\Sigma_{12}}{\sqrt{\Sigma_{11} \Sigma_{22}}} \right) - \frac{1}{2} (\Sigma_{11} \Sigma_{22})^{\frac{s- 1}{2}} \bar{\sigma}_s' \left( \frac{\Sigma_{12}}{\sqrt{\Sigma_{11} \Sigma_{22}}} \right)  \frac{\Sigma_{12}}{\Sigma_{11}\Sigma_{22}} \right\}.
\end{align*}
Consequently, 
\begin{align*}
    \|\nabla \breve{\sigma}_s\|_1 & = (\Sigma_{11} \Sigma_{22})^{\frac{s-1}{2}} \left|  \bar{\sigma}_s' \left( \frac{\Sigma_{12}}{\sqrt{\Sigma_{11} \Sigma_{22}}} \right) \right| + \\
    & \quad \quad \frac{(\Sigma_{11} + \Sigma_{22}) (\Sigma_{11} \Sigma_{22})^{\frac{s}{2} - 1}}{2}  \left| s \bar{\sigma}_s \left( \frac{\Sigma_{12}}{\sqrt{\Sigma_{11} \Sigma_{22}}} \right) -  \bar{\sigma}_s' \left( \frac{\Sigma_{12}}{\sqrt{\Sigma_{11} \Sigma_{22}}} \right)  \frac{\Sigma_{12}}{\sqrt{\Sigma_{11}\Sigma_{22}}} \right|.
\end{align*}
For simplicity, let $\tilde{\rho} = \frac{\Sigma_{12}}{\sqrt{\Sigma_{11} \Sigma_{22}}}$. For $s = 1$,~\citet{arora2019exact} showed that the above expression is $1 + o(\gamma)$.  We consider the case for $s > 1$. Using~\eqref{eqn:dual_derivative}, we have,
\begin{align*}
    \|\nabla \breve{\sigma}_s\|_1 & = \frac{s^2}{2s - 1}(\Sigma_{11} \Sigma_{22})^{\frac{s-1}{2}} \left|  \bar{\sigma}_{s-1} \left( \tilde{\rho} \right) \right| + \frac{(\Sigma_{11} + \Sigma_{22}) (\Sigma_{11} \Sigma_{22})^{\frac{s}{2} - 1}}{2}  \left| s \bar{\sigma}_s \left( \tilde{\rho} \right) -  \frac{s^2}{2s - 1} \bar{\sigma}_{s-1} \left( \tilde{\rho} \right)  \tilde{\rho} \right| \\
    & \leq \frac{2s}{3}(\Sigma_{11} \Sigma_{22})^{\frac{s-1}{2}} + \frac{3s}{4} (\Sigma_{11} + \Sigma_{22}) (\Sigma_{11} \Sigma_{22})^{\frac{s}{2} - 1} \\
    & \leq \frac{13s}{6}(1 + \gamma)^{s-1} \\
\end{align*}
If $\gamma \leq 1/s$, then $\|\nabla \breve{\sigma}_s\|_1 \leq 6s$, as required.



\subsection{Proof of Lemma~\ref{lemma:trace_D_times_b}}
\label{proof:trace_D_times_b}

We have,
\begin{align*}
    & \left| c_{s} \frac{\tr(\ec{\bfDelta}{l}(\bfy, \bfy'))}{d_l}\ip{\ec{\bfb}{l+1}(\bfy)}{\ec{\bfb}{l+1}(\bfy')}  - \prod_{j = l}^L \ec{\dot{\Sigma}}{j}(\bfy, \bfy') \right| \\
    & \leq \left| c_{s} \frac{\tr(\ec{\bfDelta}{l}(\bfy, \bfy'))}{d_l}  - \ec{\dot{\Sigma}}{l}(\bfy, \bfy') \right| \left| \ip{\ec{\bfb}{l+1}(\bfy)}{\ec{\bfb}{l+1}(\bfy')} \right| +  \\
    & ~~~~~~~~~~~~~~~~~~~~~~~~~~~~~ \left|\ec{\dot{\Sigma}}{l}(\bfy, \bfy') \right| \left| \ip{\ec{\bfb}{l+1}(\bfy)}{\ec{\bfb}{l+1}(\bfy')}  - \prod_{j = l + 1}^L \ec{\dot{\Sigma}}{j}(\bfy, \bfy') \right| \\
    & \leq \left| \ip{\ec{\bfb}{l+1}(\bfy)}{\ec{\bfb}{l+1}(\bfy')} \right| \varepsilon_4 +    \frac{s^2}{2s - 1}\breve{\sigma}_{s-1}\left(\ec{\Lambda}{l}(\bfx, \bfx') \right) \varepsilon_2 \\
\end{align*} 
Since $\bar{\cB}^{l+1}(\varepsilon_2)$, $\ip{\ec{\bfb}{l+1}(\bfy)}{\ec{\bfb}{l+1}(\bfy')} \leq \prod_{j = l + 1}^L \ec{\dot{\Sigma}}{j}(\bfy, \bfy') + \varepsilon_2 \leq \prod_{j = l+1}^L \ec{\dot{\Sigma}}{j}(\bfy, \bfy') + 1$. Moreover, since $\ec{\dot{\Sigma}}{l}(\bfy, \bfy') = \frac{s^2}{2s - 1}\breve{\sigma}_{s-1}\left(\ec{\Lambda}{l}(\bfy, \bfy') \right) \leq \frac{s^2}{2s -1}$, we have,
\begin{align*}
    & \left| c_{s} \frac{\tr(\ec{\bfDelta}{l}(\bfy, \bfy'))}{d_l}\ip{\ec{\bfb}{l+1}(\bfy)}{\ec{\bfb}{l+1}(\bfy')}  - \prod_{j = l+1}^L \ec{\dot{\Sigma}}{j}(\bfy, \bfy') \right| & \leq \left(\left( \frac{s^2}{2s -1}\right)^{L - l} + 1\right) \varepsilon_4 +    s \varepsilon_2 \\
    & \leq s^{L - l}\varepsilon_4 + s\varepsilon_2.
\end{align*} 



\subsection{Proof of Lemma~\ref{lemma:b_h_proj_G_close_to_trace_D_times_b}}
\label{proof:b_h_proj_G_close_to_trace_D_times_b}

Fix any $(\bfy, \bfy') \in \{(\bfx, \bfx), (\bfx, \bfx'), (\bfx', \bfx')\}$. Recall that $\ec{\bfG}{l}(\bfy, \bfy') = [\ec{\bfh}{l}(\bfy) \ \  \ec{\bfh}{l}(\bfy')] \in \R^{d_l \times 2}$. Similarly define, $\ec{\bfF}{l+1}(\bfy, \bfy') = [\ec{\bff}{l+1}(\bfy) \ \  \ec{\bff}{l+1}(\bfy')] \in \R^{d_l \times 2}$. Using Lemma~\ref{lemma:nullspace_equivalence} for each row of $\ec{\bfW}{l+1}$, we can conclude that
\begin{align*}
    \ec{\bfW}{l+1}\bfPi_{\bfG}^{\perp} \overset{d}{=}_{\ec{\bfF}{l+1} = \ec{\bfW}{l+1}\ec{\bfG}{l}} \widetilde{\bfW}\bfPi_{\bfG}^{\perp},
\end{align*}
where $\widetilde{\bfW}$ is a i.i.d. copy of $\ec{\bfW}{l+1}$. \\

Note that $\left[\left(\ec{\bfb}{l+1}(\bfy)\right)^{\top}\widetilde{\bfW} \left(\ec{\bfb}{l+1}(\bfy')\right)^{\top}\widetilde{\bfW}\right]^{\top} \in \R^{2d_l} \sim \cN(0, \bfSigma)$ where $\bfSigma = \begin{bmatrix} b \Ii_{d_l} & b' \Ii_{d_l} \\ b' \Ii_{d_l} & b'' \Ii_{d_l} \end{bmatrix}$. In the definition of $\bfSigma$, $b = \ip{\ec{\bfb}{l+1}(\bfy)}{\ec{\bfb}{l+1}(\bfy)}, b' = \ip{\ec{\bfb}{l+1}(\bfy)}{\ec{\bfb}{l+1}(\bfy')}$ and $b'' = \ip{\ec{\bfb}{l+1}(\bfy')}{\ec{\bfb}{l+1}(\bfy')}$. If $\bfM$ is such that $\bfSigma = \bfM \bfM^{\top}$, then
\begin{align*}
    \left[\left(\ec{\bfb}{l+1}(\bfy)\right)^{\top}\widetilde{\bfW} \left(\ec{\bfb}{l+1}(\bfy')\right)^{\top}\widetilde{\bfW}\right]^{\top} \overset{d}{=} \bfM \bfxi,
\end{align*}
where $\bfxi \sim \cN(0, \Ii_{2d_l})$. Thus, for fixed $\ec{\bfG}{l}$ and conditioned on the value of $\{\ec{\bfb}{l+1}(\bfy), \ec{\bfb}{l+1}(\bfy'), \ec{\bfh}{l}(\bfy), \ec{\bfh}{l}(\bfy') \}$, we have,
\begin{align*}
    & \left(\ec{\bfb}{l+1}(\bfy)\right)^{\top} \ec{\bfW}{l+1}\bfPi_{\bfG}^{\perp} \ec{\bfDelta}{l}(\bfy, \bfy') \bfPi_{\bfG}^{\perp}  \left(  \ec{\bfW}{l+1}\right)^{\top}\ec{\bfb}{l+1}(\bfy) \\
    & \overset{d}{=} \left(\ec{\bfb}{l+1}(\bfy)\right)^{\top}\widetilde{\bfW}\bfPi_{\bfG}^{\perp} \ec{\bfDelta}{l}(\bfy, \bfy') \bfPi_{\bfG}^{\perp}  \left( \widetilde{\bfW} \right)^{\top}\ec{\bfb}{l+1}(\bfy) \\
    & \overset{d}{=} ([\Ii_{d_l} \ \  0] \bfM \bfxi)^{\top} \bfPi_{\bfG}^{\perp} \ec{\bfDelta}{l}(\bfy, \bfy') \bfPi_{\bfG}^{\perp}  ([0 \ \  \Ii_{d_l}] \bfM \bfxi) \\
    & \overset{d}{=} \frac{1}{2} \bfxi^{\top} \bfM^{\top} \begin{bmatrix} 0 & \bfPi_{\bfG}^{\perp} \ec{\bfDelta}{l}(\bfy, \bfy') \bfPi_{\bfG}^{\perp} \\ \bfPi_{\bfG}^{\perp} \ec{\bfDelta}{l}(\bfy, \bfy') \bfPi_{\bfG}^{\perp} & 0 \end{bmatrix} \bfM \bfxi
\end{align*}
Define
\begin{align*}
    \bfA := \frac{1}{2} \bfM^{\top} \begin{bmatrix} 0 & \bfPi_{\bfG}^{\perp} \ec{\bfDelta}{l}(\bfy, \bfy') \bfPi_{\bfG}^{\perp} \\ \bfPi_{\bfG}^{\perp} \ec{\bfDelta}{l}(\bfy, \bfy') \bfPi_{\bfG}^{\perp} & 0 \end{bmatrix} \bfM.
\end{align*}
Then we have,
\begin{align*}
    \E[\bfxi^{\top} \bfA \bfxi] & = \tr(\bfA) = \frac{1}{2} \tr \left( \begin{bmatrix} 0 & \bfPi_{\bfG}^{\perp} \ec{\bfDelta}{l}(\bfy, \bfy') \bfPi_{\bfG}^{\perp} \\ \bfPi_{\bfG}^{\perp} \ec{\bfDelta}{l}(\bfy, \bfy') \bfPi_{\bfG}^{\perp} & 0 \end{bmatrix} \bfSigma\right) \\
    & = b' \tr(\bfPi_{\bfG}^{\perp}\ec{\bfDelta}{l}(\bfy, \bfy')\bfPi_{\bfG}^{\perp}\Ii_{d_l}) \\
    & = b' \tr(\ec{\bfDelta}{l}(\bfy, \bfy')\bfPi_{\bfG}^{\perp}) \\
    & = b' \tr(\ec{\bfDelta}{l}(\bfy, \bfy')(\bfI_{d_l} - \bfPi_{\bfG})) \\
    & = b' \left[ \tr(\ec{\bfDelta}{l}(\bfy, \bfy')) - \tr(\ec{\bfDelta}{l}(\bfy, \bfy')\bfPi_{\bfG}) \right].
\end{align*}
Since $\tr(\bfPi_{\bfG}) = \mathrm{rank}(\bfPi_{\bfG})) \leq 2$, we have,
\begin{align*}
    0 \leq \tr(\ec{\bfDelta}{l}(\bfy, \bfy')\bfPi_{\bfG}) \leq \|\ec{\bfDelta}{l}(\bfy, \bfy')\|_2 \tr(\bfPi_{\bfG}) \leq 2 \|\ec{\bfDelta}{l}(\bfy, \bfy')\|_2.
\end{align*}
Consequently, 
\begin{align*}
    \left| \E[\bfxi^{\top} \bfA \bfxi] - b'\tr(\ec{\bfDelta}{l}(\bfy, \bfy'))  \right| \leq 2b'\|\ec{\bfDelta}{l}(\bfy, \bfy')\|_2.
\end{align*}
For the upper bound on the spectrum, note that $\| \bfM \|_2^2 = \| \bfSigma \|_2 \leq b + b''$. Hence,
\begin{align*}
    \|\bfA\|_2 & \leq \frac{1}{2} \| \bfM \|_2^2 \| \bfPi_{\bfG}^{\perp}\ec{\bfDelta}{l}(\bfy, \bfy')\bfPi_{\bfG}^{\perp}\|_2 \\
    & \leq \frac{1}{2} (b + b'') \| \bfPi_{\bfG}^{\perp}\|_2 \|\ec{\bfDelta}{l}(\bfy, \bfy')\|_2 \|\bfPi_{\bfG}^{\perp}\|_2 \\
    & \leq \frac{b + b''}{2}\|\ec{\bfDelta}{l}(\bfy, \bfy')\|_2 . \\
\end{align*}
Thus, on using Lemma~\ref{lemma:gaussian_chaos} with $t = \log(3/\delta)$, we have with probability at least $1 - \delta/3$
\begin{align*}
    \left| \bfxi^{\top} \bfA \bfxi - \E[\bfxi^{\top} \bfA \bfxi] \right| & \leq 2 (\|\bfA\|_F \sqrt{t} + \| \bfA\|_2 t) \\
    & \leq 2 \| \bfA\|_2 ( \sqrt{2d_l t} +  t) \\
    & \leq (b + b'')\|\ec{\bfDelta}{l}(\bfy, \bfy')\|_2 \left( \sqrt{2d_l \log\left( \frac{3}{\delta} \right) } + \log\left( \frac{3}{\delta} \right) \right).
\end{align*}
Consequently,
\begin{align*}
    & \left| \frac{c_s}{d_l}\left(\ec{\bfb}{l+1}(\bfy)\right)^{\top} \ec{\bfW}{l+1}\bfPi_{\bfG}^{\perp} \ec{\bfDelta}{l}(\bfy, \bfy') \bfPi_{\bfG}^{\perp}  \left(  \ec{\bfW}{l+1}\right)^{\top}\ec{\bfb}{l+1}(\bfy) - \frac{c_s}{d_l} \tr(\ec{\bfDelta}{l}(\bfy, \bfy')) \ip{\ec{\bfb}{l+1}(\bfy)}{\ec{\bfb}{l+1}(\bfy')}   \right| \\
    & \leq \frac{c_s}{d_l} \left| \bfxi^{\top} \bfA \bfxi - \E[\bfxi^{\top} \bfA \bfxi] \right| + \frac{c_s}{d_l}  \left| \E[\bfxi^{\top} \bfA \bfxi] - b'\tr(\ec{\bfDelta}{l}(\bfy, \bfy'))  \right| \\
    & \leq c_s(b + b'')\|\ec{\bfDelta}{l}(\bfy, \bfy')\|_2 \left( \sqrt{\frac{2}{d_l} \log\left( \frac{3}{\delta} \right) } + \frac{1}{d_l}\log\left( \frac{3}{\delta} \right) \right) +  \frac{2b'c_s}{d_l}\|\ec{\bfDelta}{l}(\bfy, \bfy')\|_2,
\end{align*}
holds with probability at least $1 - \delta/3$. The lemma follows by taking a union bound over all possible choices of $(\bfy, \bfy')$. \\

We can use the above result to obtain the following bound for any $\bfy \in \{\bfx, \bfx'\}$
\begin{align*}
    & \sqrt{\frac{c_s}{d_l}} \left\| \left(\ec{\bfb}{l+1}(\bfy)\right)^{\top} \ec{\bfW}{l+1}\bfPi_{\bfG}^{\perp} \ec{\bfD}{l}(\bfy) \right\| \\
    & \leq \left\{ \frac{c_s}{d_l}  \left(\ec{\bfb}{l+1}(\bfy)\right)^{\top} \ec{\bfW}{l+1}\bfPi_{\bfG}^{\perp} \ec{\bfDelta}{l}(\bfy, \bfy) \bfPi_{\bfG}^{\perp}  \left(  \ec{\bfW}{l+1}\right)^{\top}\ec{\bfb}{l+1}(\bfy) \right\}^{1/2} \\
    & \leq \left\{ \frac{c_s}{d_l} \tr(\ec{\bfDelta}{l}(\bfy, \bfy))b +  2c_s b\|\ec{\bfDelta}{l}(\bfy, \bfy)\|_2 \left( \sqrt{\frac{2}{d_l} \log\left( \frac{3}{\delta} \right) } + \frac{1}{d_l}\log\left( \frac{3}{\delta} \right) \right) +  \frac{2bc_s}{d_l}\|\ec{\bfDelta}{l}(\bfy, \bfy')\|_2  \right\}^{1/2} \\
    & \leq \sqrt{bc_s \|\ec{\bfDelta}{l}(\bfy, \bfy)\|_2} \cdot \left\{ 1 +   2\left( \sqrt{\frac{2}{d_l} \log\left( \frac{3}{\delta} \right) } + \frac{1}{d_l}\left( 1 + \log\left( \frac{3}{\delta} \right) \right) \right)  \right\}^{1/2} \\
    & \leq \sqrt{2bc_s \|\ec{\bfDelta}{l}(\bfy, \bfy)\|_2},
\end{align*}
where $b = \ip{\ec{\bfb}{l+1}(\bfy)}{\ec{\bfb}{l+1}(\bfy)} $ and the last step uses the relation the bound on $d_l$.



\subsection{Proof of Lemma~\ref{lemma:b_W_Delta_diff_with_its_proj}}
\label{proof:b_W_Delta_diff_with_its_proj}

Fix any $(\bfy, \bfy') \in \{ (\bfx, \bfx), (\bfx, \bfx'), (\bfx', \bfx') \}$. For ease of writing let $\ec{\bfV}{l+1}(\bfx) := \left(\ec{\bfW}{l+1}\right)^{\top} \ec{\bfb}{l+1}(\bfx)$. We are interested in bounding the following expression.
\begin{align*}
    & \left(\ec{\bfV}{l+1}(\bfy)\right)^{\top} \ec{\bfDelta}{l}(\bfy, \bfy') \ec{\bfV}{l+1}(\bfy') - \left(\ec{\bfV}{l+1}(\bfy)\right)^{\top} \bfPi_{\bfG}^{\perp}\ec{\bfDelta}{l}(\bfy, \bfy')\bfPi_{\bfG}^{\perp} \ec{\bfV}{l+1}(\bfy') \\
    & = \left(\ec{\bfV}{l+1}(\bfy)\right)^{\top} (\bfPi_{\bfG} + \bfPi_{\bfG}^{\perp}) \ec{\bfD}{l}(\bfy) \ec{\bfD}{l}(\bfy') (\bfPi_{\bfG} + \bfPi_{\bfG}^{\perp} ) \ec{\bfV}{l+1}(\bfy') - \\
    & ~~~~~~~~~~~~~ \left(\ec{\bfV}{l+1}(\bfy)\right)^{\top} \bfPi_{\bfG}^{\perp}\ec{\bfD}{l}(\bfy) \ec{\bfD}{l}(\bfy')\bfPi_{\bfG}^{\perp} \ec{\bfV}{l+1}(\bfy') \\
    & = \left(\ec{\bfV}{l+1}(\bfy)\right)^{\top}\bfPi_{\bfG} \ec{\bfD}{l}(\bfy) \ec{\bfD}{l}(\bfy')\bfPi_{\bfG} \ec{\bfV}{l+1}(\bfy') + 
    \left(\ec{\bfV}{l+1}(\bfy)\right)^{\top}\bfPi_{\bfG}\ec{\bfD}{l}(\bfy) \ec{\bfD}{l}(\bfy')\bfPi_{\bfG}^{\perp} \ec{\bfV}{l+1}(\bfy') + \\ 
    & \ \ \ \ \ \left(\ec{\bfV}{l+1}(\bfy)\right)^{\top} \bfPi_{\bfG}^{\perp} \ec{\bfD}{l}(\bfy) \ec{\bfD}{l}(\bfy')\bfPi_{\bfG} \ec{\bfV}{l+1}(\bfy').
\end{align*} 
Consequently, 
\begin{align*}
    & \frac{c_s}{d_l} \left|\left(\ec{\bfV}{l+1}(\bfy)\right)^{\top} \ec{\bfDelta}{l}(\bfy, \bfy') \ec{\bfV}{l+1}(\bfy') - \left(\ec{\bfV}{l+1}(\bfy)\right)^{\top} \bfPi_{\bfG}^{\perp}\ec{\bfDelta}{l}(\bfy, \bfy')\bfPi_{\bfG}^{\perp} \ec{\bfV}{l+1}(\bfy')\right| \\
    & = \frac{c_s}{d_l} \bigg|\left(\ec{\bfV}{l+1}(\bfy)\right)^{\top}\bfPi_{\bfG} \ec{\bfD}{l}(\bfy) \ec{\bfD}{l}(\bfy')\bfPi_{\bfG} \ec{\bfV}{l+1}(\bfy') + 
    \left(\ec{\bfV}{l+1}(\bfy)\right)^{\top}\bfPi_{\bfG}\ec{\bfD}{l}(\bfy) \ec{\bfD}{l}(\bfy')\bfPi_{\bfG}^{\perp} \ec{\bfV}{l+1}(\bfy') + \\ 
    & \ \ \ \ \ \left(\ec{\bfV}{l+1}(\bfy)\right)^{\top} \bfPi_{\bfG}^{\perp} \ec{\bfD}{l}(\bfy) \ec{\bfD}{l}(\bfy')\bfPi_{\bfG} \ec{\bfV}{l+1}(\bfy') \bigg| \\
    & \leq \frac{c_s}{d_l} \left\|\left(\ec{\bfb}{l+1}(\bfy) \right)^{\top} \ec{\bfW}{l+1}\bfPi_{\bfG} \ec{\bfD}{l}(\bfy) \right\| \left\|\ec{\bfD}{l}(\bfy')\bfPi_{\bfG} \left(\ec{\bfW}{l+1}\right)^{\top} \ec{\bfb}{l+1}(\bfy')\right\| + \\
    & \ \ \ \ \ \frac{c_s}{d_l} \left\| \left(\ec{\bfb}{l+1}(\bfy) \right)^{\top} \ec{\bfW}{l+1}\bfPi_{\bfG}\ec{\bfD}{l}(\bfy) \right\| \left\|\ec{\bfD}{l}(\bfy')\bfPi_{\bfG}^{\perp} \left(\ec{\bfW}{l+1}\right)^{\top} \ec{\bfb}{l+1}(\bfy')\right\| + \\ 
    & \ \ \ \ \ \frac{c_s}{d_l} \left\|\left(\ec{\bfb}{l+1}(\bfy) \right)^{\top} \ec{\bfW}{l+1} \bfPi_{\bfG}^{\perp} \ec{\bfD}{l}(\bfy)\right\| \left\|\ec{\bfD}{l}(\bfy')\bfPi_{\bfG} \left(\ec{\bfW}{l+1}\right)^{\top} \ec{\bfb}{l+1}(\bfy')\right\|  \\
    & \leq \frac{c_s}{d_l} \left[(s^{(L - l)/2} + 1)\sqrt{2 \log\left(\frac{8}{\delta}\right)} + s^{L - l} \varepsilon_3 \sqrt{2}\right]\left[(s^{(L - l)/2} + 1)\sqrt{2 \log\left(\frac{8}{\delta}\right)} + s^{L - l} \varepsilon_3 \sqrt{2}\right]  \| \ec{\bfD}{l}(\bfy) \|_2 \| \ec{\bfD}{l}(\bfy') \|_2  + \\
    & \ \ \ \ \ c_s\sqrt{\frac{2}{d_l}} \left[(s^{(L - l)/2} + 1)\sqrt{2 \log\left(\frac{8}{\delta}\right)} + s^{L - l} \varepsilon_3 \sqrt{2}\right]\| \ec{\bfD}{l}(\bfy) \|_2  \|\ec{\bfb}{l+1}(\bfy')\| \| \ec{\bfD}{l}(\bfy') \|_2   + \\ 
    & \ \ \ \ \  c_s\sqrt{\frac{2}{d_l}} \left[(s^{(L - l)/2} + 1)\sqrt{2 \log\left(\frac{8}{\delta}\right)} + s^{L - l} \varepsilon_3 \sqrt{2}\right]\| \ec{\bfD}{l}(\bfy) \|_2  \|\ec{\bfb}{l+1}(\bfy)\| \| \ec{\bfD}{l}(\bfy') \|_2  \\
    & \leq \frac{c_s}{d_l} \left[(s^{(L - l)/2} + 1)\sqrt{2 \log\left(\frac{8}{\delta}\right)} + s^{L - l} \varepsilon_3 \sqrt{2}\right]^2  \| \ec{\bfDelta}{l}(\bfy, \bfy') \|_2  + \\
    & \ \ \ \ \ c_s\sqrt{\frac{8}{d_l}} \left[(s^{(L - l)/2} + 1)\sqrt{2 \log\left(\frac{8}{\delta}\right)} + s^{L - l} \varepsilon_3 \sqrt{2}\right]  (s^{(L - l)/2} + 1) \| \ec{\bfDelta}{l}(\bfy, \bfy') \|_2      \\
    & \leq \frac{c_s \varepsilon_5}{\sqrt{d_l}}\left[(s^{(L - l)/2} + 1)\sqrt{2 \log\left(\frac{8}{\delta}\right)} + s^{L - l} \varepsilon_3 \sqrt{2}\right] \times \\
    & ~~~~~~~~~~~~~` \left\{\frac{1}{\sqrt{d_l}}\left[(s^{(L - l)/2} + 1)\sqrt{2 \log\left(\frac{8}{\delta}\right)} + s^{L - l} \varepsilon_3 \sqrt{2}\right] + \sqrt{8}(s^{(L - l)/2} + 1) \right\},
\end{align*} 
where we used Lemma~\ref{lemma:b_h_proj_G_close_to_trace_D_times_b} and~\ref{lemma:proj_G_times_W_b_bound} in fourth step and the condition of $\bar{\cB}^{l+1}(\varepsilon_2) \land \bar{\cE}^l(\varepsilon_5)$ in the last step. Both lemmas hold with probability at least $1 - \delta/2$, ensuring that the overall expression is true with probability at least $1 - \delta$.



\subsection{Proof of Lemma~\ref{lemma:b_h_to_b_h_plus_1}}
\label{proof:b_h_to_b_h_plus_1}

Fix any $(\bfy, \bfy') \in \{ (\bfx, \bfx), (\bfx, \bfx'), (\bfx', \bfx') \}$. Conditioned on $\bar{\cA}^{L}(\varepsilon_1/2) \land \bar{\cB}^{l+1}(\varepsilon_2) \land \bar{\cC}(\varepsilon_3) \land \bar{\cD}^l(\varepsilon_4) \land \bar{\cE}^l(\varepsilon_5)$, we begin with bounding $\left| \ip{\ec{\bfb}{l}(\bfy)}{\ec{\bfb}{l}(\bfy')} - \prod_{j = l}^{L+1} \ec{\dot{\Sigma}}{j}(\bfx, \bfy')\right|$. Using the definition of $\ec{\bfb}{l}(\bfy)$ we have,
\begin{align*}
    & \left| \ip{\ec{\bfb}{l}(\bfy)}{\ec{\bfb}{l}(\bfy')} - \prod_{j = l}^{L+1} \ec{\dot{\Sigma}}{j}(\bfy, \bfy')\right|  \\
    & = \left| \frac{c_s}{d_l} \left(\ec{\bfb}{l+1}(\bfy) \right)^{\top} \ec{\bfW}{l+1} \ec{\bfDelta}{l}(\bfy, \bfy') \left(\ec{\bfW}{l+1}\right)^{\top} \ec{\bfb}{l+1}(\bfy') - \prod_{j = l}^{L+1} \ec{\dot{\Sigma}}{j}(\bfy, \bfy')\right| \\
    & \leq \frac{c_s}{d_l} \bigg|\left(\ec{\bfb}{l+1}(\bfy) \right)^{\top} \ec{\bfW}{l+1} \ec{\bfDelta}{l}(\bfy, \bfy') \left(\ec{\bfW}{l+1}\right)^{\top} \ec{\bfb}{l+1}(\bfy') \\
    & \ \ \ \ \ \ \ \ \ \ \ \ \ \ \ \ - \left(\ec{\bfb}{l+1}(\bfy) \right)^{\top} \ec{\bfW}{l+1} \bfPi_{\bfG}^{\perp}\ec{\bfDelta}{l}(\bfy, \bfy')\bfPi_{\bfG}^{\perp} \left(\ec{\bfW}{l+1}\right)^{\top} \ec{\bfb}{l+1}(\bfy')\bigg| \\
    & +  \Bigg| \frac{c_s}{d_l}\left(\ec{\bfb}{l+1}(\bfy)\right)^{\top} \ec{\bfW}{l+1}\bfPi_{\bfG}^{\perp} \ec{\bfDelta}{l}(\bfy, \bfy') \bfPi_{\bfG}^{\perp}  \left(  \ec{\bfW}{l+1}\right)^{\top}\ec{\bfb}{l+1}(\bfy) - \\
    & ~~~~~~~~~~~~~~~~~~~~~~~~~~~~~~~~~~~~~~~~~~~~~~~~~~~~~~~~~~~~~~~~~~ \frac{c_s}{d_l} \tr(\ec{\bfDelta}{l}(\bfy, \bfy')) \ip{\ec{\bfb}{l+1}(\bfy)}{\ec{\bfb}{l+1}(\bfy')}   \Bigg|  \\
    &  + \left| c_{s} \frac{\tr(\ec{\bfDelta}{l}(\bfy, \bfy'))}{d_l}\ip{\ec{\bfb}{l+1}(\bfy)}{\ec{\bfb}{l+1}(\bfy')}  - \prod_{j = l}^{L+1} \ec{\dot{\Sigma}}{l}(\bfy, \bfy') \right| \\
    & \leq \Bigg( \frac{c_s \varepsilon_5}{\sqrt{d_l}}\left[(s^{(L - l)/2} + 1)\sqrt{2 \log\left(\frac{16}{\delta}\right)} + s^{L - l} \varepsilon_3 \sqrt{2}\right] \times \\
    & ~~~~~~~~~~~~~~~ \left\{\frac{1}{\sqrt{d_l}}\left[(s^{(L - l)/2} + 1)\sqrt{2 \log\left(\frac{16}{\delta}\right)} + s^{L - l} \varepsilon_3 \sqrt{2}\right] + \sqrt{8}(s^{(L - l)/2} + 1) \right\} \Bigg) \\
    & + c_s \left(\ip{\ec{\bfb}{l+1}(\bfy)}{\ec{\bfb}{l+1}(\bfy)} + \ip{\ec{\bfb}{l+1}(\bfy')}{\ec{\bfb}{l+1}(\bfy')}\right) \left( \sqrt{\frac{2}{d_l} \log\left( \frac{6}{\delta} \right) } + \frac{1}{d_l}\log\left( \frac{6}{\delta} \right) \right) \varepsilon_5 \\
    & +  \frac{2c_s}{d_l}\ip{\ec{\bfb}{l+1}(\bfy)}{\ec{\bfb}{l+1}(\bfy')} \varepsilon_5 + s^{L - l}\varepsilon_4 + s\varepsilon_2 \\
    & \leq \Bigg( \frac{c_s \varepsilon_5}{\sqrt{d_l}}\left[(s^{(L - l)/2} + 1)\sqrt{2 \log\left(\frac{16}{\delta}\right)} + s^{L - l} \varepsilon_3 \sqrt{2}\right] \times \\
    & ~~~~~~~~~~~~~~~~ \left\{\frac{1}{\sqrt{d_l}}\left[(s^{(L - l)/2} + 1)\sqrt{2 \log\left(\frac{16}{\delta}\right)} + s^{L - l} \varepsilon_3 \sqrt{2}\right] + \sqrt{8}(s^{(L - l)/2} + 1) \right\} \Bigg) \\
    & + 2c_s (s^{L - l} + 1)\left( \sqrt{\frac{2}{d_l} \log\left( \frac{6}{\delta} \right) } + \frac{1}{d_l}\log\left( \frac{6}{\delta} \right) \right) \varepsilon_5  +  \frac{2c_s}{d_l}(s^{L - l} + 1) \varepsilon_5 + s^{L - l}\varepsilon_4 + s\varepsilon_2.
\end{align*}
In the fourth step, we used Lemma~\ref{lemma:trace_D_times_b},~\ref{lemma:b_W_Delta_diff_with_its_proj} and~\ref{lemma:b_h_proj_G_close_to_trace_D_times_b} with the last two each holding with probability at least $1 - \delta/2$. Consequently, the above expression holds with probability at least $1 - \delta$.



\subsection{Proof of Lemma~\ref{lemma:a_h_implies_d_h_plus_1}}
\label{proof:a_h_implies_d_h_plus_1}

We carry out the analysis conditioned on $\bar{\cA}^{l}(\varepsilon)$. Since $\left| \left(\ec{\bfh}{l}(\bfy) \right)^{\top}\ec{\bfh}{l}(\bfy') - \ec{\Sigma}{l}(\bfy, \bfy')\right| \leq \varepsilon$ for all $(\bfy, \bfy') \in \{ (\bfx, \bfx), (\bfx, \bfx'), (\bfx',\bfx') \}$ we can conclude that $\| \ec{\tilde{\bfG}}{l} - \ec{\Lambda}{l} \|_{\infty} \leq 2 \varepsilon$. Since $\breve{\sigma}_{s-1}$ is $\beta_{s-1}$-Lipschitz in $\cM_{+}^{\gamma_{s-1}}$ w.r.t. $\infty$-norm, we have, 
\begin{align*}
    |\breve{\sigma}_{s-1}(\ec{\bfG}{l}) - \breve{\sigma}_{s-1}(\ec{\Lambda}{l+1})| \leq \beta_{s-1} \| \ec{\tilde{\bfG}}{l} - \ec{\Lambda}{l+1} \|_{\infty} \leq 2\beta_{s-1}\varepsilon,
\end{align*}
for $\varepsilon \leq \gamma_{s-1}$. Fix any $(\bfy, \bfy') \in \{ (\bfx, \bfx), (\bfx, \bfx'), (\bfx',\bfx') \}$.  We have that,
\begin{align*}
    \left| c_{s} \frac{\tr(\ec{\bfDelta}{l+1}(\bfy, \bfy'))}{d_{l+1}} - \ec{\dot{\Sigma}}{l+1}(\bfy, \bfy')  \right| & \leq
    \left| c_{s} \frac{\tr(\ec{\bfDelta}{l}(\bfy, \bfy'))}{d_l} - \frac{s^2}{2s - 1}\breve{\sigma}_{s-1}\left(\ec{\Lambda}{l + 1}(\bfy, \bfy') \right) \right| \\
    & \leq \left| c_{s} \frac{\tr(\ec{\bfDelta}{l}(\bfy, \bfy'))}{d_l} - \frac{s^2}{2s - 1}\breve{\sigma}_{s-1}\left(\ec{\tilde{\bfG}}{l}(\bfy, \bfy') \right) \right| + \\
    & \ \ \ \ \ \frac{s^2}{2s - 1} |\breve{\sigma}_{s-1}(\ec{\tilde{\bfG}}{l}) - \breve{\sigma}_{s-1}(\ec{\Lambda}{l+1})| \\
    & \leq s^2(1 + \varepsilon)^{s-1} \varphi_{s-1}(d_l, \delta/2, \min\{1, \ec{\Sigma}{l}(\bfy, \bfy') + \varepsilon\}) + \frac{2s^2\beta_s \varepsilon}{2s - 1},
\end{align*}
holds with probability $1 - \delta$. The last step follows from Lemma~\ref{lemma:trace_of_diagonal}. On taking a union bound, we can conclude that 
\begin{align*}
    \left| c_{s} \frac{\tr(\ec{\bfDelta}{l+1}(\bfy, \bfy'))}{d_{l+1}} - \ec{\dot{\Sigma}}{l+1}(\bfy, \bfy')  \right| 
    & \leq s^2(1 + \varepsilon)^{s-1} \varphi_{s-1}(d_l, \delta/6, 1) + \frac{2s^2\beta_s \varepsilon}{2s - 1},
\end{align*}
holds for all $(\bfy, \bfy') \in \{ (\bfx, \bfx), (\bfx, \bfx'), (\bfx',\bfx') \}$ with probability $1 - \delta$. Since we have already invoked Lemma~\ref{lemma:trace_of_diagonal}, it also gives us that
\begin{align*}
    \|\ec{\bfDelta}{l+1}(\bfx, \bfx')\|_2 & \leq s^2 \left[2\left(1 + \varepsilon\right)\log\left(\frac{6d_l}{\delta}\right) \right]^{s-1},
\end{align*}
holds for all $(\bfy, \bfy') \in \{ (\bfx, \bfx), (\bfx, \bfx'), (\bfx',\bfx') \}$ with probability $1 - \delta$. Using this result along with Lemma~\ref{lemma:pr_a_implies_b} and $\varepsilon \leq 1/s$ yields us
\begin{align*}
    \Pr\left[ \bar{\cA}^{l}(\varepsilon) \Rightarrow \bar{\cD}^{l + 1}\left(3s^2 \varphi_{s-1}(d_l, \delta/6, 1) + \frac{2s^2\beta_s \varepsilon}{2s - 1} \right) \land \bar{\cE}^{l+1} \left(3s^2 \left[2\log\left(\frac{6d_l}{\delta}\right) \right]^{s-1} \right) \right] \geq 1 - \delta,
\end{align*}
as required.



\subsection{Proof of Lemma~\ref{lemma:a_all_implies_d_all}}
\label{proof:a_all_implies_d_all}

Consider,
\begin{align*}
    \Pr\left[ \bar{\cA}(\varepsilon) \Rightarrow \bar{\cD}(\varepsilon') \land \bar{\cE}(\varepsilon'') \right] & = \Pr\left[ \left(\bigcap_{l = 0}^L \bar{\cA}^l (\varepsilon) \right) \Rightarrow \left(\bigcap_{l = 0}^L \left\{\bar{\cD}^{l+1}(\varepsilon') \land \bar{\cE}^{l+1}(\varepsilon'')\right\} \right) \right] \\
    & = \Pr\left[ \neg\left(\bigcap_{l = 0}^L \bar{\cA}^l (\varepsilon) \right) \lor \left(\bigcap_{l = 0}^L \left\{\bar{\cD}^{l+1}(\varepsilon') \land \bar{\cE}^{l+1}(\varepsilon'')\right\} \right) \right] \\
    & = 1 - \Pr\left[ \left(\bigcap_{l = 0}^L \bar{\cA}^l (\varepsilon) \right) \land \left(\bigcup_{l = 0}^L \neg\left\{\bar{\cD}^{l+1}(\varepsilon') \land \bar{\cE}^{l+1}(\varepsilon'')\right\} \right) \right] \\
    & = 1 - \Pr\left[ \left(\bigcup_{l = 0}^L \left(\bigcap_{l = 0}^L \bar{\cA}^l (\varepsilon) \right) \land  \neg\left\{\bar{\cD}^{l+1}(\varepsilon') \land \bar{\cE}^{l+1}(\varepsilon'')\right\} \right) \right] \\
    & \geq 1 - \sum_{l = 0}^L \Pr\left[  \left(\bigcap_{l = 0}^L \bar{\cA}^l (\varepsilon) \right) \land  \neg\left\{\bar{\cD}^{l+1}(\varepsilon') \land \bar{\cE}^{l+1}(\varepsilon'')\right\}  \right] \\
    & \geq 1 - \sum_{l = 0}^L \Pr\left[   \bar{\cA}^l (\varepsilon) \land  \neg\left\{\bar{\cD}^{l+1}(\varepsilon') \land \bar{\cE}^{l+1}(\varepsilon'')\right\}  \right] \\
    & \geq 1 - \sum_{l = 0}^L \Pr\left[  \neg \left(\bar{\cA}^l (\varepsilon) \Rightarrow \bar{\cD}^{l+1}(\varepsilon') \land \bar{\cE}^{l+1}(\varepsilon'')\right) \right] \\
    & \geq 1 - \sum_{l = 0}^L \frac{\delta}{L + 1}  \\
    & \geq 1 - \delta,
\end{align*}
where the eighth line follows from Lemma~\ref{lemma:a_h_implies_d_h_plus_1} and the choice of $\varepsilon'$ and $\varepsilon''$.



\subsection{Proof of Lemma~\ref{lemma:trace_of_diagonal}}
\label{proof:trace_of_diagonal}

We can rewrite $\tr\left(\ec{\bfDelta}{l}(\bfx, \bfx')\right)$ as follows:
\begin{align*}
    \tr(\ec{\bfDelta}{l}(\bfx, \bfx')) & = \tr\left[\ec{\bfD}{l}(\bfx)\ec{\bfD}{l}(\bfx')\right] \\
    & = \tr\left[\diag\left( \sigma_s'\left(\ec{\bff}{l}(\bfx)\right) \right) \diag\left( \sigma_s'\left(\ec{\bff}{l}(\bfx')\right) \right) \right] \\
    & = \ip{\sigma_s'\left(\ec{\bff}{l}(\bfx)\right)}{\sigma_s'\left(\ec{\bff}{l}(\bfx')\right) } \\
    & = s^2\ip{\sigma_{s-1}\left( \ec{\bfW}{l}\ec{\bfh}{l-1}(\bfx)\right)}{\sigma_{s-1}\left(\ec{\bfW}{l}\ec{\bfh}{l-1}(\bfx')\right) }.
\end{align*}
Note that since all the matrices $\{\ec{\bfW}{i}\}_{i = 1}^{l-1}$ are fixed, $\ec{\bfh}{l-1}(\cdot)$ is a fixed function. Consequently,  each term of the vector $\ec{\bfW}{l}\ec{\bfh}{l-1}(\bfx)$ is distributed as $\cN\left(0, \ip{\ec{\bfh}{l}(\bfx)}{\ec{\bfh}{l}(\bfx)}\right)$ and is independent of others. If $(\ec{\bfw}{l}_j)^{\top}$ denotes the $j^{\text{th}}$ row of the matrix $\ec{\bfW}{l}$, then we can write the inner product as
\begin{align*}
    \tr(\ec{\bfDelta}{l}(\bfx, \bfx')) & = s^2\ip{\sigma_{s-1}\left( \ec{\bfW}{l}\ec{\bfh}{l-1}(\bfx)\right)}{\sigma_{s-1}\left(\ec{\bfW}{l}\ec{\bfh}{l-1}(\bfx')\right) } \\
    & = s^2 \sum_{j = 1}^{d_{l}}\sigma_{s-1}\left( \left(\ec{\bfw}{l}_j\right)^{\top}\ec{\bfh}{l-1}(\bfx)\right)\sigma_{s-1}\left(\left(\ec{\bfw}{l}_j\right)^{\top}\ec{\bfh}{l-1}(\bfx')\right) \\
    & = s^2 \sum_{j = 1}^{d_{l}} Z_j,
\end{align*}
where $Z_j$ are all independent and identically distributed as $\sigma_{s-1}(X)\sigma_{s-1}(Y)$ where $(X, Y) \sim \cN\left(0, \ec{\tilde{\bfG}}{l-1}(\bfx, \bfx') \right) $. If we define $\tilde{Z}_j := (G_{11}G_{22})^{-\frac{(s-1)}{2}}$, then $\tilde{Z}_j$ are all independent and identically distributed as $\sigma_{s-1}(\tilde{X})\sigma_{s-1}(\tilde{Y})$ where $(\tilde{X}, \tilde{Y}) \sim \cN\left(0, \Sigma_{\rho_G} \right) $ and $\rho_{G} = \frac{G_{12}}{\sqrt{G_{11}G_{22}}}$. From Lemma~\ref{lemma:concentration_of_activation_fn}, we know that
\begin{align*}
    \left| \frac{1}{d_l}\sum_{j = 1}^{d_{l}} \tilde{Z}_j - \E[\tilde{Z}_1] \right| \leq \varphi_{s-1}(d_l, \delta, \rho_G) 
\end{align*}
with probability at least $1 - \delta$. Consequently, 
\begin{align*}
    \left| \frac{1}{d_l}\sum_{j = 1}^{d_{l}} {Z}_j - \E[Z_1] \right| \leq (G_{11}G_{22})^{\frac{(s-1)}{2}} \varphi_{s-1}(d_l, \delta, \rho_G) 
\end{align*}
with probability at least $1 - \delta$. Note that $\E[Z_1] = \underset{(X, Y) \sim \cN\left(0, \ec{\tilde{\bfG}}{h}(\bfx, \bfx') \right)}{\E} [\sigma_{s-1}(X)\sigma_{s-1}(Y)] = \frac{1}{c_{s-1}}\breve{\sigma}_{s-1}\left( \ec{\tilde{\bfG}}{l}(\bfx, \bfx')\right)$ and $c_{s-1} = c_s(2s - 1)$. On combining this with the previous result, we can conclude that,
\begin{align*}
    \left| c_{s} \frac{\tr(\ec{\bfDelta}{l}(\bfx, \bfx'))}{d_l} - \frac{s^2}{2s - 1}\breve{\sigma}_{s-1}\left(\ec{\bfG}{l-1}(\bfx, \bfx') \right) \right| & \leq s^2\left| \frac{1}{d_l}\sum_{j = 1}^{d_{l}} {Z}_j - \E[Z_1] \right| \\
    & \leq s^2(G_{11}G_{22})^{\frac{(s-1)}{2}} \varphi_{s-1}(d_l, \delta/2, \rho_G),
\end{align*}
holds with probability $1- \delta/2$. For the spectral norm of $\ec{\bfDelta}{l}(\bfx, \bfx')$, we have,
\begin{align*}
    \|\ec{\bfDelta}{l}(\bfx, \bfx')\|_2 & = s^2 \max_{j \in d_l} Z_j = s^2(G_{11}G_{22})^{\frac{(s-1)}{2}} \max_{j \in d_l} \tilde{Z}_j.
\end{align*}
Using~\eqref{eqn:tail_bound_zi}, we can show that
\begin{align*}
    \Pr\left(\max_{j \in d_l} \tilde{Z}_j > \left[(1 + \rho_{G})\log\left(\frac{d_l}{\delta}\right) \right]^{s-1}  \right) & \leq d_l \Pr\left(\tilde{Z}_1 > \left[(1 + \rho_{G})\log\left(\frac{d_l}{\delta}\right) \right]^{s-1}  \right) \\
    & \leq d_l \exp \left( -\frac{1 + \rho_{G}}{1 + \rho_{G}}\log\left(\frac{d_l}{\delta}\right) \right) \leq \delta.
\end{align*}
Consequently,
\begin{align*}
    \|\ec{\bfDelta}{l}(\bfx, \bfx')\|_2 & = s^2(G_{11}G_{22})^{\frac{(s-1)}{2}} \max_{j \in d_l} \tilde{Z}_j \\
    & \leq s^2(G_{11}G_{22})^{\frac{(s-1)}{2}} \left[(1 + \rho_{G})\log\left(\frac{2d_l}{\delta}\right) \right]^{s-1}
\end{align*}
holds with probability $1 - \delta/2$. Thus, both the hold simultaneously with probability at least $1 - \delta$.



\subsection{Proof of Lemma~\ref{lemma:proj_G_times_W_b_bound}}
\label{proof:proj_G_times_W_b_bound}

We will prove the claim for $\bfy \in \{\bfx, \bfx'\}$. Recall that $\ec{\bfG}{l}(\bfy, \bfy') = [\ec{\bfh}{l}(\bfy) \ \  \ec{\bfh}{l}(\bfy')] \in \R^{d_l \times 2}$. For simplicity, we drop the arguments from $\ec{\bfh}{l}$ and $\ec{\bfb}{l + 1}$. Define $\bfPi_{\bfh} = \bfh \bfh^{\top}$ and $\bfPi_{\bfG / \bfh} = \bfPi_{\bfG} - \bfPi_{\bfh}$. Note that $\bfPi_{\bfG / \bfh}$ is still a projection matrix of rank $0$ or $1$. \\

Recall that $\ec{\bfb}{l + 1}$ is the gradient with respect to the pre-activation layer $l + 1$ given by $\ec{\bff}{l+1} = \ec{\bfW}{l+1} \ec{\bfh}{l}$. If we view the output $f$ as a function of $\ec{\bfh}{l}, \ec{\bfW}{l+1},  \dots, \ec{\bfW}{L+1}$, then we have the following relation:
\begin{align*}
    \frac{\partial}{\partial \ec{\bfh}{l}} f(\ec{\bfh}{l}, \ec{\bfW}{l+1},  \dots, \ec{\bfW}{L+1}) = \left(\ec{\bfb}{l + 1}\right)^{\top} \ec{\bfW}{l+1}.
\end{align*}
Recall that $\sigma_s$ is $s$-homogeneous,that is, $\sigma_s(\lambda x) = \lambda^s \sigma_s(x)$ for any $\lambda > 0$. Using this recursion repeatedly, we have,
\begin{align*}
    f(\lambda \ec{\bfh}{l}, \ec{\bfW}{l+1},  \dots, \ec{\bfW}{L+1}) & = \lambda^{s^{L -l}} f(\ec{\bfh}{l}, \ec{\bfW}{l+1},  \dots, \ec{\bfW}{L+1}) \\
    \implies \frac{\partial}{\partial \lambda} f(\lambda \ec{\bfh}{l}, \ec{\bfW}{l+1},  \dots, \ec{\bfW}{L+1}) & = s^{L - l} \lambda^{(s^{L -l} - 1)} f(\ec{\bfh}{l}, \ec{\bfW}{l+1},  \dots, \ec{\bfW}{L+1}).
\end{align*}
Using this, we can write,
\begin{align*}
    f(\ec{\bfh}{l}, \ec{\bfW}{l+1},  \dots, \ec{\bfW}{L+1}) & = s^{l - L}  \frac{\partial}{\partial \lambda} f(\lambda \ec{\bfh}{l}, \ec{\bfW}{l+1},  \dots, \ec{\bfW}{L+1}) \bigg|_{\lambda = 1} \\
    & = s^{l - L}  \frac{\partial}{\partial (\lambda \ec{\bfh}{l})} f(\lambda \ec{\bfh}{l}, \ec{\bfW}{l+1},  \dots, \ec{\bfW}{L+1}) \bigg|_{\lambda = 1}\frac{\partial (\lambda \ec{\bfh}{l}) }{\partial \lambda} \bigg|_{\lambda = 1}  \\
    & = s^{l - L}  \frac{\partial}{\partial (\lambda \ec{\bfh}{l})} f(\lambda \ec{\bfh}{l}, \ec{\bfW}{l+1},  \dots, \ec{\bfW}{L+1}) \bigg|_{\lambda = 1} \ec{\bfh}{l}  \\
    & = s^{l - L}  \left(\ec{\bfb}{l + 1}\right)^{\top} \ec{\bfW}{l+1}  \ec{\bfh}{l}.
\end{align*}
By definition of $\bfPi_{\bfh}$, we have,
\begin{align*}
    \left\| \bfPi_{\bfh} \left( \ec{\bfW}{l+1} \right)^{\top} \ec{\bfb}{l+1}  \right\| & = \left\| \frac{1}{\|\ec{\bfh}{l} \|^2} \cdot \ec{\bfh}{l} \left( \ec{\bfh}{l} \right)^{\top} \left( \ec{\bfW}{l+1} \right)^{\top} \ec{\bfb}{l+1}  \right\| \\
    & = s^{L - l} \frac{|f(\bfy; \bfW)|}{\|\ec{\bfh}{l} \|}.
\end{align*}
Since $\bar{\cA}^{L}(\varepsilon_1/2) \land \bar{\cC}(\varepsilon_3)$, $\|\ec{\bfh}{l}\| \geq 1 - \varepsilon_1/2 \geq 1/2$ and $|f(\bfy; \bfW)| \leq \varepsilon_3$. Consequently, 
\begin{align*}
    \left\| \bfPi_{\bfh} \left( \ec{\bfW}{l+1} \right)^{\top} \ec{\bfb}{l+1}  \right\| & \leq s^{L - l} \varepsilon_3 \sqrt{2}.
\end{align*}

Similar to the proof of Lemma~\ref{lemma:b_h_proj_G_close_to_trace_D_times_b}, we note that for a fixed $\bfh$ and conditioned on $\ec{\bff}{l+1}$ and $\ec{\bfb}{l+1}$ Lemma~\ref{lemma:nullspace_equivalence} gives us that
\begin{align*}
    \ec{\bfW}{l+1}\bfPi_{\bfh}^{\perp} \overset{d}{=}_{\ec{\bff}{l+1} = \ec{\bfW}{l+1}\ec{\bfh}{l}} \widetilde{\bfW}\bfPi_{\bfh}^{\perp},
\end{align*}
where $\widetilde{\bfW}$ is a i.i.d. copy of $\ec{\bfW}{l+1}$. From the definition of $\bfPi_{\bfG/\bfh}$, we can conclude that $\bfPi_{\bfG/\bfh}^{\top} \bfPi_{\bfh} = 0$. Thus, combining this with the previous equation, we can conclude that
\begin{align*}
    \ec{\bfW}{l+1}\bfPi_{\bfG/\bfh} \overset{d}{=}_{\ec{\bff}{l+1} = \ec{\bfW}{l+1}\ec{\bfh}{l}} \widetilde{\bfW}\bfPi_{\bfG/\bfh}.
\end{align*}
If $\text{rank}(\bfPi_{\bfG/\bfh}) = 1$, then $\bfPi_{\bfG/\bfh} = \bfu \bfu^{\top}$ for some unit vector $\bfu$. Thus,
\begin{align*}
    \left\| \bfPi_{\bfG/\bfh} \left( \ec{\bfW}{l+1} \right)^{\top} \ec{\bfb}{l+1} \right\| & = \left\| \bfu \bfu^{\top} \left( \ec{\bfW}{l+1} \right)^{\top} \ec{\bfb}{l+1} \right\| \\
    & = \left| \bfu^{\top} \left( \ec{\bfW}{l+1} \right)^{\top} \ec{\bfb}{l+1} \right| \\
    & \overset{d}{=} \left| \bfu^{\top} \widetilde{\bfW}^{\top} \ec{\bfb}{l+1} \right| = |t|,
\end{align*}
where $t \sim \cN(0, \|\ec{\bfb}{l+1}\|^2)$. Therefore, with probability at least $1 - \delta/2$,
\begin{align*}
    |t| \leq \|\ec{\bfb}{h+1}\| \sqrt{2 \log\left(\frac{4}{\delta}\right)} \leq (s^{(L - l)/2} + 1)\sqrt{2 \log\left(\frac{4}{\delta}\right)}.
\end{align*}
In the above step, we used similar arguments as used in Appendix.~\ref{proof:a_h_implies_d_h_plus_1} to bound $\|\ec{\bfb}{l+1}\|$. If $\text{rank}(\bfPi_{\bfG/\bfh}) = 0$, $\left\| \bfPi_{\bfG/\bfh} \left( \ec{\bfW}{l+1} \right)^{\top} \ec{\bfb}{l+1} \right\| = 0$. On combining this with the previous result, we obtain that with probability $1 - \delta$
\begin{align*}
    \left\| \bfPi_{\bfG} \left( \ec{\bfW}{l+1} \right)^{\top} \ec{\bfb}{l+1} \right\| & \leq \left\| \bfPi_{\bfG/\bfh} \left( \ec{\bfW}{l+1} \right)^{\top} \ec{\bfb}{l+1} \right\| + \left\| \bfPi_{\bfh} \left( \ec{\bfW}{l+1} \right)^{\top} \ec{\bfb}{l+1} \right\| \\
    & \leq (s^{(L - l)/2} + 1)\sqrt{2 \log\left(\frac{4}{\delta}\right)} + s^{L - l} \varepsilon_3 \sqrt{2}.
\end{align*}
Taking the union bound over the two possible choices of $\bfy$ gives us the final result.


\section{Proof of Theorem~\ref{theorem:conf_interval}}
\label{sec:thm_2_proof}

The proof of Theorem~\ref{theorem:conf_interval} involves combining several smaller results, some of which follow from Theorem~\ref{theorem:appx_error} and the others are minor modifications of existing results. We being with stating these results along with proofs, if needed, and then combine them to prove Theorem~\ref{theorem:appx_error}. \\

Firstly, using Theorem~\ref{theorem:appx_error} along with Lemma 5.1 from~\citet{zhou2020neuralUCB}, we can conclude that for $m \geq \text{poly}(T, L, K, \lambda^{-1}, \lambda_0^{-1}, S^{-1}, \log(1/\delta))$, there exists a $\bfW^*$ such that $h(\bfx) = \ip{\bfg(\bfx; \bfW_0)}{\bfW^* - \bfW_0}$ with probability at least $1 - \delta$ for all $\bfx \in \cZ = \{\{\bfx_{t, a}\}_{a = 1}^K \}_{t=1}^T$. Furthermore, $\|\bfW^* - \bfW_0\|_2 \leq S$. \\

Secondly, using Lemma~\ref{theorem:gradient_approx_theorem} and~\ref{theorem:ntk_derivative_approx_theorem}, along with~\eqref{eqn:gradient_def}, we can conclude that for $m \geq \text{poly}(T, L, K, \lambda^{-1}, \lambda_0^{-1}, S^{-1}, \log(1/\delta))$, $\|\bfg(\bfx; \bfW_0)\| \leq \Oc(s^L)$, independent of $m$. \\

Thirdly, we know that the spectral norm of the Hessian of the neural net is $\Oc(m^{-1/2})$~\citep[Theorem 3.2]{liu2020linearity} for all $\bfW$ such that $\|\bfW - \bfW_0\| \leq R$, where $R$ is some finite radius. Here the implied constant depends on $s, L$ and $R$. Thus, by choosing $m$ to be sufficiently large, we can conclude that for all $\|\bfW -\bfW_0\| \leq \sqrt{T/\lambda}$, $\|\tilde{\bfH}(\bfW)\| \leq C_{s,L}m^{-1/3}$. Here $\tilde{\bfH}$ denotes the Hessian and $C_{s,L}$ denotes a constant that depends only on $s$ and $L$. Using this relation, we can conclude that for any  $\|\bfW -\bfW_0\| \leq \sqrt{T/\lambda}$, we have
\begin{align*}
    |f(\bfx;\bfW) - f(\bfx;\bfW_0) - \ip{\bfg(\bfx; \bfW_0)}{\bfW - \bfW_0}| & \leq C_{s, L} m^{-1/3} T/\lambda \\
    \implies |f(\bfx;\bfW) - \ip{\bfg(\bfx; \bfW_0)}{\bfW - \bfW_0}| & \leq C_{s, L} \lambda^{-1} m^{-1/6},
\end{align*}
where the last step follows by choosing a sufficiently large $m$ and Assumption~\ref{ass1}. Moreover, the above result holds for any $\bfx \in \cZ$. \\

Lastly, using the relation $h(\bfx) = \ip{\bfg(\bfx; \bfW_0)}{\bfW^* - \bfW_0}$ along with the result from~\citet[Theorem 1]{vakili2021optimal} on the kernel defined by neural net at initialization, i.e. $\hat{k}(\bfx; \bfx') = \ip{\bfg(\bfx; \bfW_0)}{\bfg(\bfx'; \bfW_0)}$, we can conclude that
\begin{align*}
    |h(\bfx) - \bfg^{\top}(\bfx; \bfW_0)\bfV_t^{-1}\bfr| \leq \left( S + \nu \sqrt{\frac{2}{\lambda} \log(1/\delta)} \right) \sigmahat_t(\bfx).
\end{align*}
In the above equation, $\bfV_t$ and $\sigmahat_t$ are defined with respect to the dataset $\cD$ and $\bfr = \sum_{i = 1}^t y_i \bfg(\bfx_i; \bfW_0)$. Also, we have used the independence of dataset from noise sequence to invoke the above theorem. \\

Using these results we can prove our main result. For any $\bfx \in \cZ$, we have,
\begin{align*}
    |h(\bfx) - f(\bfx; \bfW_t)| & \leq |h(\bfx) - \bfg^{\top}(\bfx; \bfW_0)\bfV^{-1}\bfr| + |\bfg^{\top}(\bfx; \bfW_0)\bfV^{-1}\bfr - f(\bfx; \bfW_t)| \\
    & \leq \left( S + \nu \sqrt{\frac{2}{\lambda} \log(1/\delta)} \right) \sigmahat_t(\bfx) + |\bfg^{\top}(\bfx; \bfW_0)\bfV_t^{-1}\bfr - f(\bfx; \bfW_t)| \\
    & \leq \frac{C_{s, L}}{\lambda m^{1/6}} + \left( S + \nu \sqrt{\frac{2}{\lambda} \log(1/\delta)} \right) \sigmahat(\bfx) + |\bfg^{\top}(\bfx; \bfW_0)\bfV_t^{-1}\bfr - \ip{\bfg(\bfx; \bfW_0)}{\bfW_t - \bfW_0}| \\
    & \leq \frac{C_{s, L}}{\lambda m^{1/6}} + \left( S + \nu \sqrt{\frac{2}{\lambda} \log(1/\delta)} \right) \sigmahat(\bfx) + \|\bfW_t - \bfW_0 - \bfV_t^{-1}\bfr\| \|\bfV_t\| \sigmahat(\bfx)  \\
    & \leq \frac{C_{s, L}}{\lambda m^{1/6}} + \left( S + \nu \sqrt{\frac{2}{\lambda} \log(1/\delta)} + (1 - \eta \lambda)^{J/2}\sqrt{t/\lambda} + \frac{C_{s,L}'}{\lambda m^{1/6}} \right) (\lambda + C_{s,L}''t) \sigmahat(\bfx),
\end{align*}
as required. As before $C_{s,L}, C_{s,L}'$ and $C_{s,L}''$ represent constants that depend only on $s$ and $L$. In the last but one step, we used the fact that for any positive definite matrix $\bfA \in \R^{d \times d}$ and $\bfx, \bfy \in \R^d$, $\ip{\bfx}{\bfy} \leq (\bfx^{T} \bfA^{-1} \bfx) \cdot (\bfy^{T} \bfA \bfy) \leq (\bfx^{T} \bfA^{-1} \bfx) \|\bfA\|_2 \|\bfy\|_2$. And in the last step, we used the results from~\citet[Lemmas B.3,B.5,C.2]{zhou2020neuralUCB} along with the bound on gradient established earlier.

\section{Additional Details on NeuralGCB}
\label{sec:thm_3_proof}

We first provide all the pseudo codes for NeuralGCB followed by proof of Theorem~\ref{theorem:regret_guarantees}.

\subsection{Pseudo Codes}

\begin{algorithm}
	\caption{NeuralGCB}
	\label{alg:NeuralGCB}
	\begin{algorithmic}[1]
		\STATE \textbf{Require}: Time horizon $T$, maximum initial variance $\sigma_0$, error probability $\delta$
		\STATE \textbf{Initialize}: $R \gets \lceil \log_2 T \rceil$, Ensemble of $R$ Neural Nets with $\bfW_0^{(r)} = \bfW_0$, $\Psi_{0}^{(r)} \leftarrow \emptyset, \ \  \forall \ r \in [R]$, $\cH \leftarrow \emptyset$, arrays \texttt{ctr}, \texttt{max\_mu}, \texttt{fb} of size $R$ with all elements set to $0$, batch sizes $\{q_r\}_{r = 1}^R$
		\FOR{$t = 1,2,3, \dots T$}
		\STATE $r \leftarrow 1, \hat{A}_r(t) = [K]$
		\WHILE{True}
		\STATE Receive the context-action pairs $\{\bfx_{t, a}\}_{a = 1}^K$
		\STATE $\{f(\bfx_{t,a}; \bfW_t^{(r)}), \sigmahat_{t-1}^{(r)}(\bfx_{t,a})\}_{a =1}^{K}, \bfW_{t}^{(r)}, \texttt{fb}[r] \leftarrow$ GetPredictions$\left(\cH, \Psi_{t-1}^{(r)}, \{\bfx_{t, a}\}_{a = 1}^K, \texttt{fb}[r], \bfW_{t-1}^{(r)}, q_r\right)$ 
		\STATE $\tilde{\sigma}_{t-1}^{(r)} \leftarrow \max_{a \in \hat{A}_r(t)} \sigmahat_{t-1}^{(r)}(x_{t, a})$, \quad $\texttt{max\_mu}[r] \leftarrow \argmax_{a \in \hat{A}_r(t)} f(x_{t, a}; \bfW_{t-1}^{(r)})$
		\IF{$\tilde{\sigma}_{t-1}^{(r)} \leq \sigma_0 2^{-r}$}
		\STATE $a_{\text{UCB}} \leftarrow \argmax_{a \in \hat{A}_r(t)} \text{UCB}_{t-1}^{(r)}(\bfx_{t,a})  $
		\IF{${\sigma}_{t-1}^{(r)}(x_{t, a_{\text{UCB}}}) \leq \eta_0/\sqrt{t}$}
		\STATE Choose $a_t \gets a_{\text{UCB}}$ and set $\upsilon_t \gets 1$
		\STATE Receive $y_t = h(\bfx_{t, a_t}) + \xi_t$ and update $\cH \leftarrow \cH \cup \{(x_{t, a_t}, y_t)\}$
		\STATE Set $\Psi_{t}^{(r+1)} \leftarrow \Psi_{t-1}^{(r+1)} \cup \{(t, \upsilon_t)\}$ and $\Psi_{t}^{(r')} \leftarrow \Psi_{t-1}^{(r')} $ for all $r' \in [R]\setminus\{r+1\}$
		\STATE \textbf{break}
		\ELSE
		\STATE $\hat{A}_{r+1}(t) \leftarrow \{a \in \hat{A}_{r}(t) : \text{UCB}_{t-1}^{(r)}(\bfx_{t,a})\geq \max_{a' \in \hat{A}_r(t)} \text{LCB}_{t-1}^{(r)}(\bfx_{t,a'})\}$, $r \leftarrow r + 1$
		\ENDIF
		\ELSE
		\IF{$r = 1$ \algorithmicor $\ \texttt{ctr}[r] > \alpha_0 4^r$}
		\STATE Choose any $a_t \in \hat{A}_r(t)$ such that $\sigmahat_{t-1}^{(r)}(x_{t, a}) > \sigma_0 2^{-r}$ and set $\upsilon_t \gets 2$
		\ELSE
		\STATE Choose $a_t \gets \texttt{max\_mu}[r-1]$ and set $\upsilon_t \gets 3$
		\ENDIF
		\STATE Receive $y_t = h(\bfx_{t, a_t}) + \xi_t$ and update $\cH \leftarrow \cH \cup \{(x_{t, a_t}, y_t)\}$, $\texttt{ctr}[r] \leftarrow \texttt{ctr}[r] + 1$
		\STATE Set $\Psi_{t}^{(r)} \leftarrow \Psi_{t-1}^{(r)} \cup \{(t, \upsilon_t)\}$ and $\Psi_{t}^{(r')} \leftarrow \Psi_{t-1}^{(r')} $ for all $r' \in [R]\setminus\{r\}$
		\STATE \textbf{break}
		\ENDIF
		\ENDWHILE
		\ENDFOR
	\end{algorithmic}
\end{algorithm}

In Alg.~\ref{alg:NeuralGCB}, $\text{UCB}_{t}^{(r)}$ and $\text{LCB}_t^{(r)}$ refer to the upper and lower confidence scores respectively at time $t$ corresponding to index $r$ and are defined as $\text{UCB}_{t}^{(r)}(\cdot) = f(\cdot; \bfW_{t}^{(r)}) + \beta_t \sigmahat_{t}^{(r)}(\cdot)$ and $\text{LCB}_{t}^{(r)}(\cdot) = f(\cdot; \bfW_{t}^{(r)}) - \beta_t \sigmahat_{t}^{(r)}(\cdot)$. The arrays \texttt{ctr}, \texttt{max\_mu} and \texttt{fb} are used to store the exploitation count, maximizer of the neural net output and feedback time instants respectively. The exploitation count is tracked to ensure that the exploitation budget is not exceeded. Similarly, the maximizer of the neural net output is stored to guide the algorithm in exploitation at the next level and the feedback time instants are used to keep track of last time instant when the neural net was retrained, or equivalently obtained a feedback. Lastly, $\upsilon_t$'s provide analytical and notational convenience and are not crucial to the working of the algorithm. \\

GetPredictions is a local routine that calculates the predictive mean and variance after appropriately retraining the neural net and described in Alg.~\ref{alg:GetPredictions} and used as a sub-routine in NeuralGCB. We would like to emphasize that NeuralGCB computes only the mean with the trained parameters. The predictive variance is computed using the gradients at initialization. This is made possible by having an additional ad hoc neural net which is just use to compute the gradients for calculating the variance. This is similar to the setup used in~\citep{kassraie2021neural}. \\

After computing the predictive variance for the set of current actions, the GetPredictions routine decides whether the neural net needs to be retrained based on the previous feedback instant $t_{\texttt{fb}}$ and batch parameter $q$. In particular, we consider two training schemes, fixed and adaptive. The fixed frequency scheme retrains the neural net if there are an additional of $q$ points in $\Psi$ after $t_{\texttt{fb}}$. Consequently, the neural corresponding to the $r^{\text{th}}$ subset in the partition is retrained after every $q_r$ points are added to the subset. On the other hand, in the adaptive scheme, inspired by the rarely switching bandits~\cite{abbasi2011improved, Wang2021}, the neural net is retrained whenever $\det(\bfV) > q \det(\bfV_{\texttt{fb}})$ where $\bfV$ and $\bfV_{\texttt{fb}}$ are as defined in line 2 and 3 of Alg.~\ref{alg:GetPredictions}. Since $\log(\det(\bfV)/\det(\lambda \Ii))$ is a measure of information gain, the neural net is effectively retrained only once sufficient additional information has been obtained since the last time the neural net was trained. We would like to point out that in the main text, we only refer to the fixed training scheme due to space constraints. However, our NeuralGCB works even with the adaptive scheme and the regret guarantees under this training schedule are formalized in Appendix~\ref{sec:thm_3_proof}. \\

Lastly, TrainNN is another local routine that carries out the training of neural net for a given dataset using gradient descent for $J$ epochs with a step size of $\eta$ starting with $\bfW_0$.

\begin{algorithm}
	\caption{GetPredictions}
	\label{alg:GetPredictions}
	\begin{algorithmic}[1]
		\STATE \textbf{Input}: Set of past actions and their corresponding rewards $\cH$, Index set $\Psi$, Current set of actions $\{\bfx_{t, a}\}_{a = 1}^{K}$, Previous feedback instant $t_{\texttt{fb}}$, $\bfW_{t_{\texttt{fb}}}$, Batch choice parameter $q$
		\STATE $\bfZ \leftarrow \lambda \Ii + \frac{1}{m}\sum_{i \in \Psi} \bfg(\bfx_i;\bfW_0)\bfg(\bfx_i;\bfW_0)^{\top}$
		\STATE $\bfZ_{\texttt{fb}} \leftarrow \lambda \Ii + \frac{1}{m}\sum_{i \in \Psi, i \leq t_{\texttt{fb}}} \bfg(\bfx_i;\bfW_0)\bfg(\bfx_i;\bfW_0)^{\top}$
		\STATE Set $\sigma^2(\bfx_{t,a}) \leftarrow \bfg(\bfx_{t, a};\bfW_0)^{\top}\bfZ^{-1}\bfg(\bfx_{t,a};\bfW_0)/m$ for all $a \in [K]$
		\STATE \textbf{For fixed batch size:} $\texttt{to\_train} = (|\Psi| - t_{\texttt{fb}} == q)$
		\STATE \textbf{For adaptive batch size:} $\texttt{to\_train} = (\det(\bfZ) > q \det(\bfZ_{\texttt{fb}}))$
		\IF{\texttt{to\_train} }
		\STATE $\bfW \leftarrow$ TrainNN$(m, L, J, \eta, \lambda, \bfW_0, \{(x_r, y_r) \in \cH : r\in \Psi\})$, $t_{\texttt{fb}} \leftarrow |\Psi|$
		\ELSE
		\STATE $\bfW \leftarrow \bfW_{t_{\texttt{fb}}}$
		\ENDIF
		\RETURN $\{f(\bfx_{t,a}; \bfW), \sigma(\bfx_{t,a})\}_{a =1}^{K}, \bfW, t_{\texttt{fb}}$.
	\end{algorithmic}
\end{algorithm}

\begin{algorithm}
	\caption{TrainNN($m, L, J, \eta, \lambda, \bfW_0, \{(\bfx_i, y_i)\}_{i = 1}^n$)}
	\label{alg:TrainNN}
	\begin{algorithmic}[1]
		\STATE Define $\mathcal{L}(\bfW) = \sum_{i = 1}^n(f(\bfx_i; \bfW) - y_i)^2 + m\lambda \|\bfW - \bfW_0\|^2$
		\FOR{$j = 0,1,\dots, J - 1$}
		\STATE $\bfW_{j+1} = \bfW_j - \eta \nabla_{\bfW} \mathcal{L}(\bfW_j)$
		\ENDFOR
		\RETURN $\bfW_J$
	\end{algorithmic}
\end{algorithm}

\subsection{Proof of Theorem~\ref{theorem:regret_guarantees}}

To analyse the regret of NeuralGCB, we divide the time horizon into three disjoint sets depending on how the action at that time instant was chosen. In particular, for $i = \{1, 2, 3\}$, we define 
\begin{align*}
    \cT_i(s) & := \{t \in \psi_{T}^{(r)}: \upsilon_t = i\} \\
    \cT_i & := \bigcup_{r = 1}^{R} \cT_i(r).
\end{align*}
Thus, $\cT_1, \cT_2 $ and$\cT_3$ consist of all the points chosen by the chosen algorithm at line $12, 21$ and $23$ respectively. We consider the regret incurred at time instants in each of these sets separately. We begin with $\cT_1$. For any $t \in \cT_1(r)$,
\begin{align*}
    h(\bfx_{t, a^*_t}) - h(\bfx_{t, a_t})  & \leq f(\bfx_{t, a^*_t}; \bfW_{t-1}^{(r)}) + \tilde{\beta} + \beta_{t-1} \sigmahat_{t-1}^{(r-1)}(\bfx_{t, a^*_t}) - (f(\bfx_{t, a_t}; \bfW_{t-1}^{(r)}) - \tilde{\beta} - \beta_{t-1} \sigmahat_{t-1}^{(r-1)}(\bfx_{t, a_t})) \\
    & \leq f(\bfx_{t, a_t}; \bfW_{t-1}^{(r)}) + \tilde{\beta} +  \beta_{t-1} \sigmahat_{t-1}^{(r-1)}(\bfx_{t, a_t}) - f(\bfx_{t, a_t}; \bfW_{t-1}^{(r)}) + \tilde{\beta} +  \beta_{t-1}  \sigmahat_{t-1}^{(r-1)}(\bfx_{t, a_t}) \\
    & \leq 2\tilde{\beta} + 2\beta_{t-1} \sigmahat_{t-1}^{(r-1)}(\bfx_{t, a_t}) \\
    & \leq 2\tilde{\beta} + \frac{2\eta_0 \beta_{t-1}}{\sqrt{t}},
\end{align*}
where we use Theorem~\ref{theorem:conf_interval} in the first step, $\tilde{\beta} = \frac{C_{s, L}}{\lambda m^{1/6}}$ and $\beta_t = \left( S + \nu \sqrt{\frac{2}{\lambda} \log(1/\delta)} + (1 - \eta \lambda)^{J/2}\sqrt{t/\lambda} + \frac{C_{s,L}'}{\lambda m^{1/6}} \right) (\lambda + C_{s,L}''t)$. Since this is independent of $r$, this relation holds for all $t \in \cT_1$. Consequently, the regret incurred over the time instants in $\cT_1$ can be bounded as
\begin{align*}
    \sum_{t \in \cT_1} h(\bfx_{t, a^*_t}) - h(\bfx_{t, a_t}) & \leq \sum_{t \in \cT_1} \left[ 2\tilde{\beta} + \frac{2\eta_0 \beta_{t-1}}{\sqrt{t}} \right] \\
    & \leq 2\tilde{\beta} |\cT_1| +  2\eta_0 \beta_T \sqrt{T}.
\end{align*}
We now consider a time instant in $\cT_2$ or $\cT_3$. Fix any $r > 1$ and $t \in \cT_2(r) \cup \cT_3(r)$ which implies that $a_t \in \hat{A}_r(t)$. Consequently, $\bfx_{t,a_t}$ satisfies the following relation: 
\begin{align*}
    f(\bfx_{t, a_t}; \bfW_{t-1}^{(r-1)}) + \beta_{t-1} \sigmahat_{t-1}^{(r-1)}(\bfx_{t, a_t}) & \geq \max_{a' \in \hat{A}_r(t)} f(\bfx_{t, a'}; \bfW_{t-1}^{(r-1)}) - \beta_{t-1} \sigmahat_{t-1}^{(r-1)}(\bfx_{t, a'}) 
\end{align*}
Note that output of the neural net is based on the parameters evaluated during the last time the neural net was trained, which we denote by $t_{\texttt{fb}}$. In other words, $f(\bfx_{t, a_t}; \bfW_{t-1}^{(r-1)}) = f(\bfx_{t, a_t}; \bfW_{t_{\texttt{fb}}}^{(r-1)})$. However, $\sigmahat_{t-1}^{(r-1)}$ is evaluated using all the points in $\Psi_{t-1}^{(r)}$. We can account for the discrepancy using Lemma 12 from~\citet{abbasi2011improved} which states that for any two positive definite matrices $\bfA \succeq \bfB \in \R^{d \times d}$ and $\bfx \in \R^{d}$, we have $\bfx^{\top}\bfA \bfx \leq \bfx^{\top}\bfB \bfx \cdot \sqrt{\det(\bfA)/\det(\bfB)}$. Using this result along with noting that $\sigmahat_{t-1}^{(r-1)}(\bfx_{t, a}) = \bfg^{\top}(\bfx_{t, a}; \bfW_0) \bfV_{t-1}^{-1}\bfg(\bfx_{t, a}; \bfW_0)$, $\sigmahat_{t_{\texttt{fb}}}^{(r-1)}(\bfx_{t, a}) = \bfg^{\top}(\bfx_{t, a}; \bfW_0) \bfV_{t_{\texttt{fb}}}^{-1}\bfg(\bfx_{t, a}; \bfW_0)$ and $\bfV_{t-1} \succeq \bfV_{\texttt{fb}}$, we can conclude that $\sigmahat_{t_{\texttt{fb}}}^{(r-1)}(\bfx_{t, a}) \leq \sigmahat_{t-1}^{(r-1)}(\bfx_{t, a}) \cdot \sqrt{\det(\bfV_{t-1})/\det(\bfV_{\texttt{fb}})}$. \\

For any batch such that $\det(\bfV_t)/\det(\bfV_{\texttt{fb}}) \leq 4$, we have, $\sigmahat_{t_{\texttt{fb}}}^{(r-1)}(\bfx_{t, a}) \leq 2\sigmahat_{t-1}^{(r-1)}(\bfx_{t, a})$ for all $a \in [K]$. Consequently, $|h(\bfx_{t, a}) - f(\bfx_{t, a_t}; \bfW_{t_{\texttt{fb}}}^{(r-1)})| \leq \tilde{\beta} + \beta_{t_{\texttt{fb}}}\sigmahat_{t_{\texttt{fb}}}^{(r-1)}(\bfx_{t, a}) \leq \tilde{\beta} + 2\beta_{t-1}\sigmahat_{t-1}^{(r-1)}(\bfx_{t, a})$. Thus, for any such batch and a time instant $t \in \cT_2(r)$, using Theorem~\ref{theorem:conf_interval} we have,
\begin{align*}
    f(\bfx_{t, a_t}; \bfW_{t-1}^{(r-1)}) + \beta_{t-1} \sigmahat_{t-1}^{(r-1)}(\bfx_{t, a_t}) & \geq \max_{a' \in \hat{A}_r(t)} f(\bfx_{t, a'}; \bfW_{t-1}^{(r-1)}) - \beta_{t-1} \sigmahat_{t-1}^{(r-1)}(\bfx_{t, a'}) \\
    \implies f(\bfx_{t, a_t}; \bfW_{t_{\texttt{fb}}}^{(r-1)}) + \beta_{t-1} \sigmahat_{t-1}^{(r-1)}(\bfx_{t, a_t}) & \geq f(\bfx_{t, a^*_t}; \bfW_{t_{\texttt{fb}}}^{(r-1)}) - \beta_{t-1} \sigmahat_{t-1}^{(r-1)}(\bfx_{t, a^*})  \\
    \implies h(\bfx_{t, a_t}) + 3 \beta_{t-1} \sigmahat_{t-1}^{(r-1)}(\bfx_{t, a_t}) + \tilde{\beta} & \geq h(\bfx_{t, a^*}) - 3\beta_{t-1}\sigmahat_{t-1}^{(r-1)}(x_{t, a^*}) - \tilde{\beta}.
\end{align*}
Since $\hat{A}_r(t) \subseteq \hat{A}_{r-1}(t)$, $\sigmahat_{t-1}^{(r-1)}(\bfx_{t, a}) \leq \tilde{\sigma}_{t-1}^{(r-1)} \leq \sigma_0 2^{1-r}$ for all $a \in \hat{A}_r(t)$. Thus,
\begin{align*}
    h(\bfx_{t, a^*_t}) - h(\bfx_{t, a_t})  & \leq 3 \beta_{t-1} \sigmahat_{t-1}^{(r-1)}(\bfx_{t, a}) + 3\beta_{t-1} \sigmahat_{t-1}^{(r-1)}(\bfx_{t, a^*_t}) + 2\tilde{\beta} \leq 12 \sigma_0 \beta_{t-1} \cdot 2^{-r} + 2 \tilde{\beta}.
\end{align*}
Similarly, for a time instant $t \in \cT_3(r)$, we have,
\begin{align*}
    h(\bfx_{t, a^*_t}) - h(\bfx_{t, a_t})  & \leq  f(\bfx_{t, a^*_t}; \bfW_{t_{\texttt{fb}}}^{(r-1)}) + \tilde{\beta} + 2\beta_{t-1} \sigmahat_{t-1}^{(r-1)}(\bfx_{t, a^*_t}) - f(\bfx_{t, a_t}; \bfW_{t_{\texttt{fb}}}^{(r-1)}) +  \tilde{\beta} + 2\beta_{t-1} \sigmahat_{t-1}^{(r-1)}(\bfx_{t, a_t}) \\
    & \leq  3\beta_{t-1} \sigmahat_{t-1}^{(r-1)}(\bfx_{t, a^*_t}) + 3 \beta_{t-1} \sigmahat_{t-1}^{(r-1)}(\bfx_{t, a}) + 2\tilde{\beta} \leq 12 \sigma_0 \beta_{t-1} \cdot 2^{-r} + 2 \tilde{\beta} \\
    & \leq 12 \sigma_0 \beta_{t-1} \cdot 2^{-r} + 2 \tilde{\beta},
\end{align*}
where the first step uses Theorem~\ref{theorem:conf_interval} and the last step uses the fact that $\bfx_{t, a_t} \in \hat{A}_{r-1}(t)$. For any $t \in \cT_2(1) \cup \cT_3(1)$, we use the trivial bound $h(\bfx_{t, a^*}) - h(\bfx_{t, a_t}) \leq 2$. \\

We now bound $|\cT_2(r)|$ using similar techniques outlined in~\citet{valko2013finite, auer2002using}. Fix any $r \geq 1$. Using the fact that for all $t \in \cT_2(r)$, $\sigmahat_{t-1}^{(r-1)}(\bfx_{t, a_t}) \geq \sigma_0 2^{-r}$ and the bound on the sum of posterior standard deviations from~\citet[Lemma 4]{Chowdhury2017bandit}, we can write
\begin{align*}
    \sigma_0 2^{-r} |\cT_2(r)| & \leq \sum_{t \in T_2(r)} \sigmahat_{t-1}^{(r-1)}(\bfx_{t, a_t}) \leq \sqrt{8 |\cT_2(r)| (\lambda \Gamma_k(T) + 1)} \\
    \implies |\cT_2(r)| & \leq 2^{r} \sqrt{8\sigma_0^{-2} |\cT_2(r)| (\lambda \Gamma_k(T) + 1)}
\end{align*}
We also used the fact that the information gain of the kernel induced by the finite neural net is close to that of NTK for sufficiently large $m$~\citep[Lemma D.5]{kassraie2021neural}. This also provides a guideline for setting the length of the exploration sequence. Specifically, we can set $\alpha_0$ such that $|\cT_3(r)|$ also satisfies $|\cT_3(r)| \leq 2^{r} \sqrt{8\sigma_0^{-2} |\cT_3(r)| (\lambda \Gamma_k(T) + 1)}$, or equivalently, $\alpha_0 = \Oc(\Gamma_k(T))$.  In general, we have $|\cT_3(r)| \leq \alpha_0 4^r \implies |\cT_3(r)| \leq 2^r \sqrt{\alpha_0 |\cT_3(r)|}$. \\

We can use these conditions to bound the regret incurred for $t \in \cT_2'$, the subset of $\cT_2$ that consists only of batches that satisfy the determinant condition. We have,
\begin{align*}
    \sum_{t \in \cT_2'} h(\bfx_{t, a^*_t}) - h(\bfx_{t, a_t}) & \leq \sum_{r = 1}^R \sum_{t \in \cT_2'(r)}h(\bfx_{t, a^*_t}) - h(\bfx_{t, a_t}) \\
    & \leq |\cT_2(1)| + \sum_{r = 2}^R \sum_{t \in \cT_2(r)} \left[12 \sigma_0 \beta_{t-1} \cdot 2^{-r} + 2 \tilde{\beta}\right] \\
    & \leq |\cT_2(1)| + \sum_{r = 2}^R \left[12 \sigma_0 \beta_{t-1} \cdot 2^{-r} + 2 \tilde{\beta}\right] |\cT_2(r)| \\
    & \leq |\cT_2(1)| + 2 \tilde{\beta}|\cT_2| + \sum_{r = 2}^R (12 \sigma_0 \beta_{t-1} \cdot 2^{-r}) \cdot 2^{r} \sqrt{8\sigma_0^{-2} |\cT_2(r)| (\lambda \Gamma_k(T) + 1)} \\
    & \leq |\cT_2(1)| + 2 \tilde{\beta}|\cT_2| + 12 \beta_{T}  \sqrt{8 |\cT_2| R (\lambda \Gamma_k(T) + 1)},
\end{align*}
where we used Cauchy-Schwarz inequality in the last step. Using a similar analysis, we can bound the regret incurred for $t \in \cT_3'$, which is defined similarly to $\cT_2'$, to obtain,
\begin{align*}
    \sum_{t \in \cT_3'} h(\bfx_{t, a^*_t}) - h(\bfx_{t, a_t}) & \leq |\cT_3(1)| + 2 \tilde{\beta}|\cT_3| + 12 \beta_{T}  \sqrt{8 |\cT_3| R \alpha_0}.
\end{align*}

Now, for batches for which the determinant condition is not satisfied, the regret incurred in any batch corresponding to subset index $r$ can be trivially bounded as $q_r$. We show that the number of such batches is bounded by $\Oc(\Gamma_k(T))$. Fix a $r \geq 1$. Let there be $B_r$ such batches for which the determinant condition is not satisfied and $\cT'(r)$ denote the set of time instants at which such batches begin. Then,
\begin{align*}
    4^{B_r} \leq \prod_{t \in \cT'(r)} \frac{\det(\bfV_{t + q_r})}{\det(\bfV_t)} \leq \frac{\det(\bfV_{T+1})}{\det(\lambda \Ii)}.
\end{align*}
Since $\log\left( \det(\bfV_{T+1})/\det(\lambda \Ii) \right)$ is $\Oc(\Gamma_k(T))$, we can conclude that $B_r$ is also $\Oc(\Gamma_k(T))$. Thus, regret incurred in all such batches can be bounded as $\Oc(\max_{r} q_r \Gamma_k(T) \log T)$ as there are $R = \log T$ batches. \\

We can combine the two cases to obtain the regret incurred over $\cT_2$ and $\cT_3$ as
\begin{align*}
    \sum_{t \in \cT_2} h(\bfx_{t, a^*_t}) - h(\bfx_{t, a_t}) & \leq |\cT_2(1)| + 2 \tilde{\beta}|\cT_2| + 12 \beta_{T}  \sqrt{8 |\cT_2| R (\lambda \Gamma_k(T) + 1)} + \Oc(\max_{r} q_r \Gamma_k(T) \log T) \\
    \sum_{t \in \cT_3} h(\bfx_{t, a^*_t}) - h(\bfx_{t, a_t}) & \leq |\cT_3(1)| + 2 \tilde{\beta}|\cT_3| + 12 \beta_{T}  \sqrt{8 |\cT_3| R \alpha_0} + \Oc(\max_{r} q_r \Gamma_k(T) \log T).
\end{align*}

The overall regret is now obtained by adding the regret incurred over $\cT_1$, $\cT_2$ and $\cT_3$. We have,
\begin{align*}
    R(T) & = \sum_{t = 1}^T h(\bfx_{t, a^*}) - h(\bfx_{t, a_t}) \\
    & = \sum_{t \in \cT_1} h(\bfx_{t, a^*}) - h(\bfx_{t, a_t}) + \sum_{t \in \cT_2} h(\bfx_{t, a^*}) - h(\bfx_{t, a_t}) + \sum_{t \in \cT_3} h(\bfx_{t, a^*}) - h(\bfx_{t, a_t}) \\ 
    & \leq 2\tilde{\beta} (|\cT_1| + |\cT_2| + |\cT_3|) + |\cT_2(1)| + |\cT_3(1)| +  2\eta_0 \beta_T \sqrt{T} + 12 \beta_{T}  \sqrt{8 |\cT_3| R \alpha_0} \\
    & \ \ \ \ \ + 12 \beta_{T}  \sqrt{8 |\cT_2| R (\lambda \Gamma_k(T) + 1)} + \Oc(\max_{r} q_r \Gamma_k(T) \log T) \\
    & \leq 2\tilde{\beta} T +  2\eta_0 \beta_T \sqrt{T} + 12 \beta_{T}  \sqrt{8 T  (\lambda \Gamma_k(T) + 1 + \alpha_0)} + \Oc(\max_{r} q_r \Gamma_k(T) \log T),
\end{align*}
where we again use the Cauchy Schwarz inequality in the last step along with the fact that $|\cT_2(1)|$ and $|\cT_3(1)|$ satisfy $\Oc(\Gamma_k(T))$. Since $m \geq \mathrm{poly}(T)$, the first term $\tilde{\beta}{T}$ is $\Oc(1)$ and using the values of $J$ and $\eta$ specified in Assumption~\ref{ass2} along with the bound on $m$, we have $\beta_T \leq \beta \leq 2S + \nu\sqrt{2\lambda^{-1} \log(1/\delta)}$. Consequently, the regret incurred by NeuralGCB satisfies $\Oc(\sqrt{T \Gamma_k(T) \log(1/\delta)} + \sqrt{T \Gamma_k(T)} + \sqrt{T \log(1/\delta)} + \max_{r} q_r \Gamma_k(T) \log T)$. \\

The above analysis carries through almost as is for the case of adaptive batch sizes. Since the batches always satisfy the determinant condition with $q_r$ instead of $4$, we do not need to separately consider the case where the determinant condition is not met. Using the same steps as above, we can conclude that the regret incurred by NeuralGCB when run with adaptive step sizes is $\Oc(\sqrt{T \Gamma_k(T) \log(1/\delta)} \cdot \max_{r} \sqrt{q_r})$. Thus for the adaptive case, we incur an additional multiplicative factor. This can be explained by the fact that in the adaptive batch setting, the number of times the neural nets are retrained is $\Oc(\Gamma_k(T))$. However, in the fixed batch setting, even if we use the optimal batch size of $\sqrt{T/\Gamma_k(T)}$, we end up retraining the neural nets about $\Oc(\sqrt{T \Gamma_k(T)})$ times, which is more than in the case of the adaptive batch setting. The more frequent of retraining of neural nets helps us achieve a tighter regret bound in the case of fixed batch setting.

\section{Empirical Studies}
\label{sec:expts_appendix}

In this section, we provide further details of the setup used during our experiments. For completeness, we first restate the construction and preprocessing of the datasets followed by the experimental setup. We then provide the complete array of results for the all the algorithm on different datasets.

\subsection{Datasets}

For each of the synthetic datasets, we construct a contextual bandit problem with a feature dimension of $d = 10$ and $K = 4$ actions per context running over a time horizon of $T = 2000$ rounds. The set of context vectors $\{\{\bfx_{t, a}\}_{a = 1}^K\}_{t = 1}^T$ are drawn uniformly from the unit sphere. Similar to~\citet{zhou2020neuralUCB}, we consider the following three reward functions:
\begin{align}
    h_1(\bfx) = 4 |\bfa^{\top}\bfx|^2; \quad \ h_2(\bfx) = 4 \sin^2(\bfa^{\top} \bfx); \quad \ h_3(\bfx) = \|\bfA \bfx\|_2.
\end{align}
For the above functions, the vector $\bfa$ is drawn uniformly from the unit sphere and each entry of matrix $\bfA$ is randomly generated from $\cN(0, 0.25)$.

We also consider two real datasets for classification namely Mushroom and Statlog (Shuttle), both of which are available on the UCI repository~\citep{UCI}. 
\paragraph{Mushroom:} The Mushroom dataset consists of $8124$ samples with $22$ features where each data point is labelled as an edible or poisonous mushroom, resulting in a binary classification problem. For the purpose of the experiments, we carry out some basic preprocessing on this dataset. We drop the attributed labeled ``veil-type" as it is the same for all the datapoints resulting in a $d=21$ dimensional feature vector. Furthermore, we randomly select $1000$ points from each class to create a smaller dataset of size $2000$ to run the experiements. \\
\paragraph{Statlog:} The Statlog (Shuttle) dataset consists of $58000$ entries with $d = 7$ attributes (we do not consider time as an attribute) each. The original dataset is divided into $7$ classes. Since the data is skewed, we convert this into a binary classification problem by combining classes. In particular, we combine the five smallest classes, namely, $2,3,5,6,7$ and combine them to form a single class and drop class $4$ resulting in a binary classification problem, with points labelled as $1$ forming the first class and ones with labels in $\{2,3,5,6,7\}$ forming the second. Once again, we select a subset of $2000$ points by randomly choosing $1000$ points from each class to use for the experiments. Lastly, we normalize the features of this dataset. \\

The classification problem is then converted into a contextual bandit problem using techniques outlined in~\cite{Li2010contextualclassification}, which is also adopted in~\citet{zhou2020neuralUCB} and~\cite{gu2021batched}.
Specifically, each datapoint in the dataset $(\bfx, y) \in \R^d \times \R$ is transformed into $K$ vectors of the form $\ec{\bfx}{1} = (\bfx, \mathbf{0}, \dots, \mathbf{0}) \dots, \ec{\bfx}{K} = (\mathbf{0}, \dots, \mathbf{0}, \bfx) \in \R^{Kd} $ corresponding to the $K$ actions associated with context $\bfx$. Here $K$ denotes the number of classes in the original classification problem. The reward function is set to $h(\ec{\bfx}{k}) = \1\{y = k\}$, that is, the agent receives a reward of $1$ if they classify the context correctly and $0$ otherwise. As mentioned above, we have $d = 21$ for Mushroom dataset and $d = 7$ for the Shuttle dataset and $K = 2$ for both of them.

\subsection{Experimental Setting}

For all the experiments, the rewards are generated by adding zero mean Gaussian noise with a standard deviation of $0.1$ to the reward function. All the experiments are run for a time horizon of $T = 2000$. We report the regret averaged over $10$ Monte Carlo runs with different random seeds along with an error bar of one standard deviation on either side. For all the datasets, we generate new rewards for each Monte Carlo run by changing the noise. Furthermore, for the real datasets, we also shuffle the context vectors for each Monte Carlo run. \\

\paragraph{Setting the hyperparameters:} For all the algorithms, the parameter $\nu$, the standard deviation proxy for the sub-Gaussian noise, to $0.1$, the standard deviation of the Gaussian noise added. Also, the RKHS norm of the reward, $S$ to $4$ for synthetic functions, and $1$ for real datasets in all the algorithms. For all the experiments, we perform a grid search for $\lambda$ and $\eta$ over $\{0.05, 0.1, 0.5\}$ and $\{0.001, 0.01, 0.1\}$, respectively, for each algorithm and choose the best ones for the corresponding algorithm and reward function. The exploration parameter $\beta_t$ is set to the value prescribed by each algorithm. For NeuralUCB and NeuralTS, the prescribed value for the exploration parameter in the original work was given as:
\begin{align*}
    \beta_t & = \beta_0\left( \nu \sqrt{\log \left(\frac{\det(\bfV_{t})}{\det(\lambda \Ii)}\right) + C_2 m^{-1/6}\sqrt{\log m} L^4t^{5/3}\lambda^{-1/6}  + 2\log(1/\delta)} + S \right) \\
    & + (\lambda + C_3 tL)\left[(1 - \eta m \lambda)^{J/2}\sqrt{t/\lambda} + m^{-1/6}\sqrt{\log m}L^{7/2}t^{5/3}\lambda^{-5/3}(1 + \sqrt{t/\lambda}) \right],
\end{align*}
where $\beta_0 = \sqrt{1 + C_1m^{-1/6}\sqrt{\log m}L^4t^{7/6}\lambda^{-7/6}}$. We ignore the smaller order terms and set $\beta_t = 2S + \nu \sqrt{\log \left(\frac{\det(\bfV_{t})}{\det(\lambda \Ii)} \right) + 2\log(1/\delta)} $ as $(\lambda + C_3 tL)(1 - \eta m \lambda)^{J/2}\sqrt{t/\lambda} \leq S$ for given choice of parameters as shown in~\citet{zhou2020neuralUCB}. For NeuralGCB and SupNNUCB, we apply the same process of ignoring the smaller order terms and bounding the additional constant term by $S$ to obtain  $\beta_t = 2S + \nu\sqrt{\frac{2}{\lambda} \log(1/\delta)}$, which is used in the algorithms. Also for NeuralGCB, $\eta_0$ was set to $0.2$ and $\alpha_0$ to $20 \log T$, as suggested in Sec.~\ref{sec:thm_3_proof}, for all experiments. The number of epochs is set to $200$ for synthetic datasets and Mushroom, and to $400$ for Statlog.  \\

\paragraph{Neural Net architecture:} We consider a 2 layered neural net as described in Equation~\eqref{eq:nn} for all the experiments. We perform two different sets of experiments with two different activation functions,  namely $\sigma_1$ (or equivalently, ReLU) and $\sigma_2$. For the experiments with $\sigma_1$ as the activation, we set the number of hidden neurons to $m = 20$ and $m=50$ for synthetic and real datasets respectively. For those with $\sigma_2$, $m$ is set to $30$ and $80$ for synthetic and real datasets, respectively. \\
  
\paragraph{Training Schedule:} For the experiments with sequential algorithms, we retrain the neural nets at every step, including NeuralGCB. For the batched algorithms, we consider a variety of training schedules. Particularly, for each setting (i.e., dataset and activation function) we consider two fixed and two adaptive batch size training schedules, which are denoted by $F1, F2$ and $A1, A2$ respectively. We specify the training schedules below for different dataset and activation functions beginning with those for $\sigma_1$ (ReLU). These were chosen to ensure that both algorithms have roughly the same running time to allow for a fair comparison.
\begin{enumerate}
    \item For Batched NeuralUCB run on a synthetic function:
    \begin{itemize}
        \item $F1$: 200 batches, or equivalently retrain after every $10$ steps.
        \item $F2$: 300 batches, or equivalently retrain after every $6-7$ steps.
        \item $A1$: $q = 2$
        \item $A2$: $q = 3$
    \end{itemize}
    \item For Batched NeuralUCB run on Mushroom:
    \begin{itemize}
        \item $F1$: 100 batches, or equivalently retrain after every $20$ steps.
        \item $F2$: 200 batches, or equivalently retrain after every $10$ steps.
        \item $A1$: $q = 500$
        \item $A2$: $q = 700$
    \end{itemize}
    \item For Batched NeuralUCB run on Shuttle:
    \begin{itemize}
        \item $F1$: 100 batches, or equivalently retrain after every $20$ steps.
        \item $F2$: 200 batches, or equivalently retrain after every $10$ steps.
        \item $A1$: $q = 5$
        \item $A2$: $q = 10$
    \end{itemize}
    \item For NeuralGCB run on a synthetic function:
    \begin{itemize}
        \item $F1$: $q_1 = 4, q_2 = 20, q_r = 40$ for $r \geq 3$ (i.e., retrain after $q_r$ steps)
        \item $F2$: $q_1 = 2, q_2 = 10, q_r = 20$ for $r \geq 3$.
        \item $A1$: $q_1 = 1.2, q_2 = 1.8, q_r = 2.7$ for $r \geq 3$.
        \item $A2$: $q_1 = 1.5, q_2 = 2.25, q_r = 3.375$ for $r \geq 3$.
    \end{itemize}
    \item For NeuralGCB run on Mushroom:
    \begin{itemize}
        \item $F1$: $q_1 = 20, q_2 = 40, q_r = 80$ for $r \geq 3$ (i.e., retrain after $q_r$ steps)
        \item $F2$: $q_1 = 10, q_2 = 20, q_r = 40$ for $r \geq 3$.
        \item $A1$: $q_1 = 300, q_2 = 1200, q_r = 4800$ for $r \geq 3$.
        \item $A2$: $q_1 = 500, q_2 = 2000, q_r = 8000$ for $r \geq 3$.
    \end{itemize}
    \item For NeuralGCB run on Shuttle:
    \begin{itemize}
        \item $F1$: $q_1 = 20, q_2 = 40, q_r = 80$ for $r \geq 3$ (i.e., retrain after $q_r$ steps)
        \item $F2$: $q_1 = 10, q_2 = 20, q_r = 40$ for $r \geq 3$.
        \item $A1$: $q_1 = 5, q_2 = 12.5, q_r = 31.25$ for $r \geq 3$.
        \item $A2$: $q_1 = 10, q_2 = 25, q_r = 62.5$ for $r \geq 3$.
    \end{itemize}
\end{enumerate}

\begin{figure*}[t]
\centering
\subfloat[$h_1(x)$ with $\sigma_1(x)$]{\label{fig:ip_s1_appendix}\centering \includegraphics[scale = 0.23]{plots/inner_product_squared_s1.pdf}}
~
\subfloat[$h_1(x)$ with $\sigma_2(x)$]{\label{fig:ip_s2_appendix}\centering \includegraphics[scale = 0.23]{plots/inner_product_squared_s2.pdf}}
~
\subfloat[$h_1(x)$ with fixed batch and $\sigma_1(x)$]{\label{fig:ip_s1_fixed_batch} \centering \includegraphics[scale = 0.23]{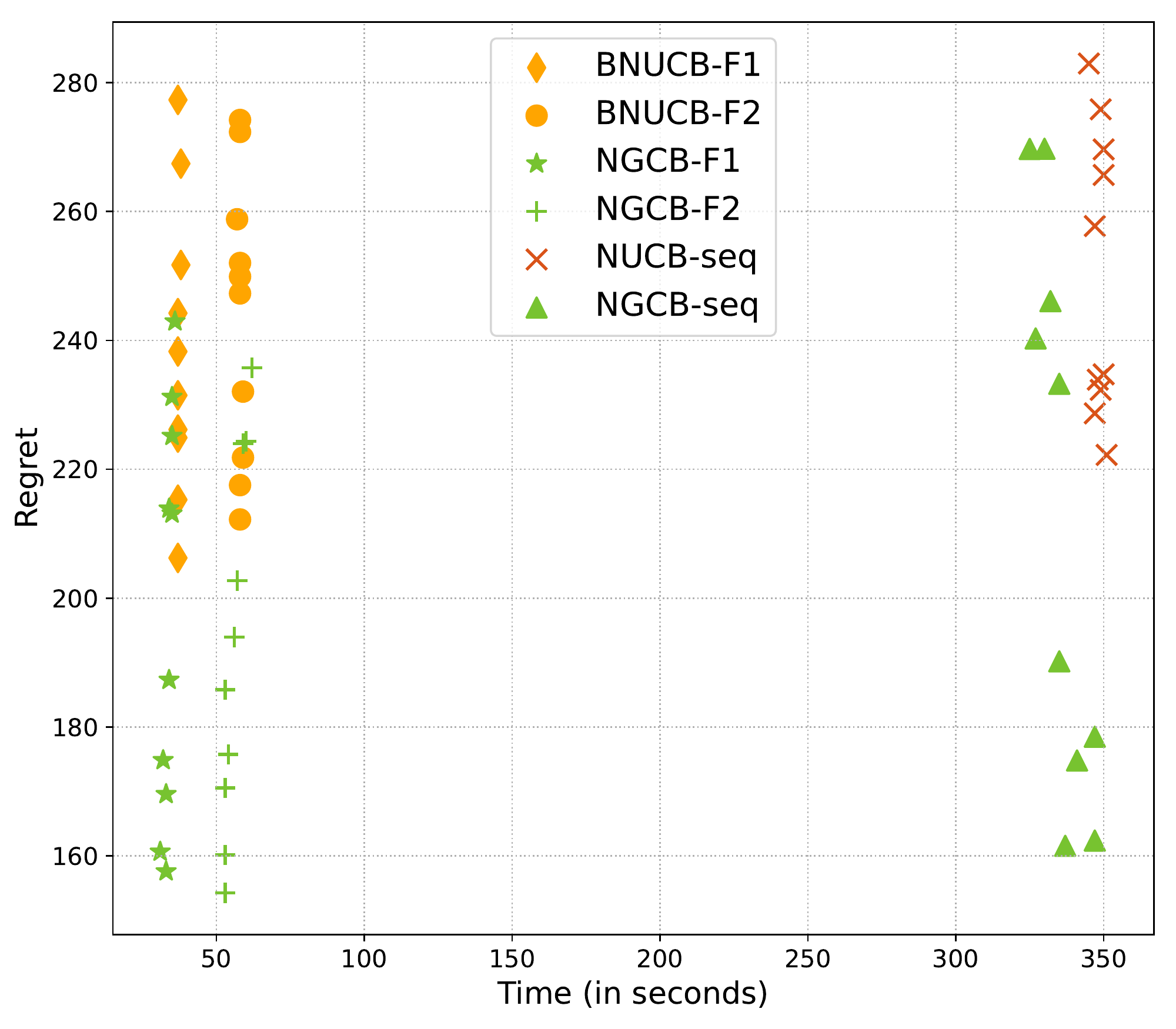}}

\subfloat[$h_1(x)$ with adaptive batch and $\sigma_1(x)$]{\label{fig:ip_s1_adaptive_batch}\centering \includegraphics[scale = 0.23]{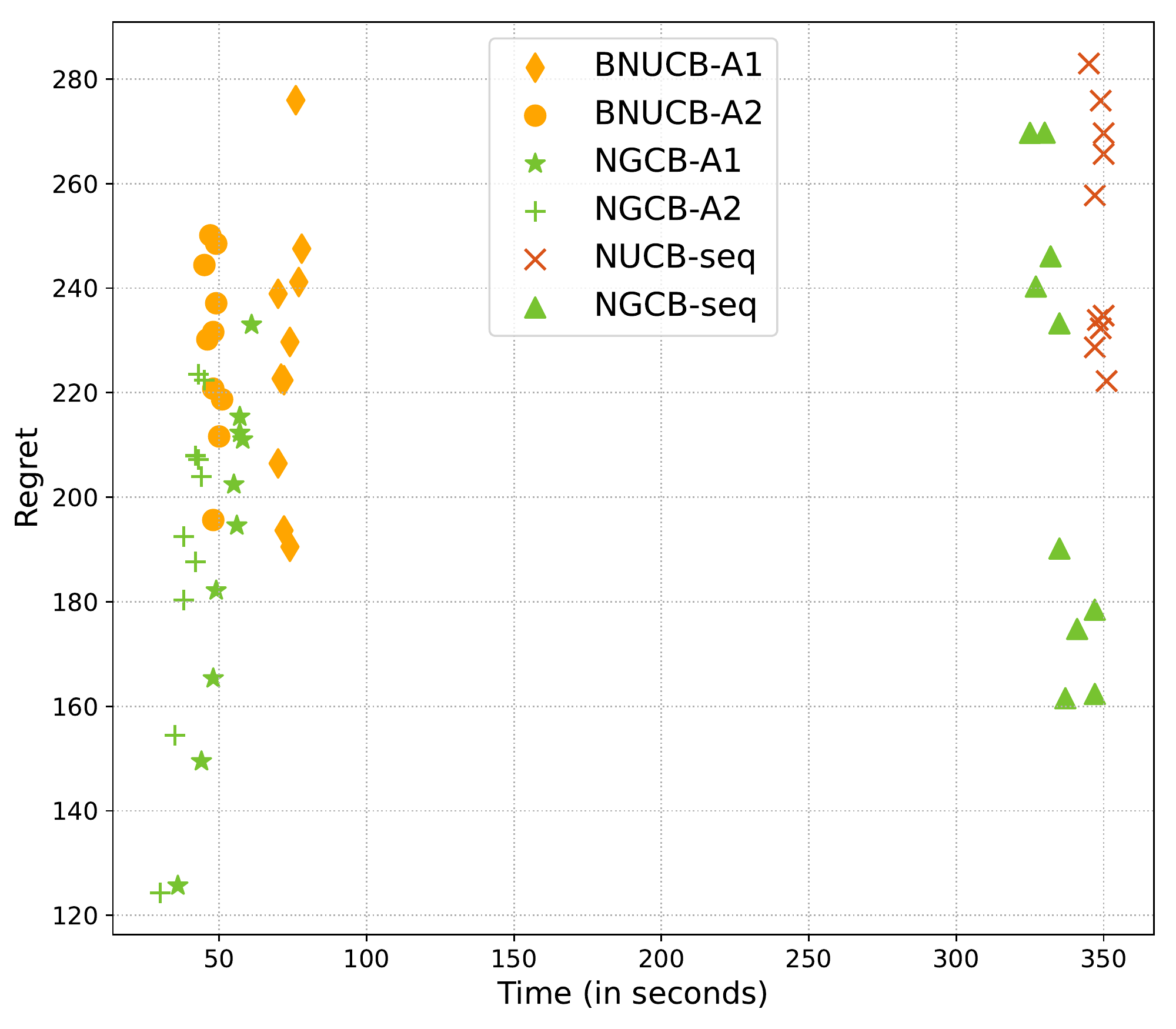}}
~
\subfloat[$h_1(x)$ with fixed batch and $\sigma_2(x)$]{\label{fig:ip_s2_fixed_batch}\centering \includegraphics[scale = 0.23]{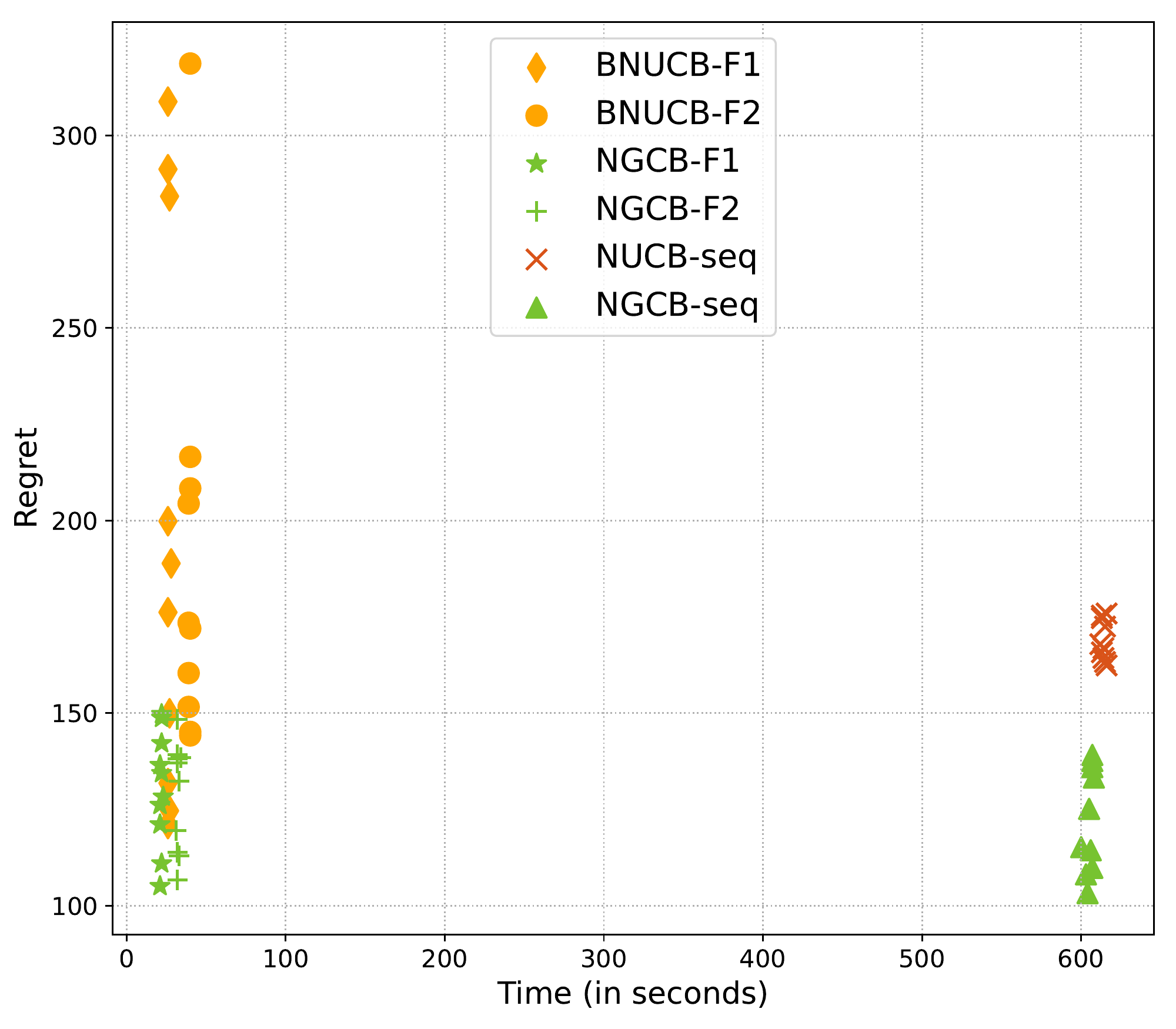}}
~
\subfloat[$h_1(x)$ with adaptive batch and $\sigma_2(x)$]{\label{fig:ip_s2_adaptive_batch} \centering \includegraphics[scale = 0.23]{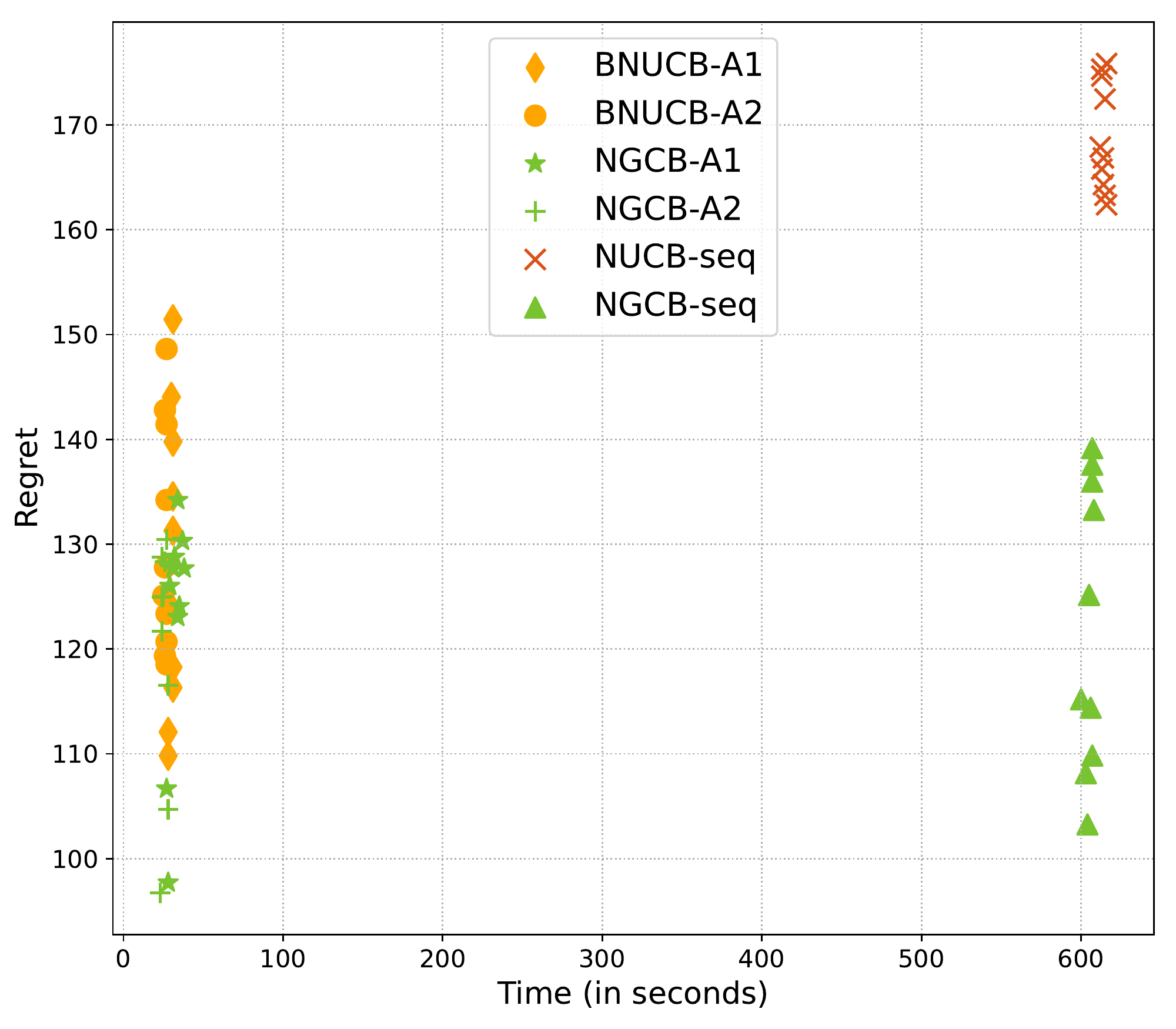}}

\caption{Plots with reward function $h_1(x)$}
\label{fig:plots_h1}
\end{figure*}

\begin{figure*}[t]
\centering
\subfloat[$h_2(x)$ with $\sigma_1(x)$]{\label{fig:cosine_s1_appendix}\centering \includegraphics[scale = 0.23]{plots/cosine_s1.pdf}}
~
\subfloat[$h_2(x)$ with $\sigma_2(x)$]{\label{fig:cosine_s2_appendix}\centering \includegraphics[scale = 0.23]{plots/cosine_s2.pdf}}
~
\subfloat[$h_2(x)$ with fixed batch and $\sigma_1(x)$]{\label{fig:cosine_s1_fixed_batch} \centering \includegraphics[scale = 0.23]{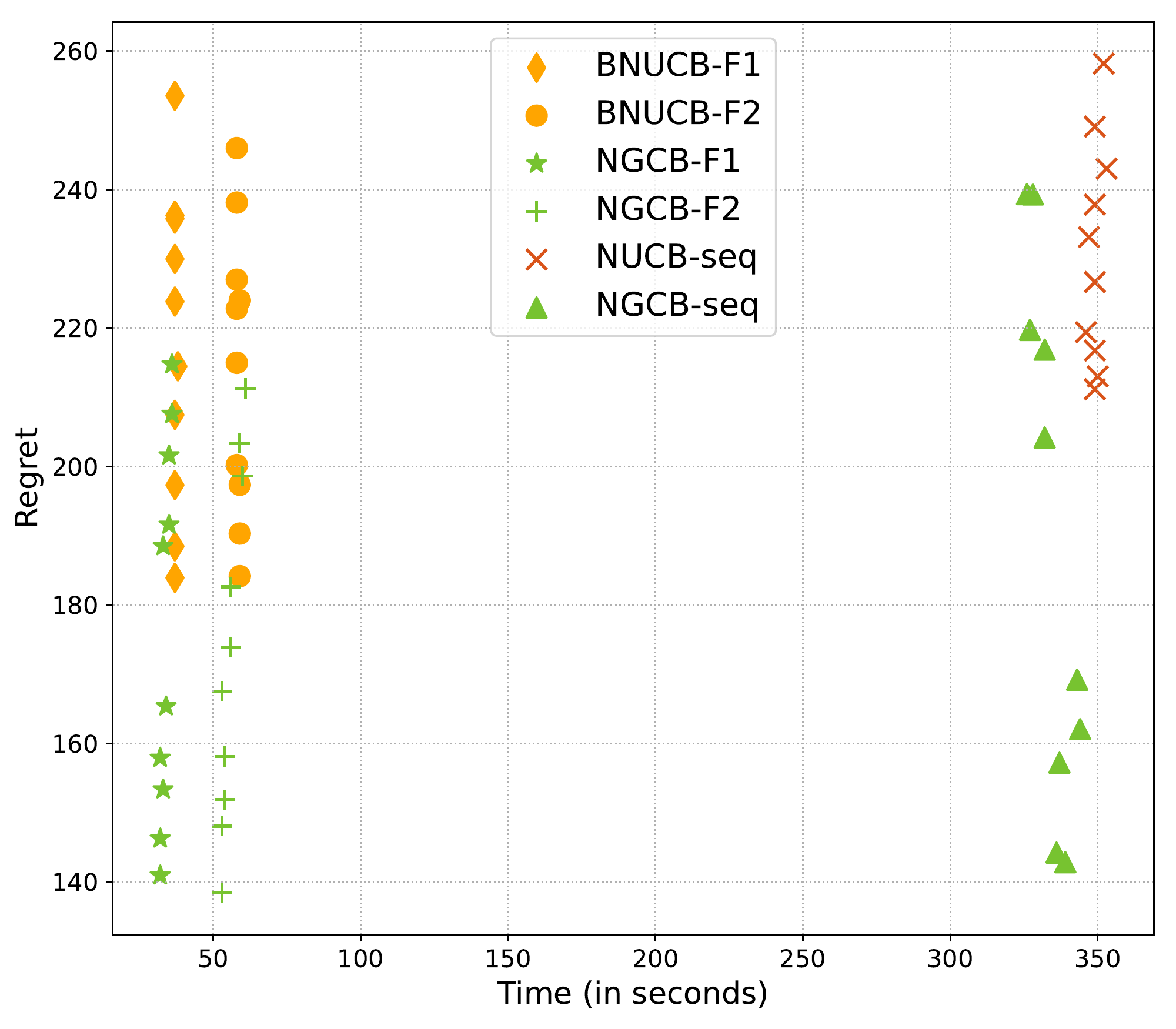}}

\subfloat[$h_2(x)$ with adaptive batch and $\sigma_1(x)$]{\label{fig:cosine_s1_adaptive_batch}\centering \includegraphics[scale = 0.23]{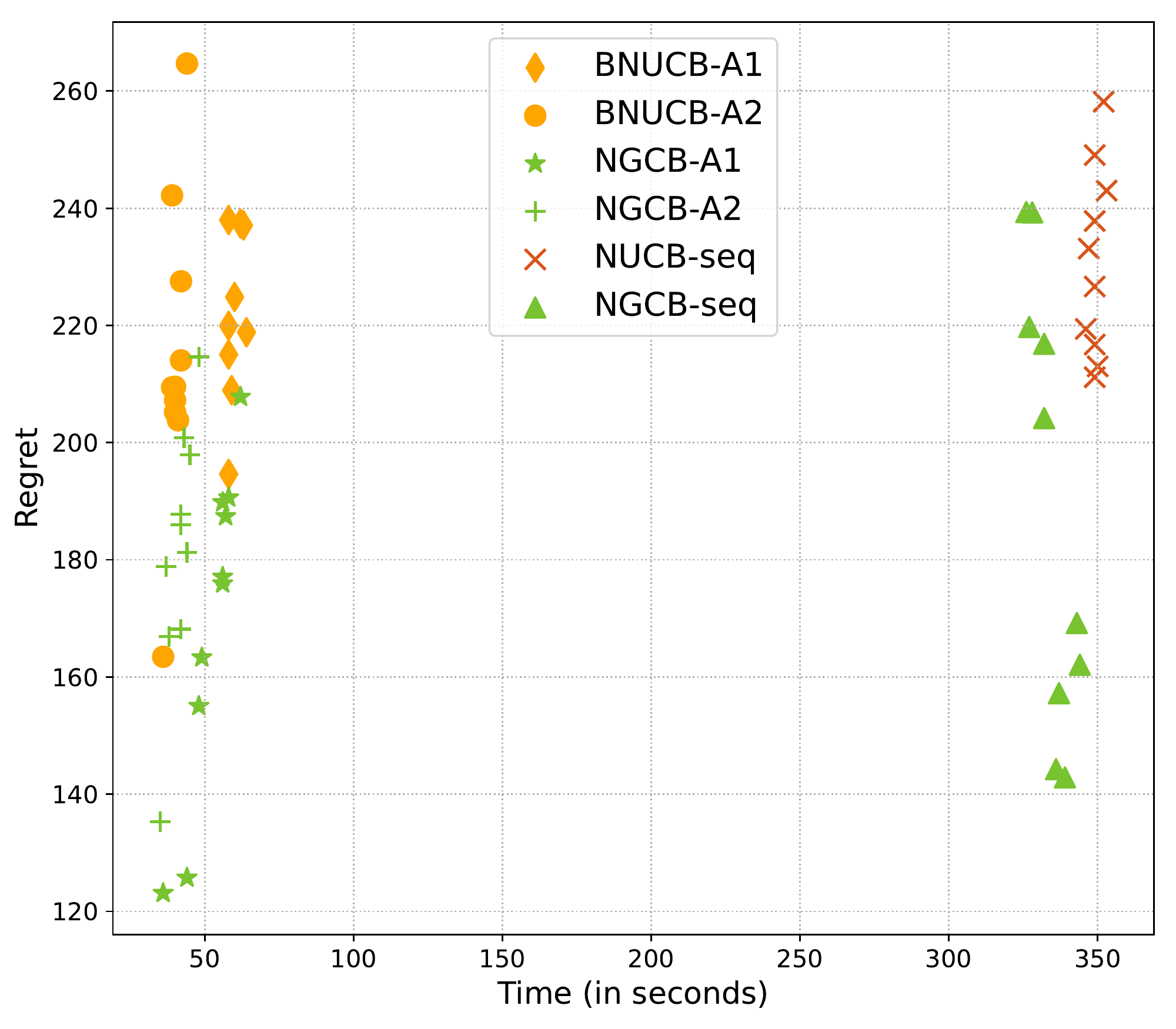}}
~
\subfloat[$h_2(x)$ with fixed batch and $\sigma_2(x)$]{\label{fig:cosine_s2_fixed_batch}\centering \includegraphics[scale = 0.23]{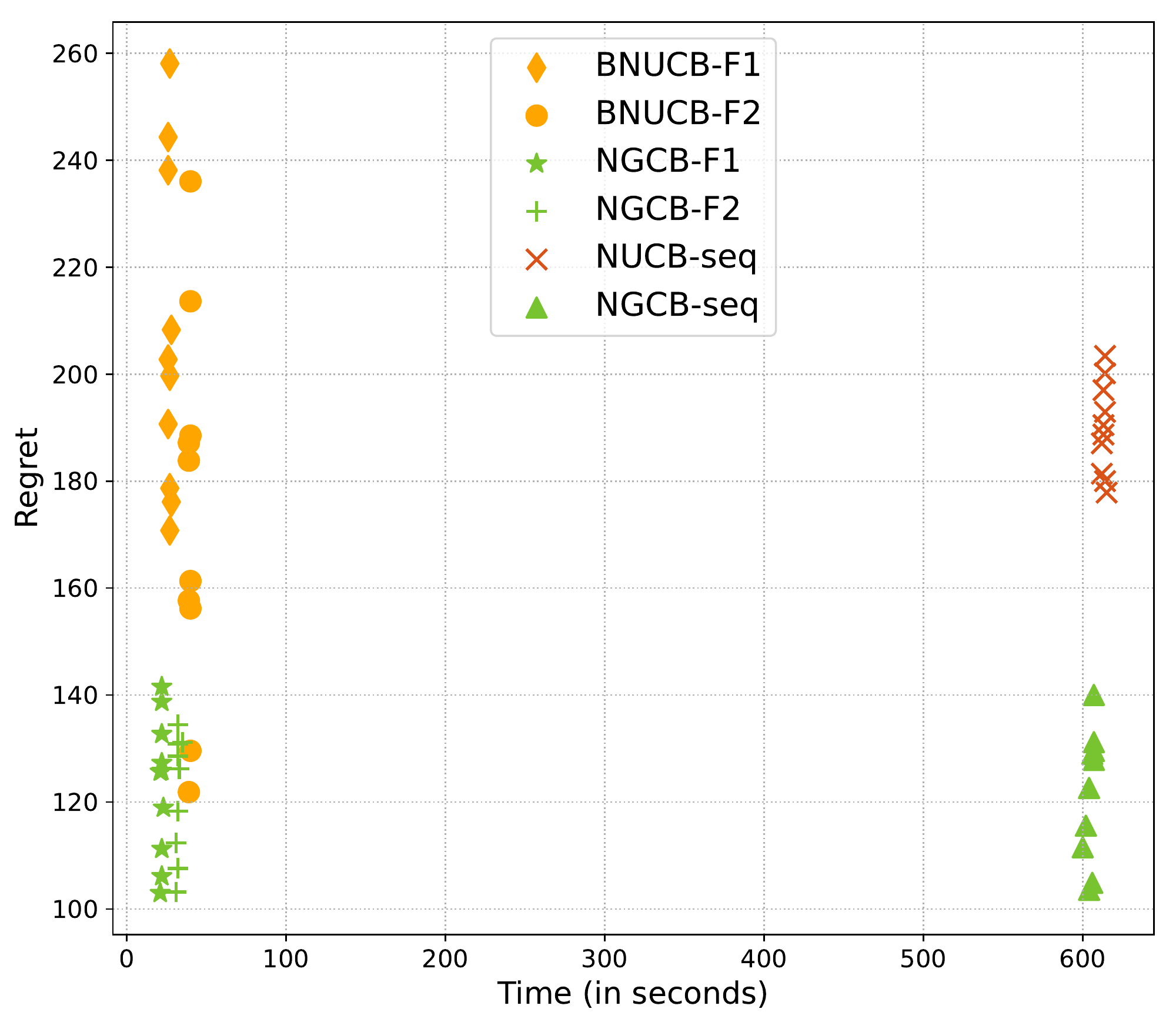}}
~
\subfloat[$h_2(x)$ with adaptive batch and $\sigma_2(x)$]{\label{fig:cosine_s2_adaptive_batch} \centering \includegraphics[scale = 0.23]{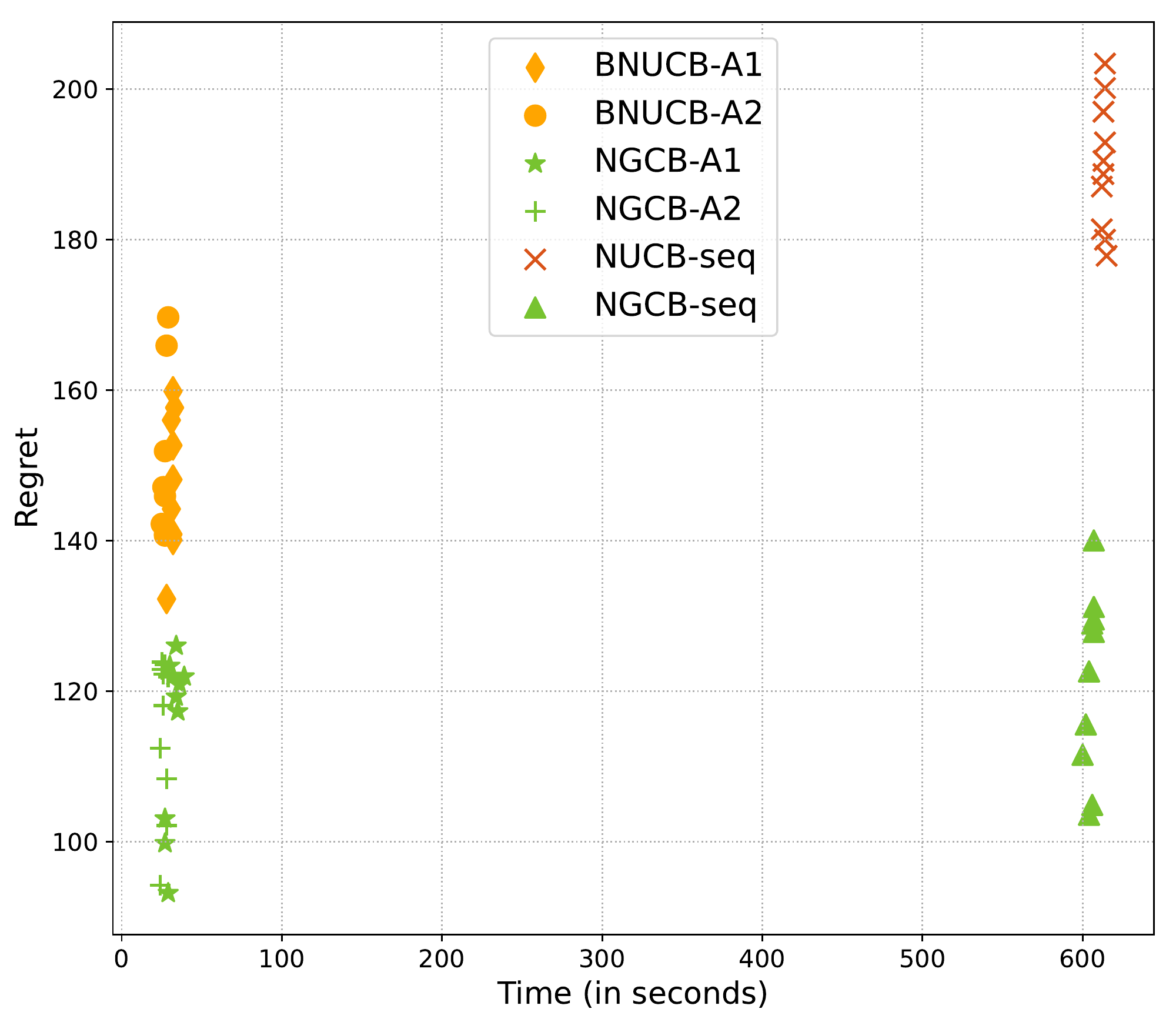}}

\caption{Plots with reward function $h_2(x)$}
\label{fig:plots_h2}
\vspace{-1em}
\end{figure*}

Similarly, for $\sigma_2$, the training schedules are given as follows:

\begin{enumerate}
    \item For Batched NeuralUCB run on a synthetic function:
    \begin{itemize}
        \item $F1$: 200 batches, or equivalently retrain after every $10$ steps.
        \item $F2$: 300 batches, or equivalently retrain after every $6-7$ steps.
        \item $A1$: $q = 3$
        \item $A2$: $q = 4$
    \end{itemize}
    \item For Batched NeuralUCB run on Mushroom:
    \begin{itemize}
        \item $F1$: 100 batches, or equivalently retrain after every $20$ steps.
        \item $F2$: 200 batches, or equivalently retrain after every $10$ steps.
        \item $A1$: $q = 500$
        \item $A2$: $q = 700$
    \end{itemize}
    \item For Batched NeuralUCB run on Shuttle:
    \begin{itemize}
        \item $F1$: 50 batches, or equivalently retrain after every $40$ steps.
        \item $F2$: 100 batches, or equivalently retrain after every $20$ steps.
        \item $A1$: $q = 5$
        \item $A2$: $q = 10$
    \end{itemize}
    \item For NeuralGCB run on a synthetic function:
    \begin{itemize}
        \item $F1$: $q_1 = 4, q_2 = 20, q_r = 40$ for $r \geq 3$ (i.e., retrain after $q_r$ steps)
        \item $F2$: $q_1 = 2, q_2 = 10, q_r = 20$ for $r \geq 3$.
        \item $A1$: $q_1 = 1.2, q_2 = 1.8, q_r = 2.7$ for $r \geq 3$.
        \item $A2$: $q_1 = 1.5, q_2 = 2.25, q_r = 3.375$ for $r \geq 3$.
    \end{itemize}
    \item For NeuralGCB run on Mushroom:
    \begin{itemize}
        \item $F1$: $q_1 = 10, q_2 = 30, q_r = 60$ for $r \geq 3$ (i.e., retrain after $q_r$ steps)
        \item $F2$: $q_1 = 10, q_2 = 20, q_r = 40$ for $r \geq 3$.
        \item $A1$: $q_1 = 300, q_2 = 1200, q_r = 4800$ for $r \geq 3$.
        \item $A2$: $q_1 = 500, q_2 = 2000, q_r = 8000$ for $r \geq 3$.
    \end{itemize}
    \item For NeuralGCB run on Shuttle:
    \begin{itemize}
        \item $F1$: $q_1 = 5, q_2 = 12, q_r = 31$ for $r \geq 3$ (i.e., retrain after $q_r$ steps)
        \item $F2$: $q_1 = 5, q_2 = 10, q_r = 20$ for $r \geq 3$ (i.e., retrain after $q_r$ steps)
        \item $A1$: $q_r = 3$ for all $r$
        \item $A2$: $q_r = 4$ for all $r$
    \end{itemize}
\end{enumerate}

We use this notation of $F1, F2, A1$ and $A2$ to refer to the training schedules in the plots shown in the next section.

\begin{figure*}[t]
\centering
\subfloat[$h_3(x)$ with $\sigma_1(x)$]{\label{fig:xAAx_s1}\centering \includegraphics[scale = 0.23]{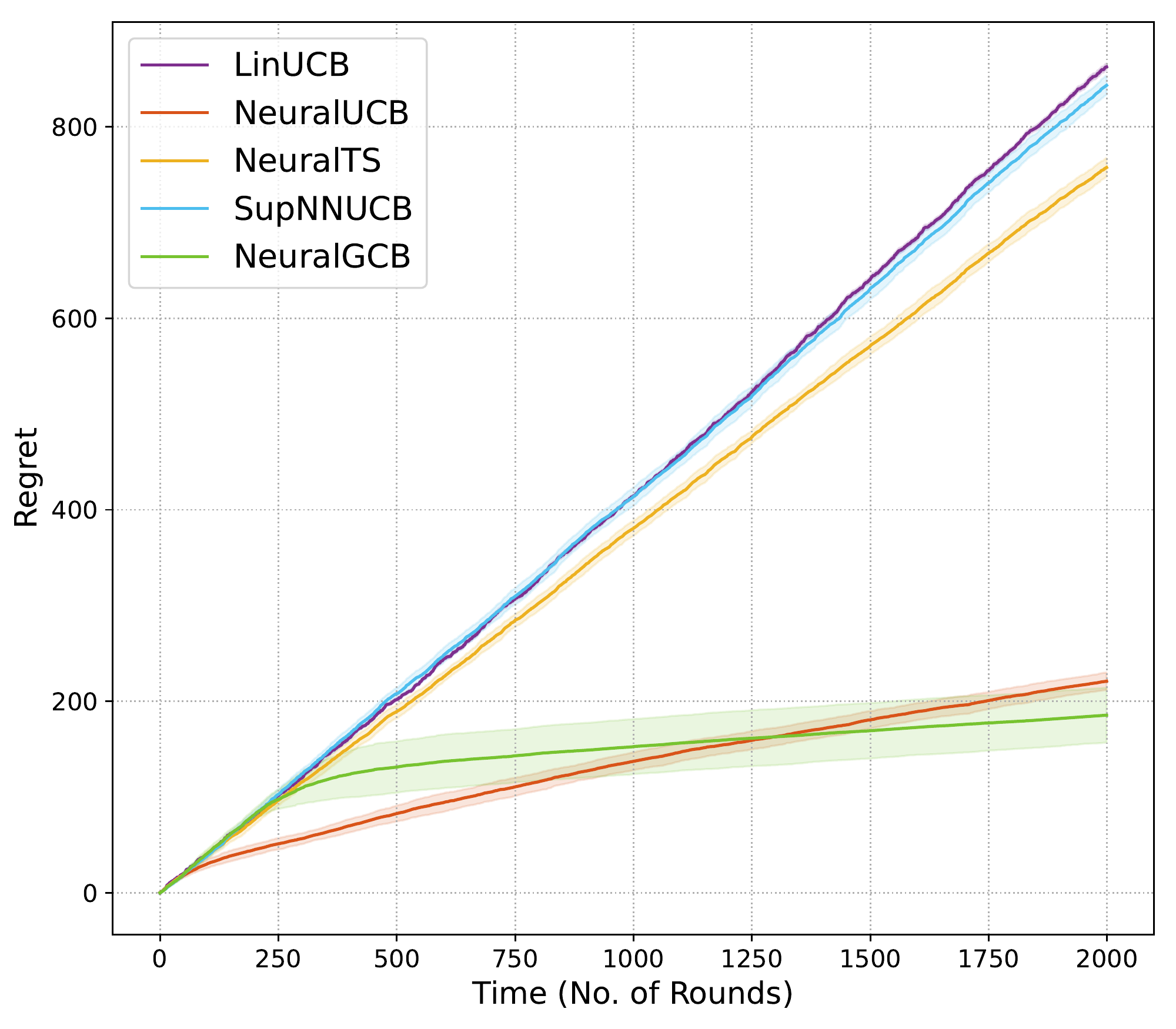}}
~
\subfloat[$h_3(x)$ with $\sigma_2(x)$]{\label{fig:xAAx_s2}\centering \includegraphics[scale = 0.23]{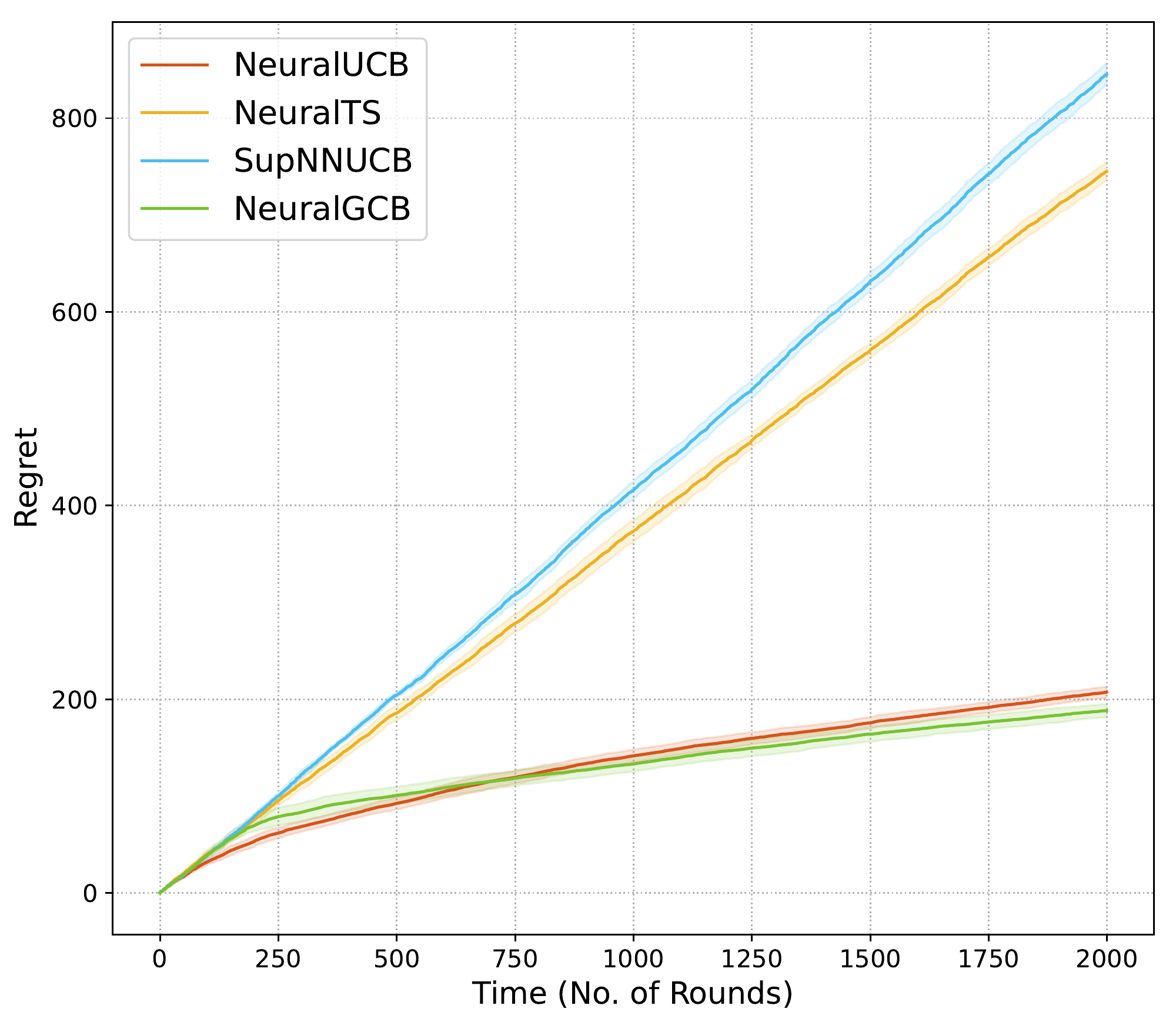}}
~
\subfloat[$h_3(x)$ with fixed batch and $\sigma_1(x)$]{\label{fig:xAAx_s1_fixed_batch} \centering \includegraphics[scale = 0.23]{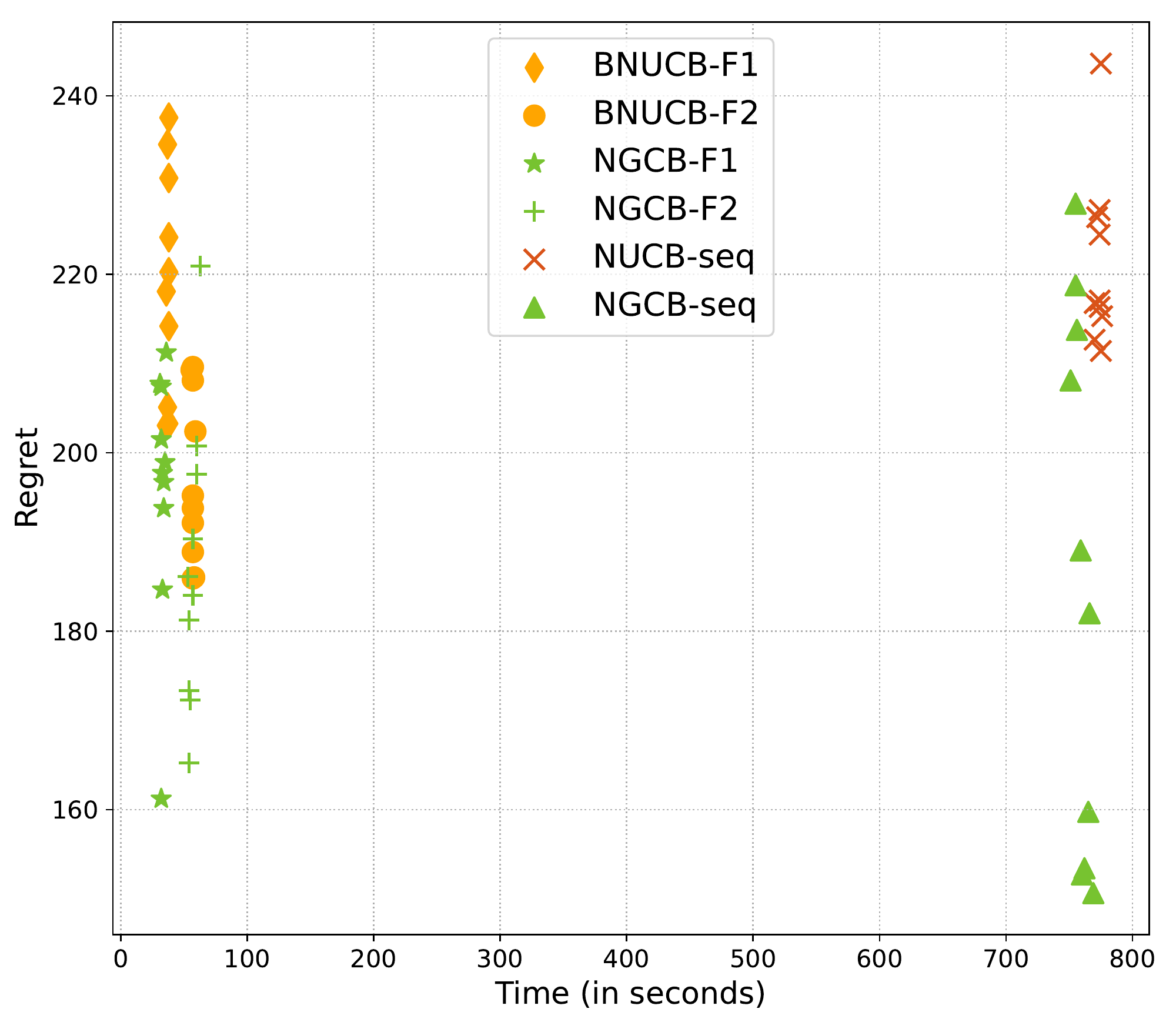}}

\subfloat[$h_3(x)$ with adaptive batch and $\sigma_1(x)$]{\label{fig:xAAx_s1_adaptive_batch}\centering \includegraphics[scale = 0.23]{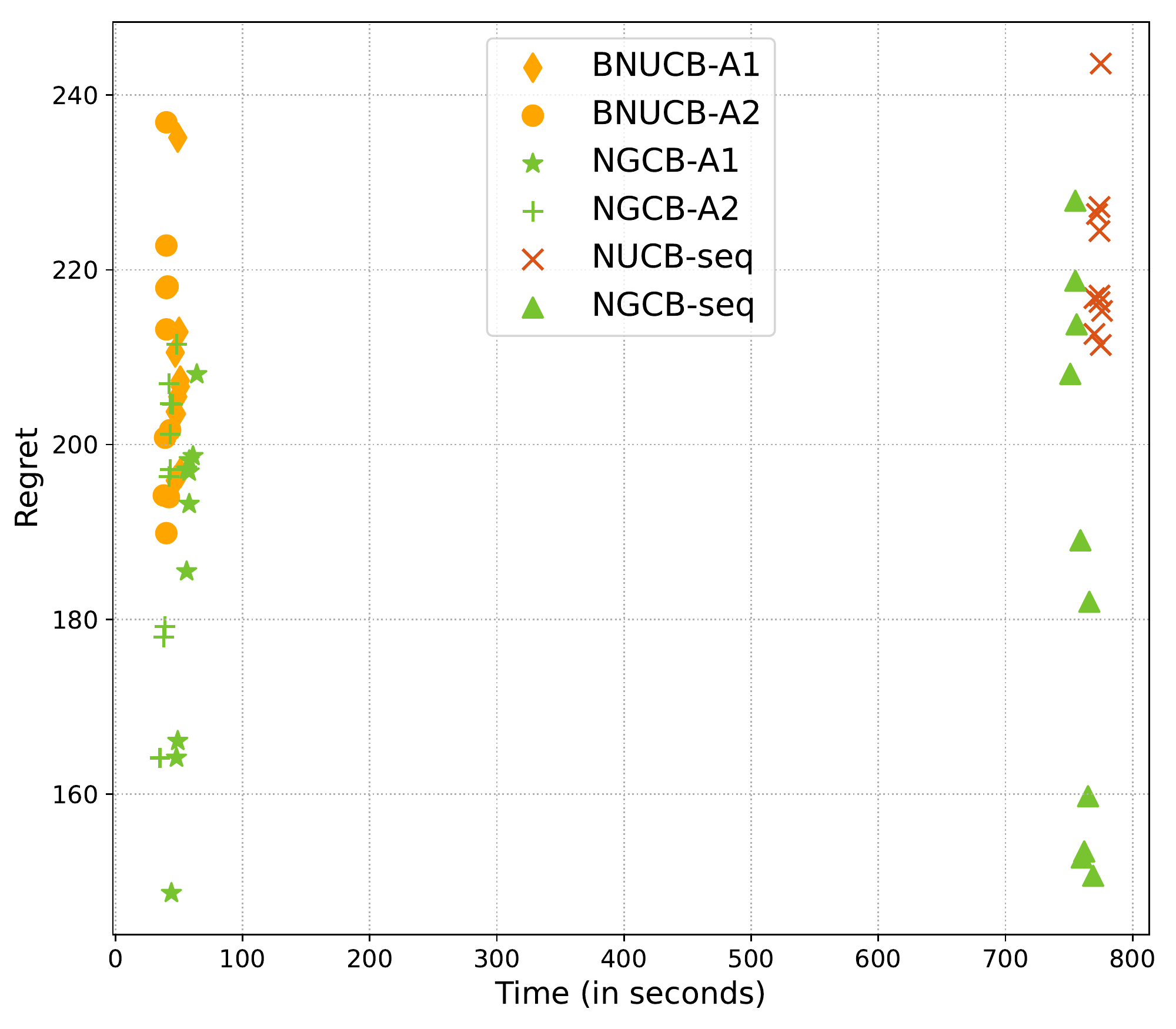}}
~
\subfloat[$h_3(x)$ with fixed batch and $\sigma_2(x)$]{\label{fig:xAAx_s2_fixed_batch}\centering \includegraphics[scale = 0.23]{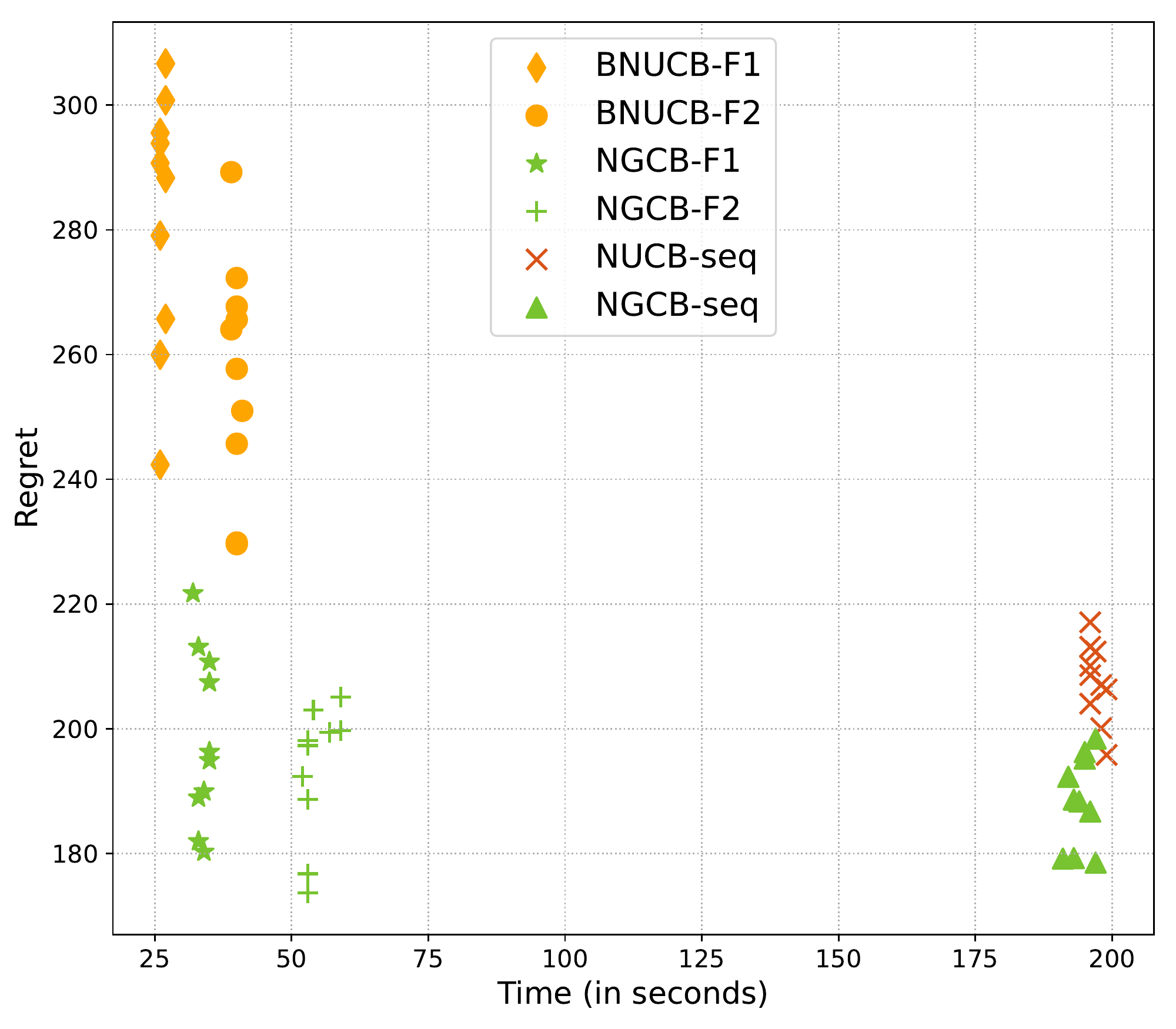}}
~
\subfloat[$h_3(x)$ with adaptive batch and $\sigma_2(x)$]{\label{fig:xAAx_s2_adaptive_batch} \centering \includegraphics[scale = 0.23]{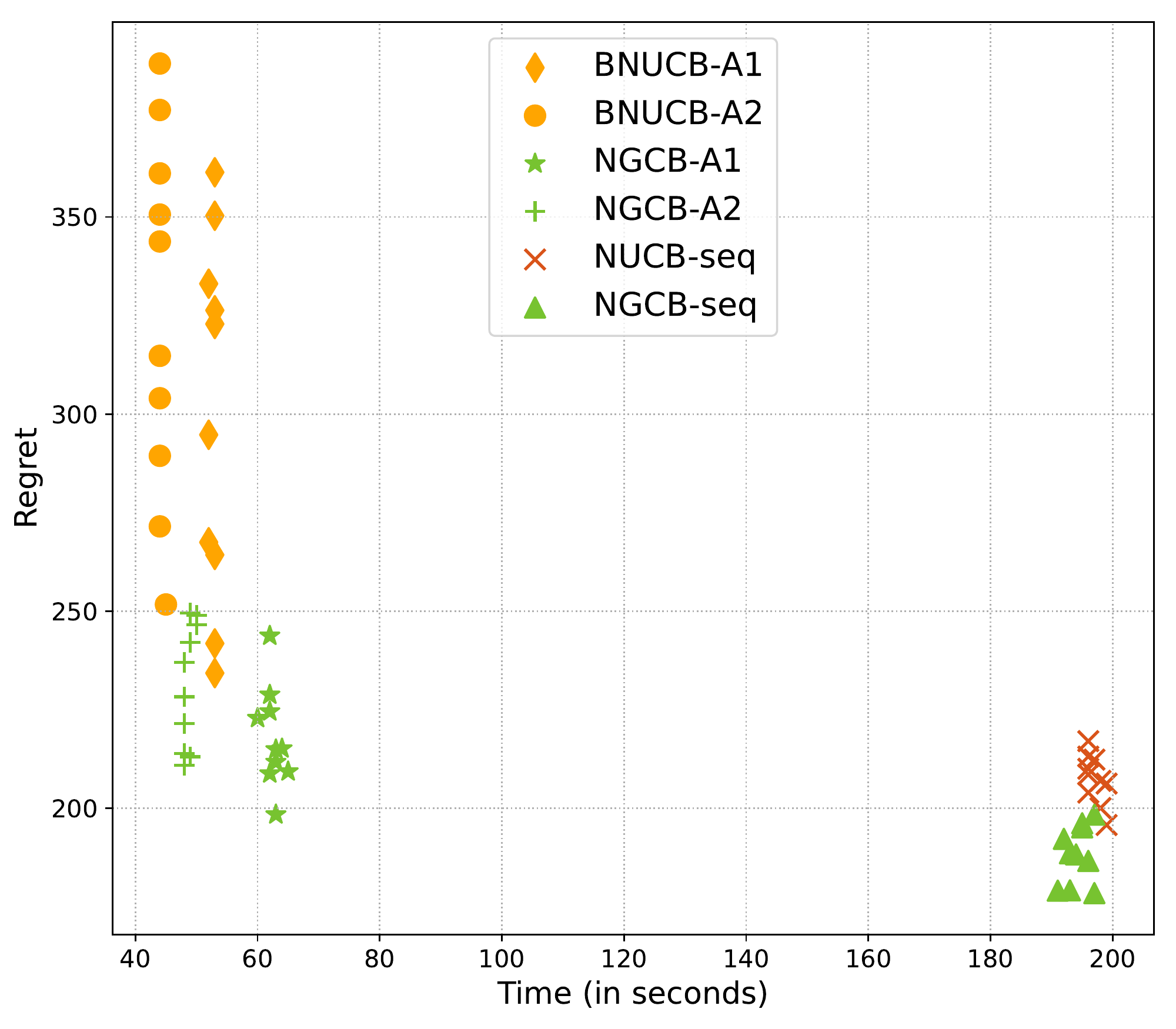}}

\caption{Plots with reward function $h_3(x)$}
\label{fig:plots_h3}
\vspace{-1em}
\end{figure*}

\subsection{Plots}

In this section, we plot the performance of NeuralGCB against several baselines for different experimental setups and reward functions. For each reward function, we plot the following $6$ figures:
\begin{itemize}
    \item Fig. (a): We plot the cumulative regret against time (number of rounds) for various baselines for the case where the activation function of the underlying neural net is $\sigma_1$, or equivalently, ReLU.
    \item Fig (b): We plot the cumulative regret against time (number of rounds) for various baselines for the case where the activation function of the underlying neural net is $\sigma_2$.
    \item Fig (c)-(f): We compare the regret and running time (in seconds) between the batched and the sequential versions of NeuralUCB and NeuralGCB. We plot the regret incurred against time taken (running time, in seconds) for Batched-NeuralUCB (BNUCB), (Batched)-NeuralGCB (BNGCB), NeuralUCB (NUCB-seq), (Sequential) NeuralGCB (NGCB-seq) under different training schedules for $10$ different runs. In particular, in (c) and (e), we consider fixed batch sizes with $\sigma_1$ and $\sigma_2$ as activation functions respectively and in (d) and (f) with consider adaptive batch sizes with $\sigma_1$ and $\sigma_2$ as activation functions respectively. The batch sizes used in the plots are described in the previous section.
\end{itemize}

\begin{figure*}[t]
\centering
\subfloat[Mushroom with $\sigma_1(x)$]{\label{fig:mushroom_s1_appendix}\centering \includegraphics[scale = 0.23]{plots/mushroom_s1.pdf}}
~
\subfloat[Mushroom with $\sigma_2(x)$]{\label{fig:mushroom_s2_appendix}\centering \includegraphics[scale = 0.23]{plots/mushroom_s2.pdf}}
~
\subfloat[Mushroom with fixed batch and $\sigma_1(x)$]{\label{fig:mushroom_s1_fixed_batch} \centering \includegraphics[scale = 0.23]{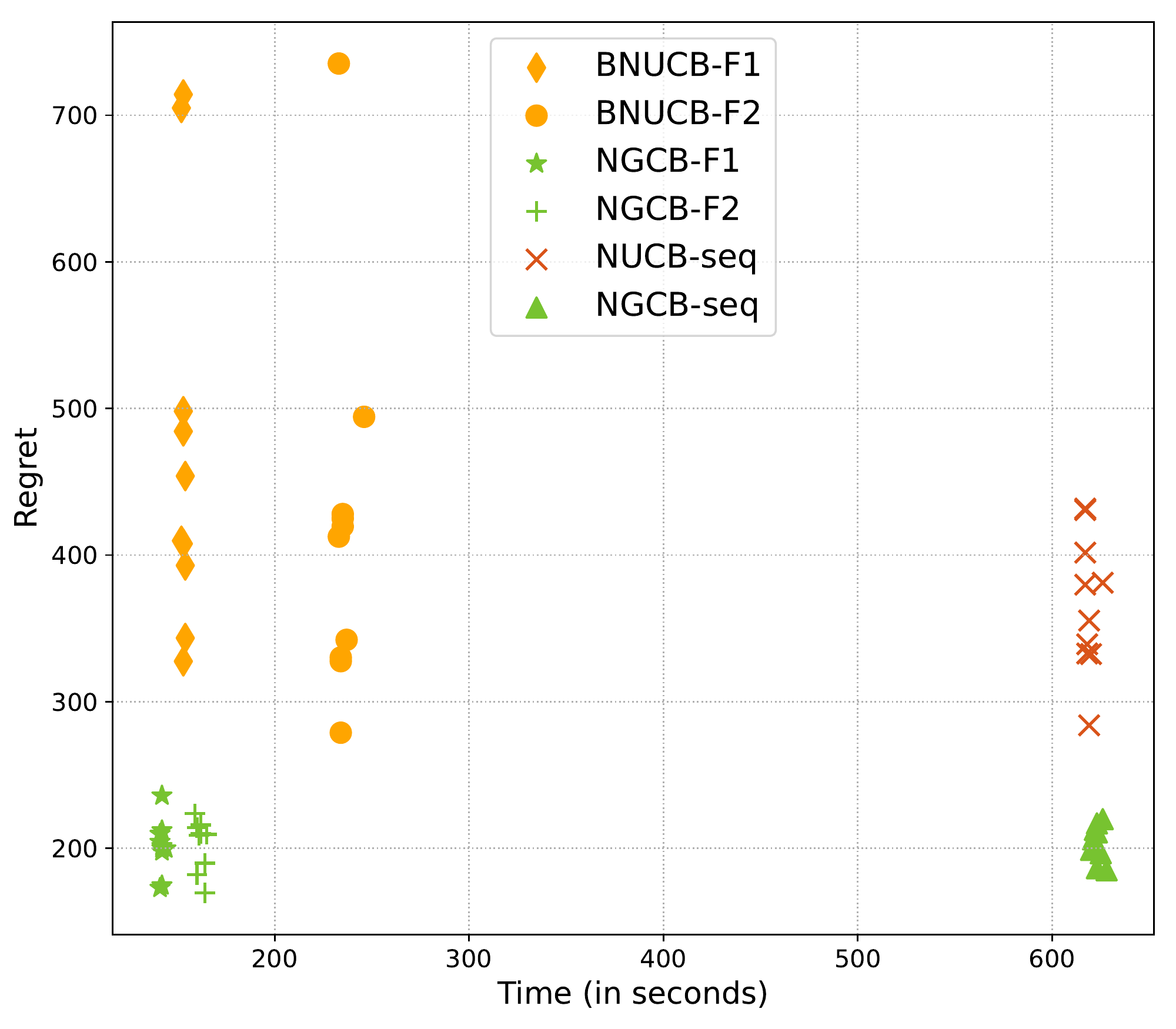}}

\subfloat[Mushroom with adaptive batch and $\sigma_1(x)$]{\label{fig:mushroom_s1_adaptive_batch}\centering \includegraphics[scale = 0.23]{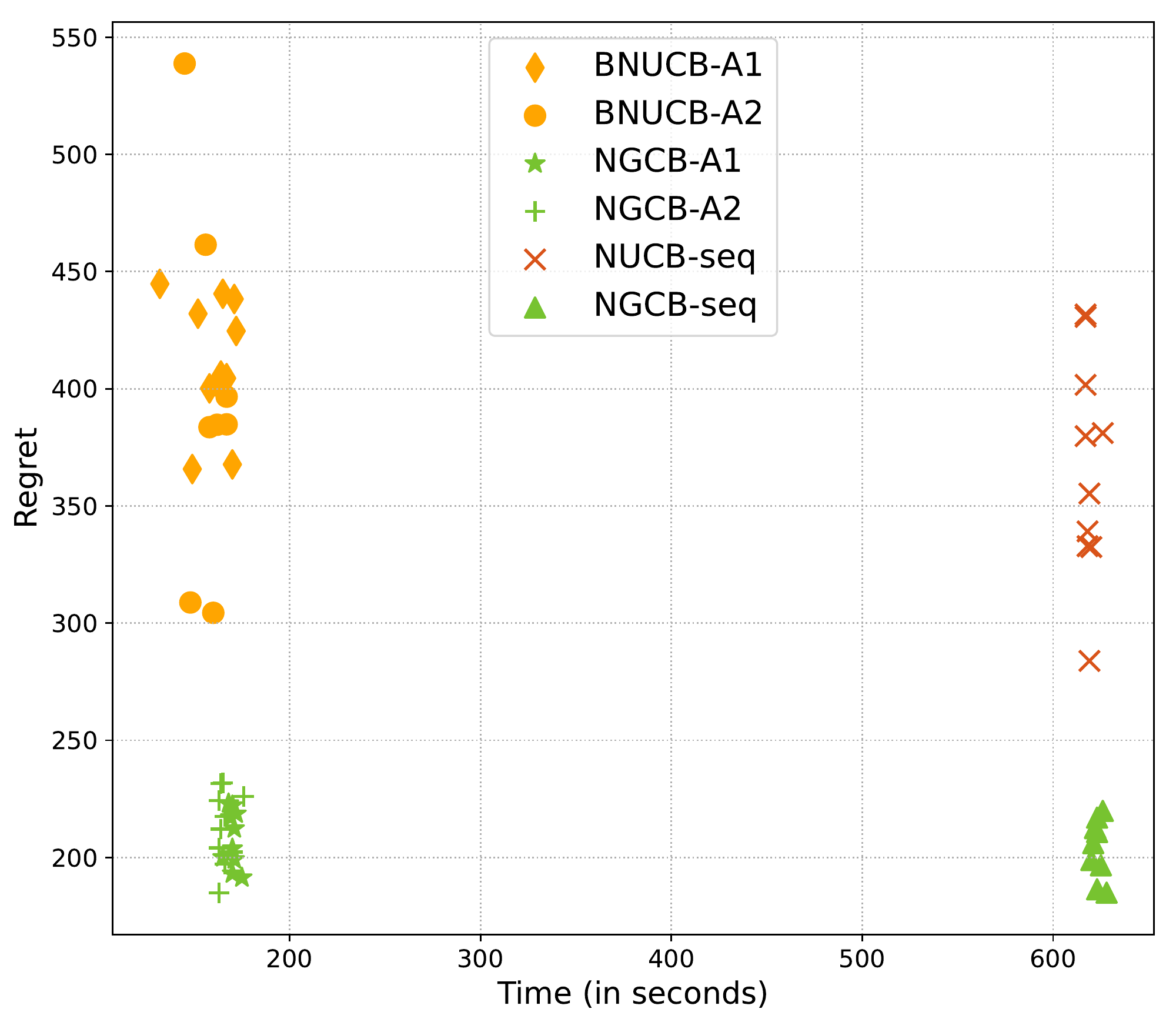}}
~
\subfloat[Mushroom with fixed batch and $\sigma_2(x)$]{\label{fig:mushroom_s2_fixed_batch}\centering \includegraphics[scale = 0.23]{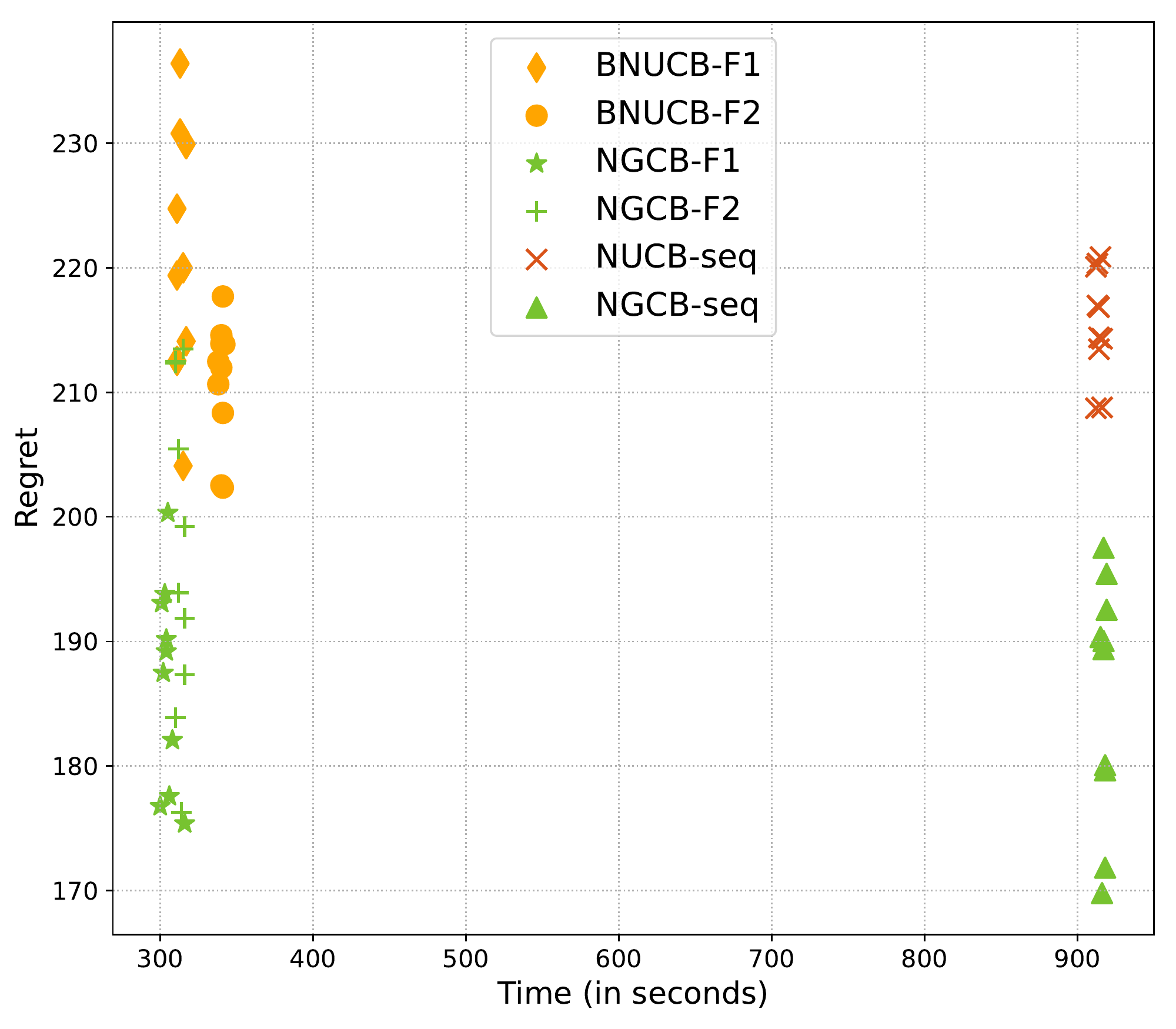}}
~
\subfloat[Mushroom with adaptive batch and $\sigma_2(x)$]{\label{fig:mushroom_s2_adaptive_batch} \centering \includegraphics[scale = 0.23]{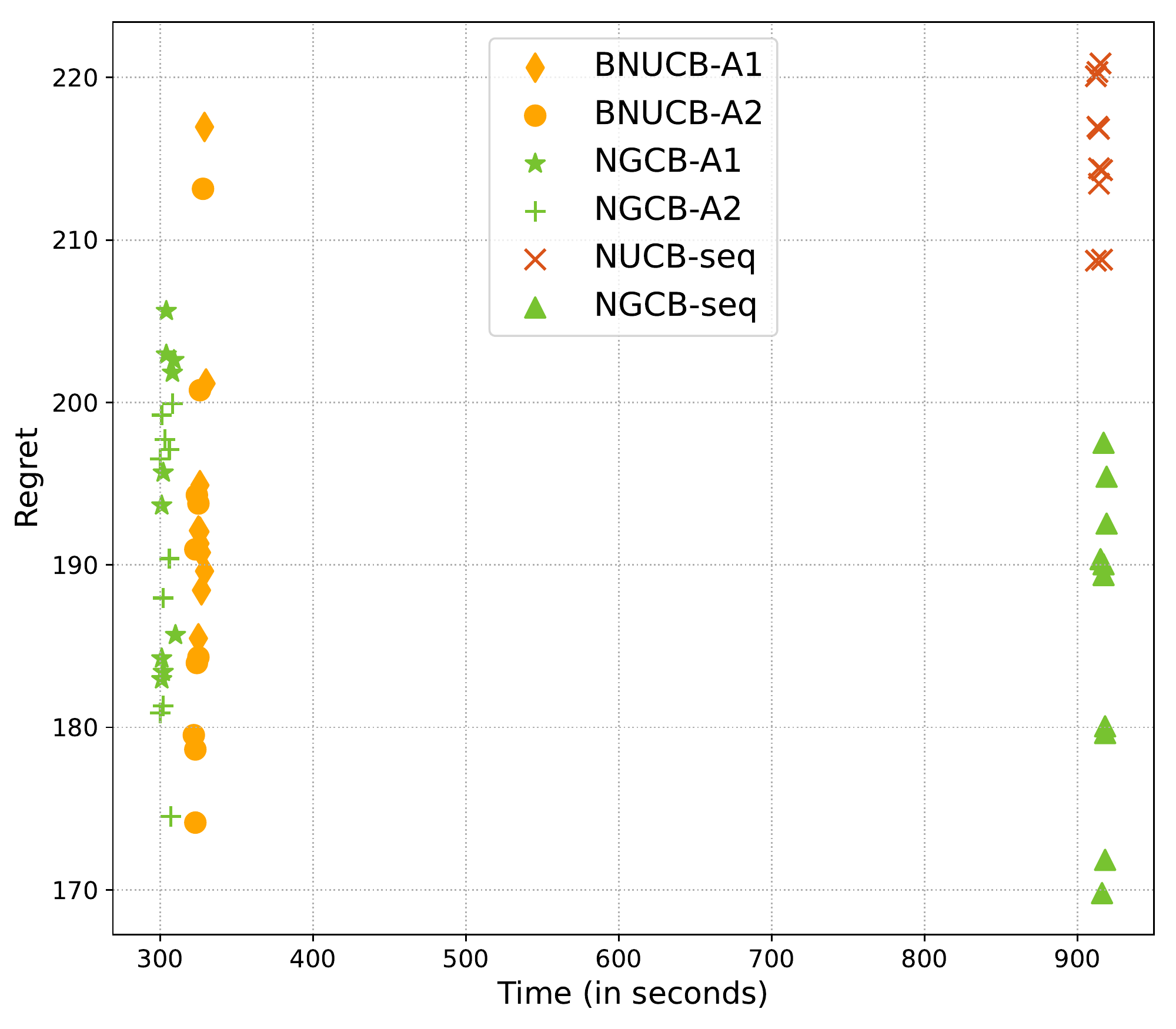}}

\caption{Plots for the dataset Mushroom}
\label{fig:plots_mushroom}
\end{figure*}

\begin{figure*}[t]
\centering
\subfloat[Statlog with $\sigma_1(x)$]{\label{fig:shuttle_s1}\centering \includegraphics[scale = 0.23]{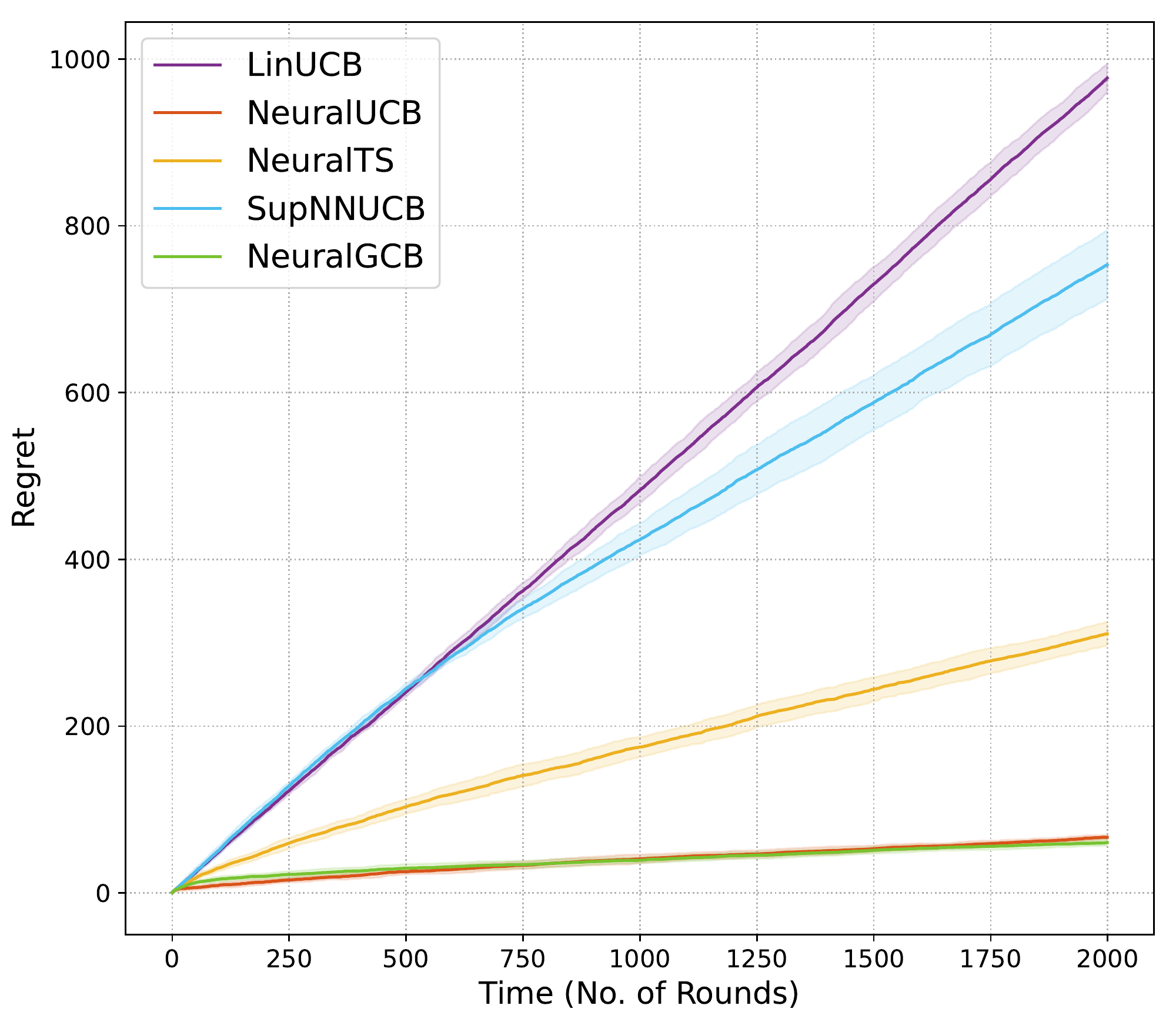}}
~
\subfloat[Statlog with $\sigma_2(x)$]{\label{fig:shuttle_s2}\centering \includegraphics[scale = 0.23]{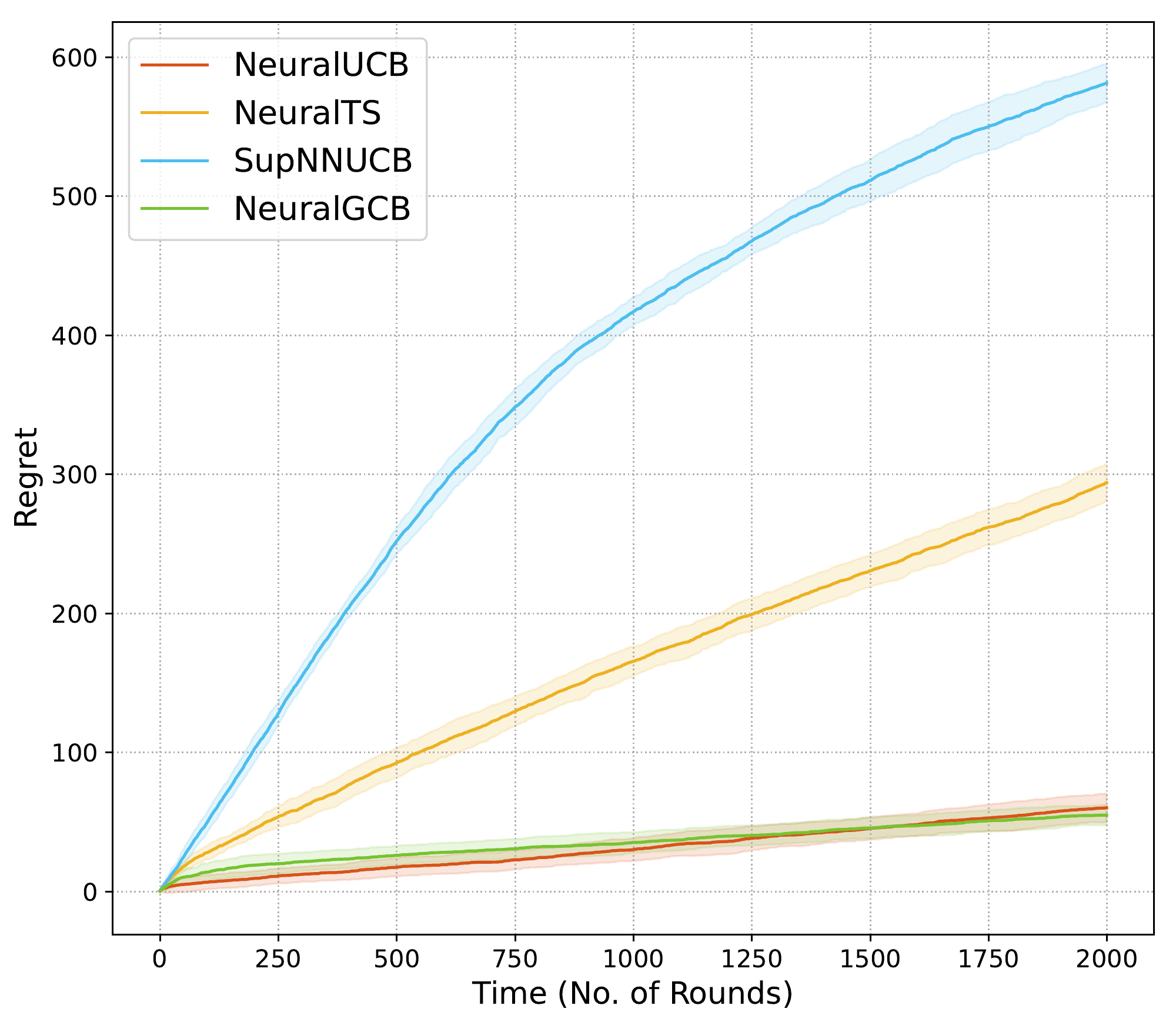}}
~
\subfloat[Statlog with fixed batch and $\sigma_1(x)$]{\label{fig:shuttle_s1_fixed_batch} \centering \includegraphics[scale = 0.23]{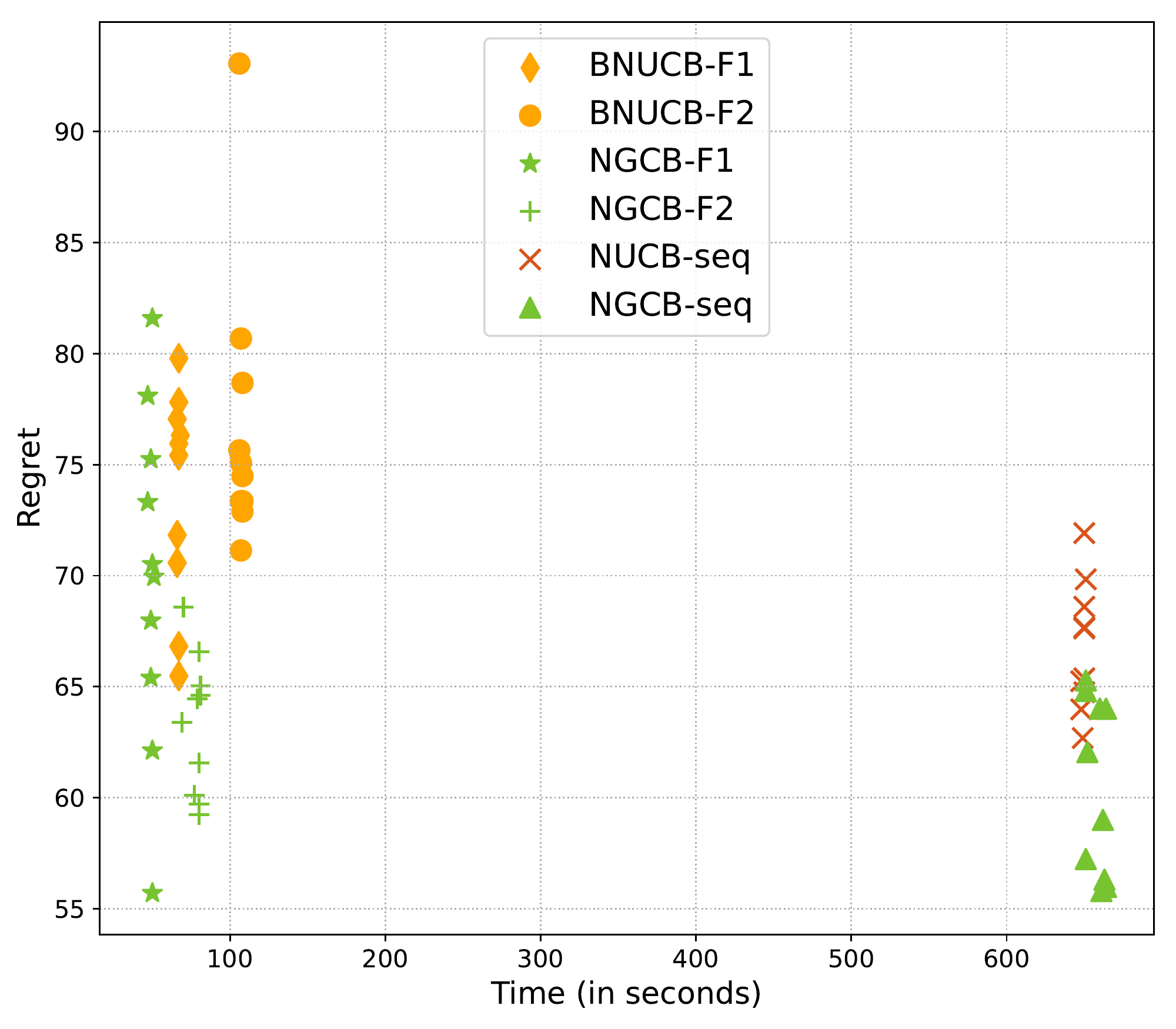}}

\subfloat[Statlog with adaptive batch and $\sigma_1(x)$]{\label{fig:shuttle_s1_adaptive_batch}\centering \includegraphics[scale = 0.23]{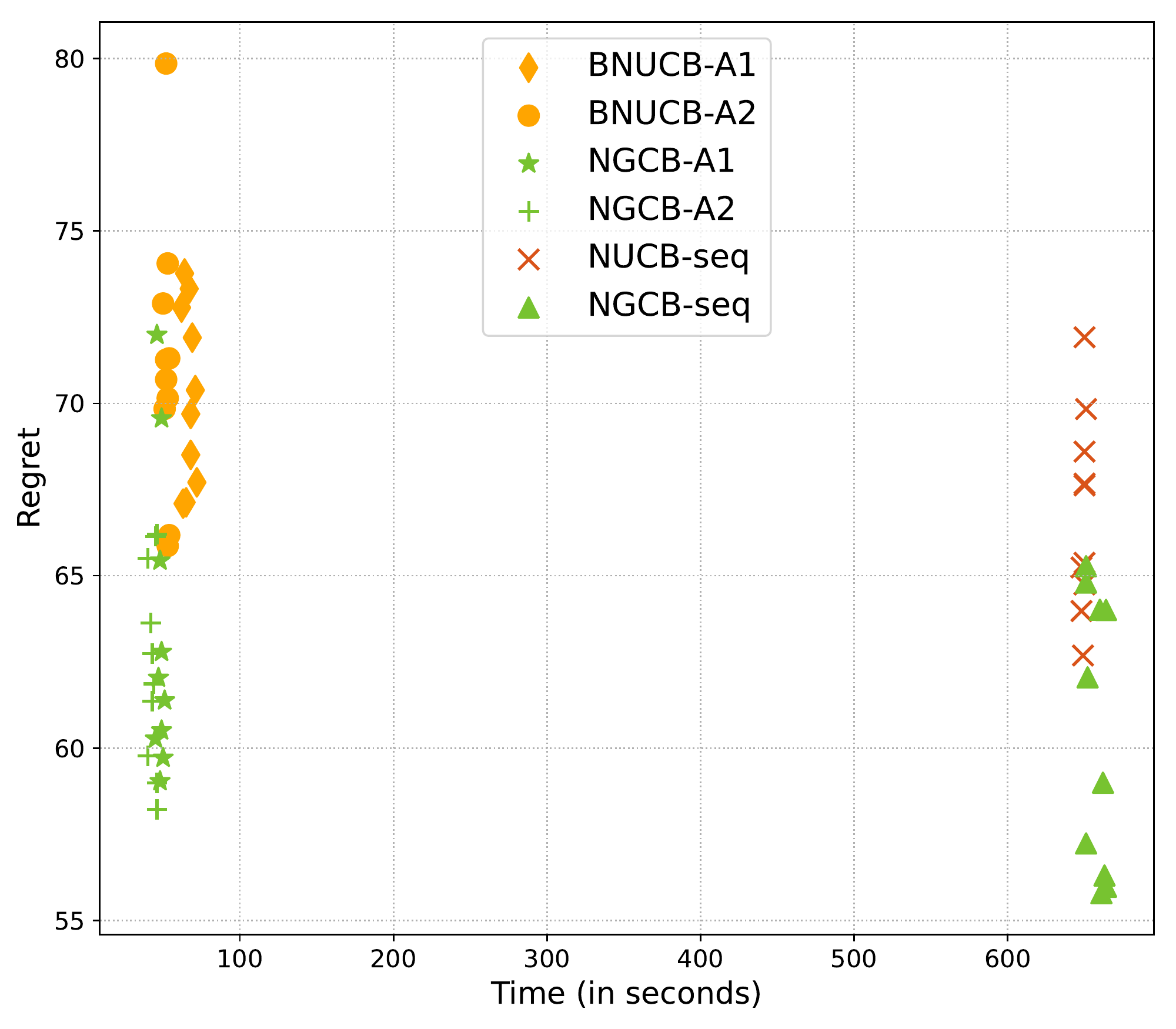}}
~
\subfloat[Statlog with fixed batch and $\sigma_2(x)$]{\label{fig:shuttle_s2_fixed_batch}\centering \includegraphics[scale = 0.23]{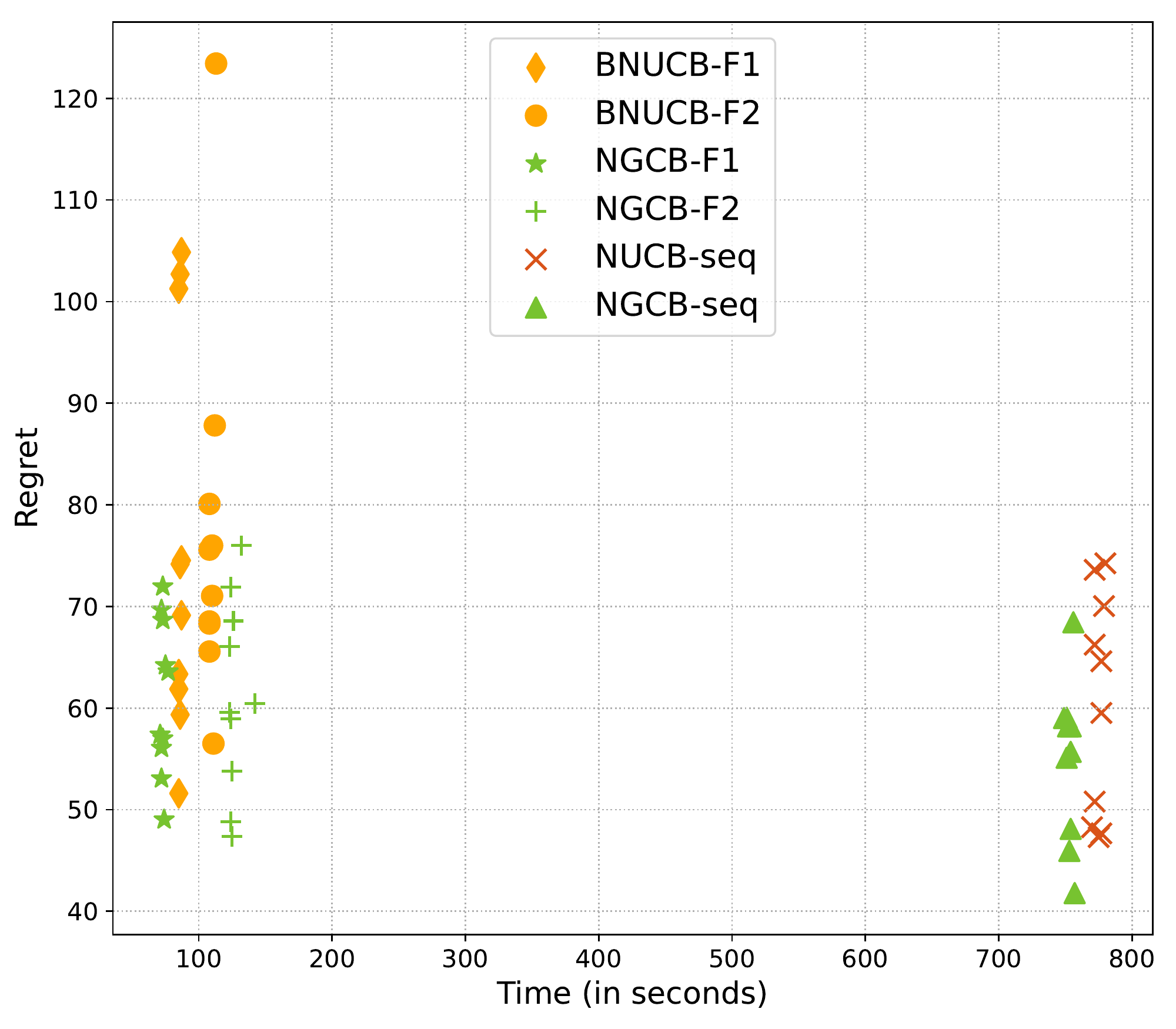}}
~
\subfloat[Statlog with adaptive batch and $\sigma_2(x)$]{\label{fig:shuttle_s2_adaptive_batch} \centering \includegraphics[scale = 0.23]{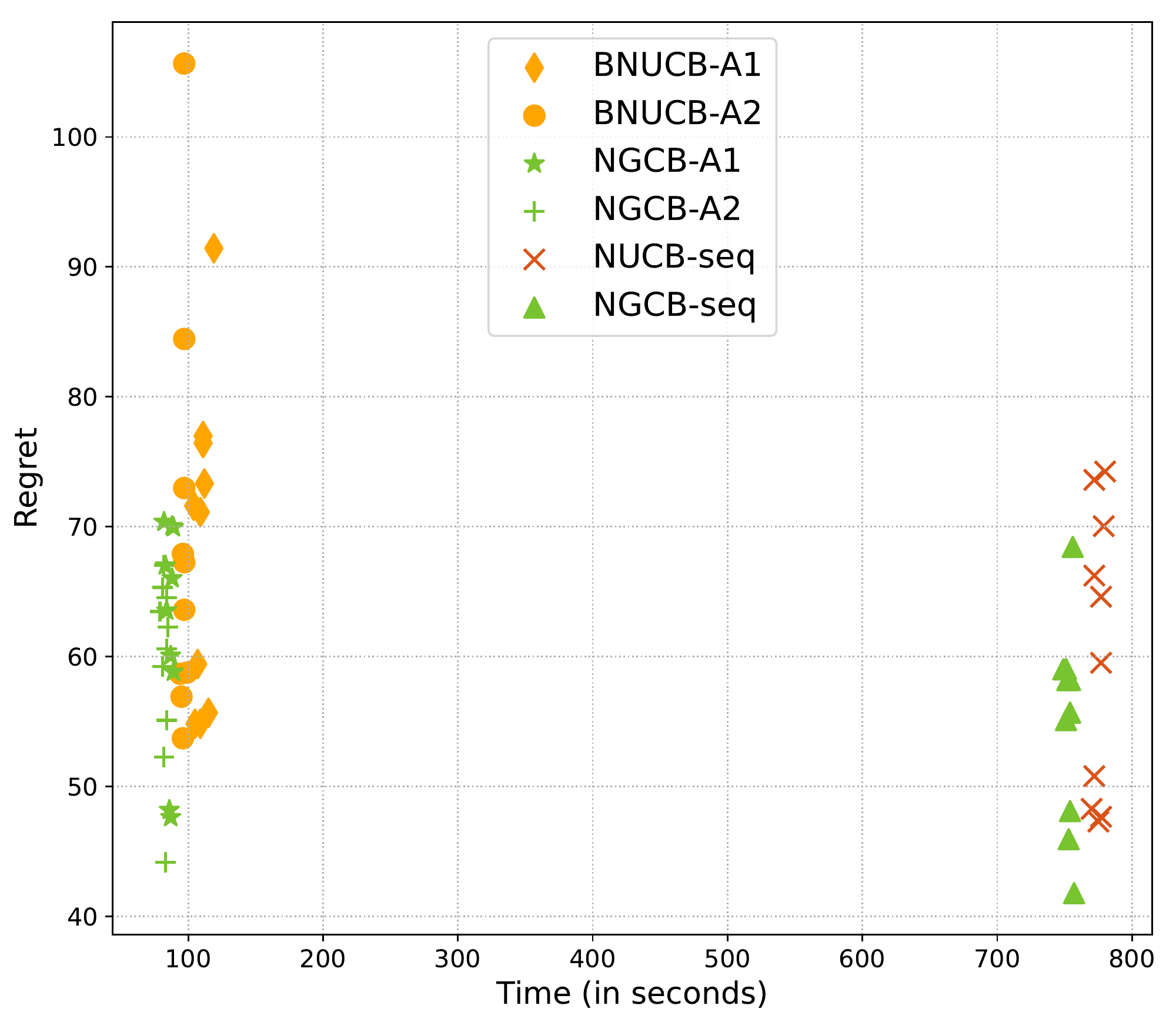}}

\caption{Plots for the dataset Statlog (Shuttle)}
\label{fig:plots_statlog}
\end{figure*}

From the above figures, it is evident that NeuralGCB outperforms all other existing algorithms on variety of datasets, including both synthetic and real world datasets. Moreover, the conclusion holds equally well for different activation functions, further bolstering the practical efficiency of the proposed algorithm. Additionally, it can be noted that the regret incurred by the algorithms in experiments with $\sigma_2$ as the activation is less than that for the experiments with $\sigma_1$ as the activation, thereby demonstrating the effect of smooth kernels on the performance of the algorithms in practice. \\

Lastly, the extensive study with different training schedules shows that batched version performs almost as well as the fully sequential version in terms of regret on all the datasets while having a considerably smaller running time. It can be noted from the plots that NeuralGCB has a superior performance compared to Batched-NeuralUCB in terms of regret for comparable running times on different datasets and with different activation functions which further strengthens the case of NeuralGCB as a practically efficient algorithm. \\

All the experiments were carried out using Python3 on a computer (CPU) with 12 GB RAM and Intel i7 processor (3.4 GHz) with an overall compute time of around 250-300 hours.

\end{document}